\documentclass{CUP-JNL-FMS}

\newif\ifFMSVersion
\FMSVersiontrue
\FMSVersionfalse

\usepackage{graphicx}
\graphicspath{{figures/}}
\usepackage{multicol,multirow}
\usepackage{amsmath,amssymb,amsfonts}
\usepackage{mathtools}
\usepackage{mathrsfs}
\usepackage{amsthm}
\usepackage{dsfont}
\usepackage{bm}
\usepackage{etoolbox,xparse}
\usepackage{tabularx,multirow,array,booktabs}
\usepackage{colortbl}
\usepackage{rotating}
\usepackage{appendix}
\usepackage{enumitem}
\usepackage{xurl}
\usepackage[numbers,sort&compress]{natbib}
\usepackage[T1]{fontenc}
\usepackage{newtxtext}
\usepackage{newtxmath}
\usepackage{textcomp}
\usepackage[colorlinks,allcolors=weblmcolor]{hyperref}

\newtheorem{theorem}{Theorem}[section]
\newtheorem{lemma}[theorem]{Lemma}
\newtheorem{proposition}[theorem]{Proposition}
\newtheorem{corollary}[theorem]{Corollary}

\newtheorem*{lemma*}{Lemma}
\theoremstyle{definition}
\newtheorem{definition}[theorem]{Definition}

\theoremstyle{remark}

\newtheorem*{remark*}{Remark}

\numberwithin{equation}{section}

\setlist[enumerate]{topsep=3pt,itemsep=3pt,parsep=0pt,partopsep=0pt,leftmargin=*,align=right,labelsep=0.600em,itemindent=0.00em,labelindent=0.6em}
\setlist[itemize]{topsep=3pt,itemsep=3pt,parsep=0pt,partopsep=0pt,leftmargin=*,align=right,labelsep=0.600em,itemindent=0.00em,labelindent=0.6em}

\newcommand{\myMFabc}[4]{\expandafter#1\csname#3#4\endcsname{{#2{#4}}}}
\newcommand{\myMFcmd}[4]{\expandafter#1\csname#3#4\endcsname{{#2{\csname#4\endcsname}}}}
\newcommand{\MFabc}[3][\newcommand]{%
    \def\doOld##1##2{\forcsvlist{\myMFabc{#1}{##1}{##2}}{#3}}%
    \providecommand{\do}{do}%
    \RenewDocumentCommand \do { >{\SplitList{,}} m } { \doOld##1 }%
    \docsvlist{#2}%
}
\newcommand{\MFcmd}[3][\newcommand]{%
    \def\doOld##1##2{\forcsvlist{\myMFcmd{#1}{##1}{##2}}{#3}}%
    \providecommand{\do}{do}%
    \RenewDocumentCommand \do { >{\SplitList{,}} m } { \doOld##1 }%
    \docsvlist{#2}%
}

\MFabc{ {\mathfrak,frak}, }{T,u,U,o,O,E,m}
\MFabc{ {\mathbb, }, }{A,N,Z,R,C,Q}

\newcommand{\hatbm}[1]{\widehat{\bm{#1}}}
\newcommand{\tildebm}[1]{\widetilde{\bm{#1}}}
\newcommand{\bmcal}[1]{\bm{\mathcal{#1}}}
\newcommand{\caltilde}[1]{\mathcal{\widetilde{#1}}}
\newcommand{\calhat}[1]{\mathcal{\widehat{#1}}}
\newcommand{\scrtilde}[1]{\widetilde{\mathscr{#1}\mspace{1mu}\mspace{-1mu}}}
\newcommand{\scrhat}[1]{\mathscr{\widehat{#1}\mspace{1mu}\mspace{-1mu}}}
\newcommand{\bmcalhat}[1]{\bm{\mathcal{\widehat{#1}}}}
\newcommand{\bmcaltilde}[1]{\bm{\mathcal{\widetilde{#1}}}}
\MFabc{ {\mathcal,cal}, {\mathscr,scr}, {\mathbb,bb}, {\bmcal,bmcal}, {\bmcal,calbm}, {\caltilde,caltilde}, {\caltilde,tildecal}, {\calhat,calhat}, {\calhat,hatcal}, {\scrtilde,scrtilde}, {\scrtilde,tildescr}, {\scrhat,scrhat}, {\scrhat,hatscr}, {\bmcalhat,bmcalhat}, {\bmcalhat,bmhatcal}, {\bmcalhat,calbmhat}, {\bmcalhat,calhatbm}, {\bmcalhat,hatbmcal}, {\bmcalhat,hatcalbm}, {\bmcaltilde,bmcaltilde}, {\bmcaltilde,bmtildecal}, {\bmcaltilde,calbmtilde}, {\bmcaltilde,caltildebm}, {\bmcaltilde,tildebmcal}, {\bmcaltilde,tildecalbm}}{A,B,C,D,E,F,G,H,I,J,K,L,M,N,O,P,Q,R,S,T,U,V,W,X,Y,Z}

\MFabc{{\bm,bm}, {\mathsf,sf}, {\widehat,hat}, {\widetilde,tilde}, {\hatbm,hatbm}, {\hatbm,bmhat}, {\tildebm,tildebm}, {\tildebm,bmtilde}}{a,b,c,d,e,f,g,h,i,j,k,l,m,n,o,p,q,r,s,t,u,v,w,x,y,z,A,B,C,D,E,F,G,H,I,J,K,L,M,N,O,P,Q,R,S,T,U,V,W,X,Y,Z}

\MFcmd{{\bm,bm}, {\widehat,hat}, {\widetilde,tilde}, {\tildebm, tildebm}, {\tildebm, bmtilde}, {\hatbm, hatbm}, {\hatbm, bmhat}}{alpha,beta,gamma,delta,epsilon,zeta,eta,theta,iota,kappa,lambda,mu,nu,xi,omicron,pi,rho,sigma,tau,upsilon,phi,chi,psi,omega,Alpha,Beta,Gamma,Delta,Epsilon,Zeta,Eta,Theta,Iota,Kappa,Lambda,Mu,Xi,Omicron,Pi,Rho,Sigma,Tau,Upsilon,Phi,Chi,Psi,Omega,varrho,varphi,vartheta,varepsilon,varsigma,ell,eps}

\usepackage{xspace}
\newcommand{\actdef}[1]{%
  \expandafter\def\csname#1\endcsname{%
    {\ensuremath{\mathtt{#1}}}\xspace}}
\forcsvlist{\actdef}{ReLU,LReLU,LeakyReLU,ELU,GELU,SiLU,Softplus,dGELU,dSiLU,dSoftplus,Tanh,Sigmoid,Arctan,Softsign,SRS,dSRS,Swish,dSwish,Mish,dMish,SELU,CELU,dSELU,Sin,SinLU,SinTU,PSinTU,sine,cosine,Sine,Cosine,EUAF,DUAF,PEUAF,UAF}

\newlength{\myLength}

\let\tilde\widetilde

\let\epsilon\varepsilon

\let\subset\subseteq

\let\cdots\customcdots
\let\ldots\cdots
\let\dots\cdots
\let\myforall\forall
\def\forall{{\myforall\, }}
\let\myexists\exists
\def\exists{{\myexists\, }}

\definecolor{mygray}{RGB}{230,230,230}
\definecolor{myorange}{HTML}{ff7f0e}

\let\cite\citep

\ifFMSVersion
\articletype{RESEARCH ARTICLE}
\jyear{2026}
\makeatletter
\gdef\@DOI{}
\gdef\@logourl{}
\renewcommand{\printhistory}{}
\def\@opjournalheader{\textit{\@jname}\space {{(\@jyear),\ \textbf{\@jvol}:\@artid}} {\thepage{--}\pageref*{LastPage}}}
\makeatother

\copytext{\textcopyright\ The Author(s) 2026. This is an Open Access article, distributed under the terms of the Creative Commons Attribution licence (http://creativecommons.org/licenses/by/4.0/), which permits unrestricted re-use, distribution, and reproduction in any medium, provided the original work is properly cited.}
\else 
\articletype{}
\jyear{}
\makeatletter
\gdef\@DOI{}
\gdef\@logourl{}
\renewcommand{\printhistory}{}
\def\@opjournalheader{}
\def\ps@titlepage{%
     \def\@oddhead{}%
     \let\@evenhead\@oddhead%
     \def\@oddfoot{\vbox to 0pt{\vskip-\opshortpage{\printcopyright{\@copyrighttext}}}}
     \let\@evenfoot\@oddfoot}%
\def\ps@headings{%
    \def\@oddfoot{\hfill}
    \let\@evenfoot\@oddfoot
    \def\@evenhead{\thepage\hfill}
    \def\@oddhead{\hfill\thepage}
}
\makeatother
\copytext{}
\fi

\usepackage{lineno}

\begin{document}

\begin{Frontmatter}

\title[Sobolev Approximation by Fixed-Size Neural Networks]{Sobolev Approximation by Fixed-Size Neural Networks with Arbitrary Accuracy}

\author[1]{Baicheng Li}
\author[1]{Haizhao Yang}
\author[2]{Shijun Zhang}

\authormark{B. Li, H. Yang and S. Zhang}

\address[1]{\orgdiv{Department of Mathematics}, \orgname{University of Maryland}, \orgaddress{\city{College Park}, \state{MD}, \country{USA}}; \email{baichl@umd.edu}, \email{hzyang@umd.edu}}
\address[2]{\orgdiv{Department of Applied Mathematics}, \orgname{The Hong Kong Polytechnic University}, \orgaddress{\city{Hong Kong}, \country{China}}; \email{shijun.zhang@polyu.edu.hk}}


\abstract{
In this work, we investigate new activation functions for achieving arbitrary-accuracy Sobolev approximation by fixed-size neural networks.
We first show that any function in $W^{2,\infty}((a,b)^d)$ can be approximated with arbitrary accuracy, measured in the $W^{1,\infty}$-norm, by a fixed-size neural network using the Elementary Universal Activation Function (\EUAF).
To extend this result to $W^{s,\infty}((a,b)^d)$ for $s\in\mathbb N$, we introduce a smooth activation $\DUAF_{\infty}$ from the family of Differentiable Universal Activation Functions (\DUAF{s}). We prove that any function in $W^{s,\infty}((a,b)^d)$ can be approximated with arbitrary accuracy in the $W^{s-1,\infty}$-norm by a fixed-size $\DUAF_{\infty}$--activated network. We further construct sigmoidal variants $\widetilde{\DUAF}_{n}$ and show that, for every
$1\le s\le n$, fixed-size $\widetilde{\DUAF}_{n}$--activated networks still approximate
any $f\in W^{s,\infty}((a,b)^d)$ with arbitrary accuracy in the
$W^{s-1,\infty}$-norm. In all these results, the width and depth bounds are computed
explicitly, and the proposed activations are elementary.
}

\end{Frontmatter}

\ifFMSVersion
\linenumbers
\fi
\section{Introduction}
\label{Introduction}

Deep neural networks have achieved remarkable success in a broad spectrum of applications, including scientific computing, computer vision, natural language processing, and data-driven modeling of high-dimensional systems. Beyond these empirical advances, a substantial body of recent work has focused on the theoretical underpinnings of neural networks, with particular emphasis on their expressive power, approximation properties, and architectural efficiency. Such questions are central to understanding why ostensibly simple architectures can represent highly complex functions arising in practice.

A classical line of research on universal approximation establishes that sufficiently large networks with suitable activation functions can approximate locally integrable functions on compact domains arbitrarily well. These results, however, are predominantly qualitative and do not quantify how the approximation error depends on the dimension or on architectural parameters such as width and depth. Subsequent work has therefore shifted attention to quantitative approximation rates of deep networks, investigating how the network size (in terms of width, depth, or number of parameters) scales with the target accuracy. From an approximation-theoretic viewpoint, these results are powerful yet inherently limited: achieving sufficiently small error typically requires networks whose size grows as the desired accuracy increases. This naturally raises the question of whether one can construct activation functions for which fixed-size networks can approximate certain function classes with arbitrary accuracy. A series of works \citep{yarotsky21a,maiorov99,petersen24,shen22,wang25EUAF} has answered this question affirmatively in the setting of $L^p$-approximation, $p \in [1,\infty)$, by exhibiting so-called super-expressive activation functions.

In many applications, however, approximation in $L^p$ of the function values alone is not sufficient, as it offers essentially no control over derivatives. In PDE-based and variational models, the loss functional and physically relevant quantities typically depend on $\nabla u$, $\Delta u$, or higher-order derivatives. Consequently, an $L^p$-accurate surrogate may still exhibit large PDE residuals or significant errors in energy and other derived quantities. Sobolev norms (e.g., $W^{s,p}$) directly address this issue by measuring the joint error of a function and its derivatives, and thus provide the natural framework for analyzing neural-network surrogates of solutions to differential equations and related variational problems. Motivated by these considerations, there has been growing interest in Sobolev approximation theory for neural networks. Yet, existing results for standard activations such as \ReLU, sigmoid, or $\tanh$ typically suffer from intrinsic limitations: the achievable smoothness is restricted, and the network size must increase as the target accuracy improves, often with a pronounced dependence on the dimension.

To the best of our knowledge, no prior work has constructed activation functions that enable arbitrary-accuracy Sobolev approximation by fixed-size network architectures. This paper bridges this gap by providing explicit, computable activation functions and corresponding architectures that achieve such fixed-size, arbitrary-accuracy approximation guarantees in Sobolev norms. First in Section~\ref{sec:EUAF:main}, we show that any function in $W^{2,\infty}((a,b)^d)$ can be approximated with arbitrary accuracy in the $W^{1,\infty}$-norm by a fixed-size neural network using the Elementary Universal Activation Function (\EUAF) proposed in \citep{shen22}. To extend this result to $W^{s,\infty}((a,b)^d)$, in Section~\ref{sec:duaf-construction}, we introduce a smooth, elementary, and explicitly computable activation function $\DUAF_{\infty}$ from the family of \emph{Differentiable Universal Activation Functions} (\DUAF{s}). We prove that any function in $W^{s,\infty}((a,b)^d)$ can be approximated with arbitrary accuracy in the $W^{s-1,\infty}$-norm by a fixed-size $\DUAF_{\infty}$--activated neural network, with explicitly computed bounds on both the network width and depth. We further construct a sigmoidal activation $\widetilde{\DUAF}_n\in C^{n}(\mathbb R)$, and show that for $1\le s\le n$ and any $f\in W^{s,\infty}((a,b)^d)$, there exists a fixed-size $\widetilde{\DUAF}_n$--activated neural network that approximates $f$ arbitrarily well in the $W^{s-1,\infty}$-norm. The corresponding network width and depth depend only on $d$ and $s$ (and polynomially on $n$), and remain independent of the approximation accuracy. Finally, in Section~\ref{sec:QKST:main}, we show that structured Sobolev
targets allow much smaller fixed-size architectures. If the target admits a
regular Kolmogorov-type approximation with a fixed number $Q$ of channels, then
$\DUAF_\infty$ networks achieve width linear in $dQ$ and depth independent of
$d$. 

\subsection{Comparison with existing super-expressive activations.}
Table~\ref{tab:super-expressive-activation-comparison} separates the approximation
guarantee from the qualitative properties of the activation function.  Existing
super-expressive constructions mainly concern function-value approximation in
$C([a,b]^d)$ or its consequence in $L^p$ spaces.  In contrast,
Theorem~\ref{thm:global-w2}, Corollary~\ref{cor:global-winfty}, and
Theorem~\ref{thm:sigmoidal} give fixed-size arbitrary-accuracy approximation in
Sobolev norms.

\begin{table}[h]
\centering
\caption{Comparison of representative super-expressive activation functions}
\label{tab:super-expressive-activation-comparison}

\begingroup
\scriptsize
\setlength{\tabcolsep}{4.6pt}
\renewcommand{\arraystretch}{1.35}

\resizebox{0.999\textwidth}{!}{
\begin{tabular}{ccccccccc}
\toprule
\textbf{Reference}
& \textbf{Activation}
& \textbf{Target space}
& \textbf{Norm}
& \textbf{Smoothness}
& \textbf{Sigmoidal} \\
\midrule

Maiorov--Pinkus \citep{maiorov99}
& abstract
& $C([a,b]^d)$
& $L^\infty$
& analytic but not elementary
& yes \\

Yarotsky \citep{yarotsky21a}
& $\{\sin,\arcsin\}$
& $C([a,b]^d)$
& $L^\infty$
& analytic but using two activations
& no \\

Shen--Yang--Zhang \citep{shen22}
& \EUAF
& $C([a,b]^d)$ / $L^p([a,b]^d)$
& $L^\infty$ / $L^p$
& $C^0(\bbR)\setminus C^1(\bbR)$; admits $C^s$ variants $\forall s$
& no; admits sigmoidal variants \\

Wang et al. \citep{wang25EUAF}
& $\PEUAF$ family
& $C([a,b]^d)$
& $L^\infty$
& $C^0(\bbR)\setminus C^1(\bbR)$
& partly \\

This paper, Theorem~\ref{thm:global-w2}
& \EUAF
& $W^{2,\infty}((a,b)^d)$
& $W^{1,\infty}$
& $C^0(\bbR)\setminus C^1(\bbR)$
& no \\

This paper, Corollary~\ref{cor:global-winfty}
& $\DUAF_\infty$
& $W^{s,\infty}((a,b)^d)$, $s\in\mathbb N$
& $W^{s-1,\infty}$
& $C^\infty(\bbR)$
& no \\

This paper, Theorem~\ref{thm:sigmoidal}
& $\widetilde{\DUAF}_n$
& $W^{s,\infty}((a,b)^d)$, $1\le s\le n$
& $W^{s-1,\infty}$
& $C^n(\bbR)$
& yes \\

Analogous to Theorem~\ref{thm:sigmoidal}
& $\widetilde{\widetilde{\DUAF}}_n$
& $W^{s,\infty}((a,b)^d)$, $1\le s\le n$
& $W^{s-1,\infty}$
& $C^n(\bbR)$
& yes and strictly increasing \\

\bottomrule
\end{tabular}
}

\vspace{0.4em}

\begin{minipage}{0.98\textwidth}
\scriptsize
\small
\fontsize{9pt}{10.8pt}\selectfont
\emph{Notes.}
For Maiorov--Pinkus \citep{maiorov99}, the activation is analytic and sigmoidal,
but its construction is existential rather than given by a simple elementary
expression. For Wang et al. \citep{wang25EUAF}, ``partly'' means that the proposed
family includes sigmoidal examples, although not every member is necessarily
sigmoidal.
\end{minipage}
\endgroup
\end{table}

\subsection{Contribution and related work}
Classical universal approximation theorems show that sufficiently large neural
networks are dense in broad function spaces, but they do not quantify the dependence
of the architecture on the approximation accuracy
\citep{cybenko89,hornik89,hornik91,funahashi89,hornik90,leshno93,pinkus99survey}.
Quantitative approximation theory gives such dependence and is closely related to
nonlinear approximation \citep{devore98,cohen22,daubechies22}.  For Barron-type
classes, two-layer networks achieve the dimension-independent Monte Carlo rate
$\mathcal{O}(N^{-1/2})$ in the $L^2$--space using width at most $N$ \citep{barron93}.  For classical
smoothness classes, the standard benchmark is algebraic and dimension-dependent, of
the form $\mathcal{O}(N^{-s/d})$ for \(s\)-smooth functions in \(d\) variables
\citep{mhaskar96,devore98}.  For \ReLU networks, \citet{yarotsky17} proved that
\(\varepsilon\)-accuracy for Sobolev-smooth functions can be achieved with
\(\mathcal{O}(\varepsilon^{-d/s})\) nonzero weights up to logarithmic factors, and
\citet{yarotsky18} showed that very deep fixed-width \ReLU networks approximate
continuous functions at the optimal rate
\(\mathcal{O}(\omega_f(W^{-2/d}))\) in terms of the total number \(W\)
of weights.  For \(C^s\)-functions, \citet{lu21} obtained the width-depth bound that
\ReLU networks with width \(\mathcal{O}(N\ln N)\) and depth \(\mathcal{O}(L\ln L)\)
achieve error
\(\mathcal{O}(\|f\|_{C^s}N^{-2s/d}L^{-2s/d})\).  Related refinements for compositional
or phase-diagram regimes were developed in \citet{shen19a,yarotsky20}.  In Sobolev
and Besov spaces, \citet{guhring20} obtained \ReLU approximation rates of the form
\(\mathcal{O}(W^{-(s-r)/d})\) in weaker \(W^{r,p}\)-norms, \citet{deryck21} proved
high-order Sobolev approximation estimates for \(\tanh\)-networks, and \citet{siegel23}
proved optimal deep \ReLU rates \(\mathcal{O}(L^{-2s/d})\) on Sobolev and Besov balls
using fixed width \(25d+31\) and growing depth \(L\).  Piecewise smooth targets,
PDE-related target classes, and functions optimally approximable by affine systems
have also been studied in \citet{petersenVoigtlaender18,elbrachter22,beck20,bolcskei19,
honYang21,yang22}.  These results are sharp or nearly sharp in their respective
settings, but the width, depth, or number of weights necessarily grows as
approximation error approaches zero.

One way to reduce this dependence is to impose additional structure on the target
space or to use more expressive activations.  For analytic, Barron, and band-limited
classes, improved dimension dependence is possible under stronger assumptions on the
target function \citep{eWang18,e2019,eWojtowytsch22,siegelXu21,chenWu19,montanelli21}.
Another approach is to use nonstandard or hybrid activations: \citet{yarotsky20}
obtained root-exponential rates \(\mathcal{O}(e^{-c_d\sqrt W})\) for
\((\sin,\mathrm{\ReLU})\)-networks with \(W\) parameters; \citet{shen21a} obtained
\(\mathcal{O}(\sqrt d\,N^{-\sqrt L})\) for \((\mathrm{Floor},\mathrm{\ReLU})\)-networks
of width \(\mathcal{O}(N)\) and depth \(\mathcal{O}(L)\); and
\citet{shen21b,jiao21} obtained exponential-type rates such as
\(\mathcal{O}(\sqrt d\,2^{-N})\) by allowing stronger activation mechanisms.  These
results improve the rate or dimension dependence, but they are still growing-size
approximation results.

A more radical direction is fixed-size approximation, whose conceptual origin is the
Kolmogorov--Arnold superposition theorem \citep{kolmogorov57,arnold57}.  The classical
Kolmogorov representation uses univariate functions depending on the target, so it
does not yield a target-independent activation function.  \citet{maiorov99} proved
fixed-size universality with a single analytic sigmoidal activation and
\(\mathcal{O}(d^2)\) neurons, but the activation is constructed abstractly and is not
known to be elementary or computable in closed form.  \citet{guliyevIsmailov18}
studied related fixed-weight and fixed-architecture universality results.
\citet{yarotsky21a} introduced elementary super-expressive activations and showed
that fixed architectures depending only on \(d\) can approximate arbitrary continuous
functions, although the construction uses multiple activations and the precise
dimension dependence is not the same explicit width-depth bound as in later $\EUAF$
work. The closest
predecessor of the present paper is \citet{shen22}, which constructs the Elementary
Universal Activation Function \(\EUAF\) and proves that width \(36d(2d+1)\) and
depth \(11\), hence \(\mathcal{O}(d^2)\) neurons, suffice for arbitrary uniform
approximation on \(C([a,b]^d)\), and consequently for density in \(L^p([a,b]^d)\),
\(1\le p<\infty\).  Further \(\EUAF\)-type families and variants were developed in
\citet{wang25EUAF}.

The gap addressed here is specifically a Sobolev gap.  Existing fixed-size
super-expressive results are function-value approximation results: they prove density
in \(C([a,b]^d)\) or \(L^p([a,b]^d)\), but they do not control weak derivatives.
This distinction is essential.  A fixed-size \(L^\infty\)- or \(L^p\)-approximation
theorem cannot be differentiated to obtain \(W^{s,p}\)-approximation.  Moreover, the
gap is not merely that the original \(\EUAF\) is nonsmooth.  Even the \(C^s\)-\UAF and
sigmoidal \UAF variants discussed in \citet[Sections~4.1--4.2]{shen22} remain
statements about uniform approximation of continuous functions; they do not assert
fixed-size density in Sobolev spaces, nor derivative approximation, nor a
\(W^{s-1,\infty}\)-error bound.  Thus, to the best of our knowledge, before the
present work there was not even an existential fixed-size theorem showing that a
smooth or sigmoidal target-independent activation can approximate Sobolev functions
with arbitrary accuracy in Sobolev norms; a fortiori, there was no elementary,
explicit, computable smooth or sigmoidal activation with such a guarantee.

The first contribution of this paper is to show that the fixed-size phenomenon of
\citet{shen22} can be lifted from function-value approximation to first-order Sobolev
approximation.  Theorem~\ref{thm:global-w2} shows that every
\(f\in W^{2,\infty}((a,b)^d)\) can be approximated in \(W^{1,\infty}((a,b)^d)\) by an
\(\EUAF\)-activated network with width \(4^d(5d^2+8d+3)\) and depth \(2d+5\).  Thus
the architecture is independent of \(\varepsilon\), but the conclusion is stronger
than \(L^\infty\)-approximation because it simultaneously controls first derivatives.
The second contribution is the construction of differentiable universal activation
functions. Theorem~\ref{thm:global-ws} proves that for \(1\le s\le n+1\), every
\(f\in W^{s,\infty}((a,b)^d)\) can be approximated in \(W^{s-1,\infty}\) by a
\(\DUAF_n\)-activated fixed-size network with width
\(N_{s,d}=\mathcal{O}(4^d(s+d)\binom{s+d-1}{d})\) and depth
\(L_{s,d}=\mathcal{O}(s+d)\).  In particular, the \(C^\infty\) activation
\(\DUAF_\infty\) yields fixed-size arbitrary-accuracy approximation for every finite
Sobolev order. To the best of our knowledge, this is also the first smooth super-expressive activation function.

The third contribution is a sigmoidal Sobolev version.  Theorem~\ref{thm:sigmoidal}
constructs a bounded, monotone, \(C^n\) sigmoidal activation
\(\widetilde{\DUAF}_n\) and proves fixed-size approximation in \(W^{s-1,\infty}\) for
\(1\le s\le n\), with width \(S_nN_{s,d}\), where \(S_n:=\max\{4n^2+19,26\}\), and depth
\((3n^2+5)L_{s,d}\), i.e., only a polynomial \(\mathcal{O}(n^2)\) overhead relative to
the \(\DUAF_n\) architecture.  This is qualitatively different from the earlier
sigmoidal UAF construction in \citet{shen22}: the activation here is sigmoidal and
\(C^n\), but the approximation guarantee is in Sobolev norm rather than merely in
\(C([a,b]^d)\) or \(L^p([a,b]^d)\).  Finally, Theorem~\ref{thm:QKST} shows that if the target
belongs to the structured Kolmogorov-type class \(K_{d,Q}^s((a,b)^d)\), then the
general tensor-grid width can be replaced by width \(\mathcal{O}(dQs^2)\)
and depth \(\mathcal{O}(s)\).  Hence the paper separates two issues:
for arbitrary Sobolev targets, fixed-size approximation is possible but the width is
exponential in \(d\); for structured Sobolev targets, the same fixed-size phenomenon
is compatible with width linear in \(dQ\).

\subsection{Organization of the paper}
The rest of this paper is organized as follows. Section~\ref{sec:EUAF:main}
states the first-order Sobolev approximation result for the \EUAF activation.
Sections~\ref{sec:duaf-construction}--\ref{sec:sigmoidal-duafs} introduce the
\DUAF family and its sigmoidal variants, and establish the corresponding
fixed-size Sobolev approximation results. Section~\ref{sec:QKST:main} presents
the reduced-complexity result for structured Sobolev targets. Sections
\ref{sec:notation}--\ref{sec:DNN} collect the basic notation, neural-network
conventions, and Sobolev-space preliminaries, followed by the proof roadmap.
Section~\ref{sec:fundamental-lemmas} develops the algebraic and analytic
primitives for the \DUAF constructions. Section~\ref{sec:proof-main} proves the
main \EUAF and \DUAF approximation theorems through local Taylor approximation,
shifted-grid gluing, and affine rescaling. Section~\ref{sec:proof-qkst} proves
the structured-target theorem. Section~\ref{sec:proof-sigmoidal-loc} gives the
main technical proof for the sigmoidal activation replacement, while
Section~\ref{sec:proof-chi-est} proves the cutoff estimates used in that
argument. Finally, Section~\ref{sec:conclusion} concludes the paper and
discusses several open questions.

\section{Main results}

\subsection{Sobolev approximation using \EUAF}
\label{sec:EUAF:main}

The Elementary Universal Activation Function (\EUAF) was introduced in \citep{shen22}. It is defined by
\begin{equation}\label{equation:EUAF}
\EUAF(x):=
\begin{cases}
\tfrac{x}{1-x}, & \text{if } x \le 0,\\[0.6em]
\bigl|x - 2\lfloor \tfrac{x+1}{2} \rfloor\bigr|, & \text{if } x > 0.
\end{cases}
\end{equation}
An illustration of \EUAF on the interval $[-10,10]$ is provided in Figure~\ref{fig:EUAF-sigma}.

\begin{figure}[ht]
  \centering
  \includegraphics[width=0.526185\linewidth]{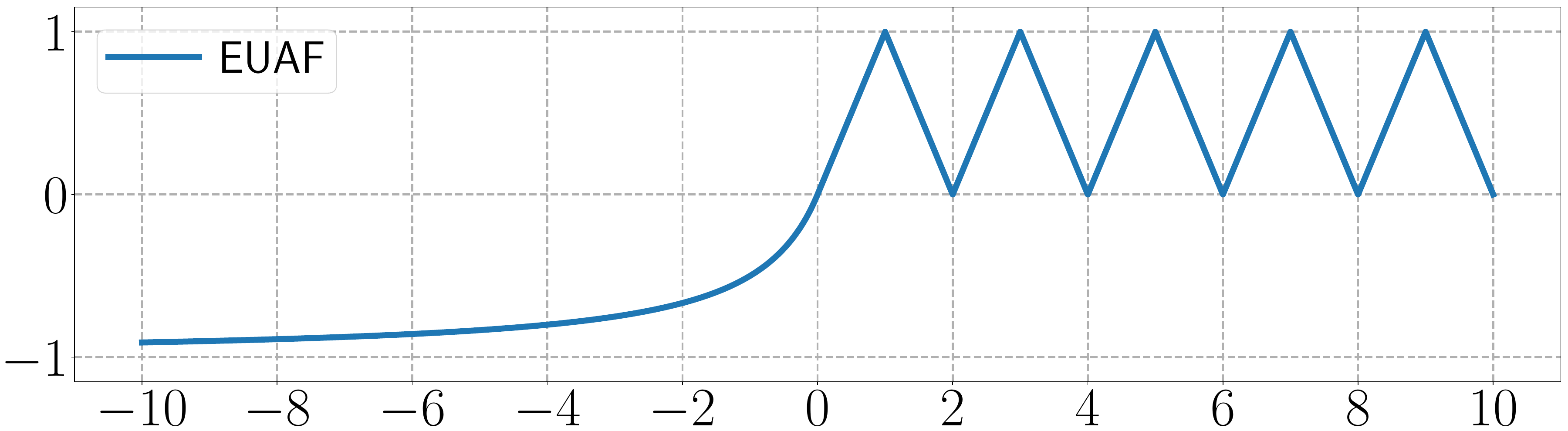}
  \caption{An illustration of  \EUAF on $[-10,10]$}
  \label{fig:EUAF-sigma}
\end{figure}

It was established in \citep{shen22} that fixed-size neural networks with \EUAF activation can approximate any function in $C([a,b]^d)$ arbitrarily well in the $L^\infty$-norm. Extending such fixed-size and arbitrary-accuracy approximation results to Sobolev norms is, however, highly non-trivial. In this work, we bridge this gap by establishing the following theorem.

\begin{theorem}
\label{thm:global-w2}
Given any $f\in W^{2,\infty}((a,b)^d)$ and  $\varepsilon>0$, there exists an $\EUAF$--activated network $\Phi$ with width no greater than $4^d(5d^2+8d+3)$ and depth no greater than $2d+5$, such that
\[
  \|\Phi-f\|_{W^{1,\infty}((a,b)^d)}<\varepsilon.
\]
\end{theorem}

The proof of Theorem~\ref{thm:global-w2} is provided in Section~\ref{sec:proof-main}. The network size appearing in Theorem~\ref{thm:global-w2} depends only on the input dimension $d$, which is typically prescribed and fixed; in particular, the architecture is independent of the target function and the approximation accuracy, and hence has fixed size. Moreover, Theorem~\ref{thm:global-w2} demonstrates that \EUAF networks can simultaneously approximate a function and its first-order derivatives within such a fixed-size architecture.
Finally, we remark that the exponential dependence of the network width on the dimension arises from the absence of any exploited superposition structure (in the spirit of the Kolmogorov-Arnold representation theorem) at this level of regularity. Developing new proof techniques to alleviate or remove this exponential dependence is an interesting direction for future research, but lies beyond the scope of the present paper.

\subsection{Construction of \DUAF{s}}\label{sec:duaf-construction}
\label{sec:def:DUAFs}
While $\EUAF$ supports simultaneous approximation of functions and their first-order derivatives, its intrinsic regularity limits its applicability to higher-order Sobolev spaces.
To overcome this limitation, we construct a family of smooth activation functions whose regularity can be prescribed explicitly.
The construction of \DUAF{s} separates the roles of transition, gating, and accumulation, allowing precise control of differentiability while preserving computability.

For each $n \in \mathbb{N} \cup \{\infty\}$, we define an increasing and continuous function $q_n \colon \mathbb{R} \to [0,1]$. For $u \notin (0,1)$, we set
\[
q_n(u):=
\begin{cases}
1, & u\ge 1,\\[0.1em]
0, & u \le 0.\\
\end{cases}
\]
While for $u \in (0,1)$, we define
\[
q_n(u):=
\begin{cases}
\frac{\int_0^{u} t^{\,n}(1-t)^{n}\,dt}
      {\int_0^{1} t^{\,n}(1-t)^{n}\,dt}, &   n\in\mathbb N,\\[0.6em]
\frac{e^{-1/u}}{e^{-1/u}+e^{1/(u-1)}}, &   n=\infty.\\
\end{cases}
\]

\begin{figure}[h]
  \centering
  \includegraphics[width=0.6\linewidth]{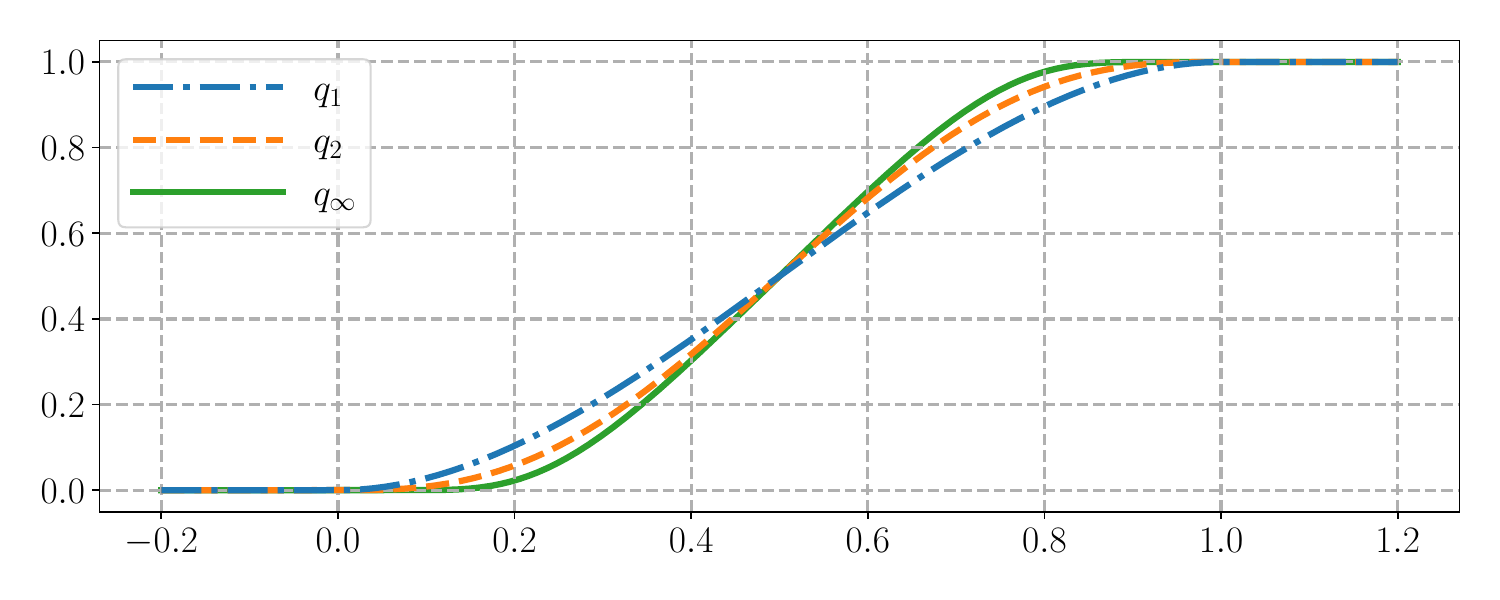}
  \caption{
  Comparison of $q_n$ on $[-0.2, 1.2]$ for $n=1,2, \infty$
  }
  \label{fig:q_n}
\end{figure}

The functions $q_n$ serve as $C^n$-regular interpolation profiles with vanishing derivatives up to order $n$ at their endpoints, and will be used to interpolate between distinct functional regimes in a $C^n$-controlled manner. An illustration of $q_n$ on the interval $[-0.2,1.2]$ is provided in Figure~\ref{fig:q_n}. 
We note that $q_n$ converges pointwise, as $n\to\infty$, to a discontinuous function which is locally constant away from $u=1/2$. In particular, $q_\infty$ does not coincide with this limit, but is instead introduced as a smooth interpolation profile.

Next, for given $n\in \mathbb N\cup\{\infty\}$, we present the detailed construction of $\DUAF_n$. The construction relies on two auxiliary functions $g_n$ and $h_n$, which happen to be the left half and right half of $\DUAF_n$.

\

\noindent\textbf{$C^n$--periodic gate $g_n$ and $\varrho_n$}

For  $n\in\mathbb N \cup\{\infty\}$, 
we define the $C^n$-periodic gate $g_n$ and $\varrho_n$ on single periods by
\begin{equation}\label{eq:def-qs}
    g_n(t)
  :=
  \begin{cases}
    q_n(2t), & t\in[0,\tfrac12],\\[0.2em]
    1 - q_n(2t-1), & t\in(\tfrac12,1],\\[0.2em]
    0, & t\in(1,2].
  \end{cases}
 \qquad
  \varrho_n(t)=
 \begin{cases}
 t, & t\in[0,\tfrac12],\\[0.4em]
 t-q_n(2t-1), & t\in(\tfrac12,1].
 \end{cases}
 \end{equation}
then extend $g_n$ $2$--periodically and $\varrho_n$ $1$--periodically to the entire $\mathbb R$.

The periodic gate $g_n$ acts as a localization mechanism, selectively activating specific regions while maintaining global $C^n$ regularity. Another important property is that $g_n^{(k)}(0)=0$ for $k=0,\ldots,n$. This endpoint vanishing property will allow us to smoothly glue the two components of $\DUAF_n$ at the origin while preserving $C^n$ regularity. The function $\varrho_n$ serves as a periodic coordinate correction, preserving linear behavior on the active region while introducing a smooth periodic modulation that is compatible with the $C^n$--structure (see Figure~\ref{fig:g-rho-h} for an illustration).

\begin{figure}[h]
  \centering
  \includegraphics[width=\linewidth]{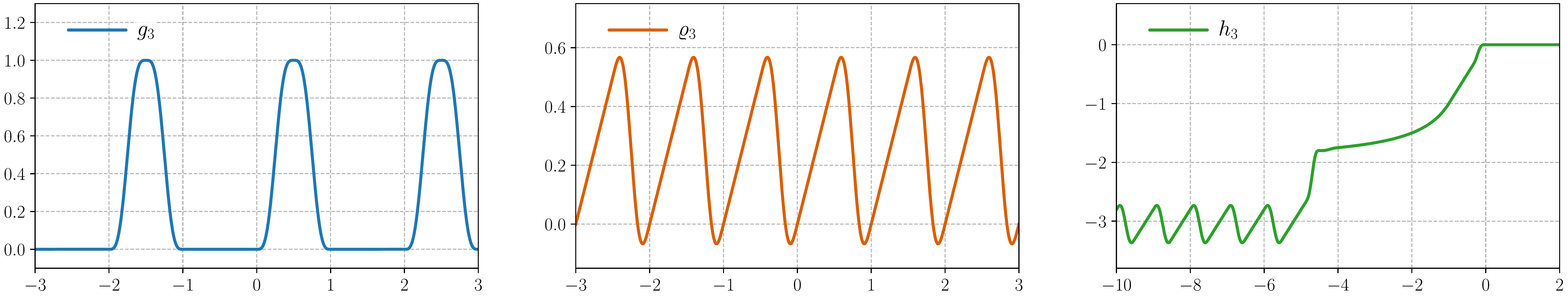}
  \caption{
  Illustration of $g_n$, $\varrho_n$ and $h_n$ for $n=3$
  }
  \label{fig:g-rho-h}
\end{figure}

\

\noindent\textbf{$C^n$--transition function $h_n$}\label{def:T}

For $n\in\mathbb N \cup\{\infty\}$, we define $h_n:\mathbb R\to\mathbb R$ by
\[
h_n(t):=
\begin{cases}
-9/5-\bigl(1-q_n(2t+10)\bigr)\bigl(1+\varrho_n(-t-5)\bigr),
& t\in(-\infty,-\frac92],\\[1em]
-9/5+q_n(2t+9)\!\left(-1/5-t^{-1}\right),
& t\in[-\frac92,-\frac{13}{11}],\\[1em]
-2-t^{-1}+q_n(\eta_n^{-1}(t+1+\eta_n))\left(t+2+t^{-1}\right),
& t\in\left(-\frac{13}{11},-1\right),\\[1em]
(1-q_n(11t/4+1))t
,
& t\in[-1,+\infty).
\end{cases}
\]
where $\eta_n:=\min\left\{\frac{2}{11},\frac{1}{8\|q_n'\|_{L^\infty(\mathbb R)}}\right\}$ (for example, $\eta_1=\frac{1}{12}$, $\eta_2=\frac{1}{15}$, $\eta_3=\frac{2}{35}$, and $\eta_{\infty}=\frac{1}{16}$).

An illustration of $h_3$ is given in Figure~\ref{fig:g-rho-h}. The function $h_n$ is globally $C^n$ on $\mathbb R$ (and $C^\infty$ for $n=\infty$). 
This follows from the endpoint flatness of the transition profile $q_n$, whose derivatives up to order $n$ vanish at $0$ and $1$. 
As a consequence, all derivatives up to order $n$ match at the junctions
\[
t=-\frac{9}{2},\qquad t=-\frac{13}{11},\qquad t=-1, \qquad t=0.
\]
The structure of $h_n$ consists of several explicit regimes. On $(-\infty, -5]$, the function reduces to a shifted periodic wave $-14/5-\varrho_n(-t-5)$. 
On the interval $\left[-4,-13/11\right]$, it coincides with the rational function $-2-t^{-1}$. On the interval $\left[-1,-4/11\right]$, it reduces to the identity function $t$, while on $\left[0,\infty\right)$ it becomes identically $0$. 
All remaining intervals are $C^n$ transition regions connecting these regimes.
Furthermore, $h_n$ is strictly increasing on $\left(-9/2,0\right)$ due to the fact that $h_n'(t)>0$ for all $t$ in this interval. 
At the junction points $t=-9/2, 0$, the derivatives of $h_n$ up to order $n$ vanish, namely
 $h_n^{(k)}(-9/2)=h_n^{(k)}(0)=0$ for $k=1,\dots,n$.
 
Taken together, the construction of $h_n$ yields a $C^n$-smooth interpolation between a periodic wave regime, a rational growth regime, and a linear regime, while preserving strict monotonicity on $\left(-9/2,0\right)$ and attaining a constant plateau for $t\ge0$.

\

\noindent\textbf{The definition of $\DUAF_n$}

The $n$-times continuously differentiable universal activation function $\DUAF_n:\mathbb R\to\mathbb R$ is defined by
  \[
  \DUAF_n(x):=
  \begin{cases}
    h_n(x), & \text{if}\ \ x\le 0,\\
    g_n(x), & \text{if}\ \ x>0.
  \end{cases}
  \]

The construction of $\DUAF_n$ can be viewed as a differentiable extension of the
$\EUAF$ framework. While $\EUAF$ is sufficient for first-order Sobolev approximation,
$\DUAF_n$ is designed to preserve the key super-expressive mechanism while providing
controlled smoothness up to order $n$. Thus, for a fixed differentiability level $n$,
one can approximate Sobolev functions of smoothness order up to $n+1$. The following
theorem makes this transition precise.

\subsection{Sobolev approximation using  $\DUAF_n$}

We first define two constants that will be used frequently through the rest of the paper. They serve as bounds on the network width and depth.
Let $d,s\in\mathbb N$, we define
\[
\begin{aligned}
N_{s,d}
&:= 4^d\left(\frac{d}{d+1}(s+5d+4)\,\binom{s+d-1}{d}+d+4\right),\\
L_{s,d}
&:= \max\{2s+2d-2, 2d+1, 7\}.
\end{aligned}
\]

While Theorem~\ref{thm:global-w2} shows that $\EUAF$--activated networks achieve
fixed-size approximation in the first-order Sobolev norm, the regularity of $\EUAF$
itself is limited by its piecewise-defined structure. In particular, $\EUAF$ is not
sufficiently smooth to support higher-order Sobolev approximation. To overcome this
limitation, we introduce a family of differentiable universal activation functions
$\DUAF_n$, $n\in\mathbb N\cup\{\infty\}$, whose construction was given in
Section~\ref{sec:duaf-construction}. These activations preserve the key expressive
features of $\EUAF$ while achieving higher smoothness through carefully designed
transition profiles. The resulting fixed-size Sobolev approximation theorem is as follows.

\begin{theorem}[Global $W^{s,\infty}$-approximation on a general cube]
\label{thm:global-ws}
Let $d,s\in\mathbb N$ and let $n\in\mathbb N\cup\{\infty\}$ satisfy $1\le s\le n+1$.
Suppose $f\in W^{s,\infty}((a,b)^d)$. Then, for any $\varepsilon>0$, there exists a
$\DUAF_n$--activated network $\Phi$ with width no greater than $N_{s,d}$ and depth no
greater than $L_{s,d}$ such that
\[
\|\Phi-f\|_{W^{s-1,\infty}((a,b)^d)}<\varepsilon.
\]
\end{theorem}

The proof of Theorem~\ref{thm:global-ws} is given in Section~\ref{sec:proof-main}.
In particular, taking $n=\infty$ gives a $C^\infty$ activation for which the smoothness
restriction $1\le s\le n+1$ becomes void. This yields the following important special
case.

\begin{figure}[t]
    \centering
    \includegraphics[width=\textwidth]{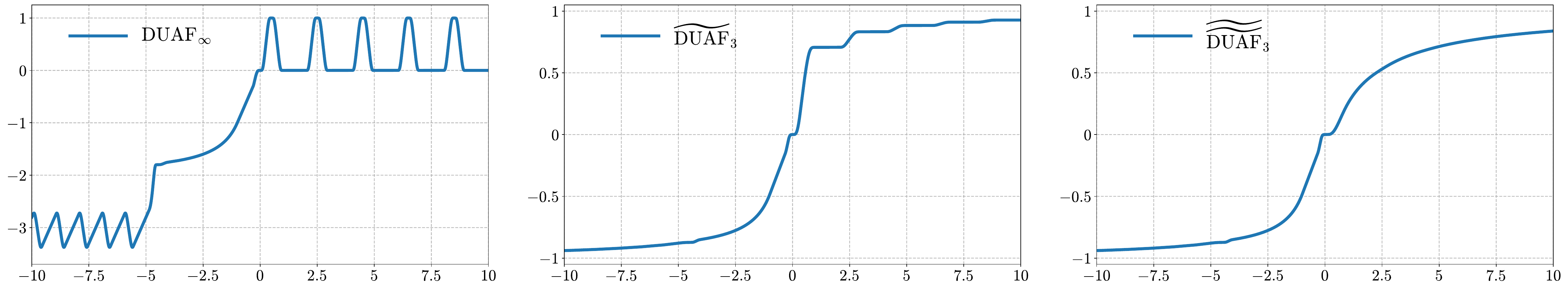}
    \caption{Illustration of $\DUAF_{\infty}$, $\widetilde{\DUAF}_3$, and $\widetilde{\widetilde{\DUAF}}_3$}
    \label{fig:duaf-and-sigmoidal-duafs}
\end{figure}

\begin{corollary}[Smooth Sobolev approximation]
\label{cor:global-winfty}
Let $d,s\in\mathbb N$ and let $f\in W^{s,\infty}((a,b)^d)$.
Then, for any $\varepsilon>0$, there exists a $\DUAF_\infty$--activated network $\Phi$
with width no greater than $N_{s,d}$ and depth no greater than $L_{s,d}$ such that
\[
\|\Phi-f\|_{W^{s-1,\infty}((a,b)^d)}<\varepsilon.
\]
\end{corollary}

Here $\DUAF_\infty$ is elementary, explicitly computable, and $C^\infty$-smooth (see
Figure~\ref{fig:duaf-and-sigmoidal-duafs} for an illustration). Therefore, the
fixed-size Sobolev approximation phenomenon is not tied to non-differentiable or
piecewise-linear activations, but persists even in the smooth setting.

\smallskip
{\noindent\itshape
\textbf{Remark.}
Corollary~\ref{cor:global-winfty} is stated in the $W^{s-1,\infty}$-norm. By monotonicity
of Sobolev norms, the same conclusion immediately holds in $W^{r,\infty}((a,b)^d)$
for every integer $0\le r\le s-1$. We do not use, and therefore do not state, any
noninteger interpolation consequence here.
\par}


\subsection{Sigmoidal \DUAF{s} and Sobolev approximation}
\label{sec:sigmoidal-duafs}

The preceding results show that fixed-size Sobolev approximation can be achieved by
the $C^n$--activations $\DUAF_n$ and, in particular, by the smooth activation
$\DUAF_\infty$. These activations, however, are not sigmoidal. Indeed, the periodic
component in the construction prevents the existence of limits at infinity. We now
show that the same approximation mechanism can be adapted to genuinely sigmoidal
activations.

For \(n\in\mathbb N\), we construct a sigmoidal activation from \(\DUAF_n\) by an
integral transformation. Define the normalization constants
\[
d_n:=
-\left(
-\frac95+\int_{-\infty}^{-\frac92}
\frac{\DUAF_n(t)+\frac95}{t^2}\,dt
\right)^{-1}>0,
\qquad
c_n:=
\left(
\int_{0}^{\infty}
\frac{\DUAF_n(t)}{(1+t)^2}\,dt
\right)^{-1}>0.
\]
We then define \(\widetilde{\DUAF}_n:\mathbb R\to\mathbb R\) by
\[
\widetilde{\DUAF}_n(x):=
\begin{cases}
d_n\left(
-\dfrac95+\displaystyle\int_x^{-\frac92}
\frac{\DUAF_n(t)+\frac95}{t^2}\,dt
\right),
& x\in(-\infty,-\tfrac92],
\\[0.8em]
d_n\,\DUAF_n(x),
& x\in[-\tfrac92,0],
\\[0.8em]
\displaystyle\int_0^x \frac{c_n\,\DUAF_n(t)}{(1+t)^2}\,dt,
& x\in(0,\infty).
\end{cases}
\]
This transformation preserves the local expressive structure of \(\DUAF_n\) on the
region used in the approximation construction, while forcing finite limits at
\(\pm\infty\) (see
Figure~\ref{fig:duaf-and-sigmoidal-duafs} for an illustration). More precisely, one verifies that
\[
\widetilde{\DUAF}_n\in C^n(\mathbb R),\qquad
\widetilde{\DUAF}_n'(x)\ge 0 \quad \text{for all } x\in\mathbb R,
\]
and
\[
\lim_{x\to-\infty}\widetilde{\DUAF}_n(x)=-1,
\qquad
\lim_{x\to+\infty}\widetilde{\DUAF}_n(x)=1.
\]
Thus \(\widetilde{\DUAF}_n\) is a bounded, monotone nondecreasing sigmoidal activation.

The following theorem shows that the fixed-size Sobolev approximation property remains
valid for this sigmoidal activation, at the cost of a controlled increase in width and
depth depending only on the regularity order \(n\).

\begin{theorem}\label{thm:sigmoidal} Let $d, n, s\in\mathbb{N}$, $1\le s\le n$, $f\in W^{s,\infty}((a,b)^d)$. Then, for any $\varepsilon>0$, there exists an $\widetilde{\DUAF}_n$-activated network $\tilde\Phi$ with width no greater than $S_nN_{s,d}$, where $S_n:=\max\{4n^2+19,26\}$, and depth no greater than $(3n^2+5)L_{s,d}$ such that 
\[ \|\tilde\Phi-f\|_{W^{s-1,\infty}((a,b)^d)}<\varepsilon. 
\] 
\end{theorem}

The proof of Theorem~\ref{thm:sigmoidal} is given in Section~\ref{D}. The argument
constructs transition networks that emulate the local mechanisms of the original
\(\DUAF_n\)-activated architecture, together with Sobolev error estimates that control
the loss caused by replacing \(\DUAF_n\) with its sigmoidal modification.

We also introduce a further sigmoidal refinement. Define
\(\widetilde{\widetilde{\DUAF}}_n:\mathbb R\to\mathbb R\) by
\[
\widetilde{\widetilde{\DUAF}}_n(x):=
\begin{cases}
\widetilde{\DUAF}_n(x),
& x\in(-\infty,0],
\\[0.8em]
\displaystyle\int_0^x \frac{c_n'\,\widetilde{\DUAF}_n(t)}{(1+t)^2}\,dt,
& x\in(0,\infty),
\end{cases}
\qquad
c_n':=
\left(
\int_0^\infty
\frac{\widetilde{\DUAF}_n(t)}{(1+t)^2}\,dt
\right)^{-1}.
\]
Then
\[
\widetilde{\widetilde{\DUAF}}_n\in C^n(\mathbb R),
\qquad
\lim_{x\to-\infty}\widetilde{\widetilde{\DUAF}}_n(x)=-1,
\qquad
\lim_{x\to+\infty}\widetilde{\widetilde{\DUAF}}_n(x)=1.
\]
Moreover, this refined activation is strictly increasing, with derivative vanishing
only at the transition points \(x=-9/2\) and \(x=0\) (see
Figure~\ref{fig:duaf-and-sigmoidal-duafs} for an illustration). An analogous fixed-size
Sobolev approximation result holds for
\(\widetilde{\widetilde{\DUAF}}_n\), with modified but explicit width and depth
bounds.

\subsection{Reduced complexity via structured targets}\label{sec:QKST:main}

Classical super-expressive constructions often start from one-dimensional approximation
and then use the Kolmogorov--Arnold Superposition Theorem to pass to higher dimensions
\citep{kolmogorov57}. This strategy is effective for uniform approximation, but it is
poorly suited to Sobolev approximation. The reason is that Kolmogorov-type
superpositions generally rely on highly irregular inner and outer functions, whose
singularities cancel only after composition. Thus, one should not expect arbitrary
Sobolev functions to admit superposition representations whose components have
comparable Sobolev regularity.

Moreover, even if one only asks for approximation by regular superpositions, there is
no general guarantee that a linear number of components suffices for arbitrary
Sobolev targets. The derivative interactions in high dimensions are typically too
rich to be captured by such a restricted structure. For this reason, we do not pursue
regular superposition approximations for all Sobolev functions. Instead, we isolate a
restricted class of targets, called the Kolmogorov Approximation Class, for which such
regular superposition approximations are available. For this class,
$\DUAF_{\infty}$--activated networks admit substantially smaller architectures.

\subsubsection*{The class $K_{d,Q}^s(\Omega)$}\label{def:qkst}
Let $d,s,Q\in\mathbb N$ and $\Omega\subset\mathbb R^{d}$ be an open region. For each $N\in\mathbb N$, let
$K_{d,Q,N}^s(\Omega)$ be the collection of all functions
$f\in W^{s,\infty}(\Omega)$ for which there exist
\[
    g_q,h_{q,p}\in W^{s,\infty}_{\mathrm{loc}}(\mathbb R),
    \qquad 
    q=1,\ldots,Q,\qquad p=1,\ldots,d,
\]
such that
\[
    \left\|
        f-
        \sum_{q=1}^Q 
        g_q\!\left(
            \sum_{p=1}^d h_{q,p}(x_p)
        \right)
    \right\|_{W^{s-1,\infty}(\Omega)}
    <\frac{1}{N}.
\]
We then define
\[
    K_{d,Q}^s(\Omega)
    :=
    \bigcap_{N=1}^{\infty}
    K_{d,Q,N}^s(\Omega).
\]

\subsubsection*{A PDE source of $K_{d,Q}^s(\Omega)$}
The class $K_{d,Q}^s(\Omega)$ is not merely an abstract superposition class;
it also contains natural solution families arising from separable linear
PDEs. In particular, finite-spectral solutions of classical
Sturm--Liouville-type boundary value problems provide concrete examples of
functions in $K_{d,Q}^s(\Omega)$ with finite $Q$.

Let $\Omega=I_1\times\cdots\times I_d$, $I_p=(a_p,b_p)$, and let $L_p$ be the
Dirichlet Sturm--Liouville operator
\[
    L_p\phi=-\frac{d}{dx_p}\!\left(A_p(x_p)\frac{d\phi}{dx_p}\right)
    +V_p(x_p)\phi,\qquad \phi(a_p)=\phi(b_p)=0.
\]
Assume that $A_p$ is uniformly positive and that the coefficients are regular
enough (for example, smooth) so that the eigenfunctions satisfy $\phi_{p,k}\in W^{s,\infty}(I_p)$.
Let $L_p\phi_{p,k}=\lambda_{p,k}\phi_{p,k}$ and set
$\Phi_{\bm{k}}(\bm{x}):=\prod_{p=1}^d\phi_{p,k_p}(x_p)$. For a finite
$\Lambda\subset\mathbb N^d$, define
$E_\Lambda:=\operatorname{span}\{\Phi_{\bm{k}}:\bm{k}\in\Lambda\}$.

Given $P\in\mathbb R[z_1,\ldots,z_d]$, let $\mathcal L:=P(L_1,\ldots,L_d)$.
Then
$\mathcal L\Phi_{\bm{k}}
=P(\lambda_{1,k_1},\ldots,\lambda_{d,k_d})\Phi_{\bm{k}}$.
Hence, if $\mu_{\bm{k}}:=P(\lambda_{1,k_1},\ldots,\lambda_{d,k_d})\neq0$
for all $\bm{k}\in\Lambda$, then $\mathcal L:E_\Lambda\to E_\Lambda$ is
invertible. Thus every $f\in E_\Lambda$ gives a unique finite-spectral
solution $u\in E_\Lambda$ of $\mathcal L u=f$.
We claim that $E_\Lambda\subset K_{d,2^d|\Lambda|}^s(\Omega)$. Indeed, for
$v\in E_\Lambda$, write
$v(\bm{x})=\sum_{\bm{k}\in\Lambda}c_{\bm{k}}\prod_{p=1}^d\phi_{p,k_p}(x_p)$.
By the polarization identity, each summand satisfies
\[
    c_{\bm{k}}\prod_{p=1}^d\phi_{p,k_p}(x_p)
    =
    \sum_{\varepsilon\in\{-1,1\}^d}
    \frac{c_{\bm{k}}}{2^d d!}
    \left(\prod_{p=1}^d\varepsilon_p\right)
    \left(\sum_{p=1}^d\varepsilon_p\phi_{p,k_p}(x_p)\right)^d.
\]
Equivalently, this is a sum of terms
$g_{\bm{k},\varepsilon}(\sum_{p=1}^d h_{\bm{k},\varepsilon,p}(x_p))$ with
$g_{\bm{k},\varepsilon}(t)=\frac{c_{\bm{k}}}{2^d d!}
(\prod_{p=1}^d\varepsilon_p)t^d$ and
$h_{\bm{k},\varepsilon,p}(t)=\varepsilon_p\phi_{p,k_p}(t)$. Since
$\phi_{p,k_p}\in W^{s,\infty}(I_p)$, these inner functions may be extended
to $W^{s,\infty}(\mathbb R)$, while the outer functions are polynomials.
Thus $v$ has an exact superposition representation with at most
$2^d|\Lambda|$ terms, and hence $v\in K_{d,2^d|\Lambda|}^s(\Omega)$.
Consequently, every finite-spectral solution $u\in E_\Lambda$ of
$\mathcal L u=f$ belongs to $K_{d,2^d|\Lambda|}^s(\Omega)$.

For instance, take $\Omega=(0,1)^d$ and $L_p=-d^2/dx_p^2$ with Dirichlet
boundary conditions. Then $\phi_{p,k}(x_p)=\sin(k\pi x_p)$ and
$\lambda_{p,k}=k^2\pi^2$. Choosing
$P(z_1,\ldots,z_d)=z_1+\cdots+z_d$ gives the classical Dirichlet Poisson
problem
\[
    -\Delta u=f
    \quad\text{in }(0,1)^d,
    \qquad
    u=0
    \quad\text{on }\partial(0,1)^d.
\]
If, for instance, $d=2$ and
\[
    f(x,y)=b_1\sin(\pi x)\sin(2\pi y)+b_2\sin(3\pi x)\sin(\pi y),
\]
then
\[
\begin{aligned}
u(x,y)
=
\sum_{\varepsilon=\pm1}
\frac{\varepsilon b_1}{20\pi^2}
\bigl(\sin(\pi x)+\varepsilon\sin(2\pi y)\bigr)^2
+
\sum_{\varepsilon=\pm1}
\frac{\varepsilon b_2}{40\pi^2}
\bigl(\sin(3\pi x)+\varepsilon\sin(\pi y)\bigr)^2.
\end{aligned}
\]
Thus $u\in K_{2,4}^s((0,1)^2)$.

The preceding discussion shows that $K_{d,Q}^s(\Omega)$ contains natural
finite-spectral solution families of separable linear PDEs. For this structured
class, one obtains a substantially sharper fixed-size Sobolev approximation
result. In contrast to general fixed-size constructions whose width may grow
rapidly with the dimension, the width below scales only linearly with $dQ$.
Moreover, the depth is independent of the dimension and depends only linearly
on the smoothness order $s$. 

\begin{theorem}
\label{thm:QKST}
Let $d, s\in\mathbb{N}$, and let $f\in K_{d,Q}^s((a,b)^d)$. 
Then, for any $\varepsilon>0$, there exists a $\DUAF_\infty$-activated network $\Phi$ 
with width no greater than $2dQ(s^2+9s+10)$ and depth no greater than $\max\{4s, 14\}$ such that
\[
\|\Phi-f\|_{W^{s-1,\infty}((a,b)^d)}<\varepsilon.
\]    
\end{theorem}

The proof of Theorem~\ref{thm:QKST} can be found in Section~\ref{sec:proof-qkst}. Consequently, this theorem also provides an efficient approximation guarantee for a nontrivial class of PDE solutions.

\section{Preliminaries and proof roadmap}

\subsection{Basic notations}
\label{sec:notation}

Below is a summary of the fundamental notations used throughout this paper.

\begin{itemize}
    \item The difference between two sets \( A \) and \( B \) is denoted by \( A \backslash B \coloneqq \{ x : x \in A, \ x \notin B \} \).
    
    \item The symbols \( \mathbb{N} \), \( \mathbb{Z} \), \( \mathbb{Q} \), and \( \mathbb{R} \) represent the sets of natural numbers, integers, rational numbers, and real numbers, respectively. We also denote $\mathbb N_{0}=\mathbb N\cup \{0\}$.
    
    \item The floor and ceiling functions of a real number \( x \) are given by
    \( \lfloor x \rfloor = \max \{ n : n \le x, \ n \in \mathbb{Z} \} \) and \( \lceil x \rceil = \min \{ n : n \ge x, \ n \in \mathbb{Z} \} \).
    \item For a $d$-dimensional multi-index 
    $\bm{\alpha}=(\alpha_1,\alpha_2,\ldots,\alpha_d)\in \mathbb N_0^d$, 
    we use the following notation:
    \[
        |\bm{\alpha}|
        := \alpha_1+\alpha_2+\cdots+\alpha_d,
    \]
    \[
        \bm{x}^{\bm{\alpha}}
        := x_1^{\alpha_1}x_2^{\alpha_2}\cdots x_d^{\alpha_d},
        \qquad 
        \bm{x}=(x_1,x_2,\ldots,x_d)^\top,
    \]
    and
    \[
        \bm{\alpha}!
        := \alpha_1!\alpha_2!\cdots \alpha_d!.
    \]
    \item For any \( p \in [1, \infty] \), the \( p \)-norm (or \( \ell^p \)-norm) of a vector \( \bm{x} = (x_1, \dots, x_d) \in \mathbb{R}^d \) is defined as
    \begin{equation*}
        \|\bm{x}\|_p = \|\bm{x}\|_{\ell^p} \coloneqq \big( |x_1|^p + \dots + |x_d|^p \big)^{1/p}\quad \text{for } p \in [1, \infty),
    \end{equation*}
    and
    \begin{equation*}
        \|\bm{x}\|_{\infty} = \|\bm{x}\|_{\ell^\infty} \coloneqq \max\big\{ |x_i| : i = 1, 2, \dots, d \big\}.
    \end{equation*}
    \item Let $\overline{B}_{r,|\cdot|}(\bm{x})\subset \mathbb R^d$ denote the closed ball centered at 
    $\bm{x}\in\mathbb R^d$ with radius $r>0$ measured by the Euclidean norm. 
    Similarly, let $\overline{B}_{r,\|\cdot\|_{\ell^\infty}}(\bm{x})\subset\mathbb R^d$ denote the closed ball centered at 
    $\bm{x}\in\mathbb R^d$ with radius $r>0$ measured by the $\ell^\infty$-norm.
     
    \item Let $\bm{n}\in\mathbb N^m$. We write $f(\bm{n})=\mathcal{O}(g(\bm{n}))$ if there exists a positive constant $C$, independent of $\bm{n}$, such that
    \[
        f(\bm{n})\le Cg(\bm{n}),
    \]
    as all components of $\bm{n}$ tend to $+\infty$.
\end{itemize}

\subsection{Neural networks}
\label{sec:DNN}

We summarize the basic notation used for deep neural networks as follows.

\begin{enumerate}

   \item Let $\sigma:\mathbb R\to\mathbb R$ be an activation function. 
    We also use $\bm{\sigma}$ to denote its componentwise 
extension to vectors. Namely, for any 
$\bm{x}=(x_1,\ldots,x_d)^\top\in\mathbb R^d$, we define
\[
    \bm{\sigma}(\bm{x})
    :=
    \begin{pmatrix}
        \sigma(x_1)\\
        \vdots\\
        \sigma(x_d)
    \end{pmatrix}.
\]
Thus, each activation layer applies the same scalar activation function to every 
component of its input vector.

    \item Let $L\in\mathbb N$, $N_0=d$, $N_{L+1}=1$, 
    $N_\ell\in\mathbb N$ for $\ell=1,2,\ldots,L$, and $N:=\max_{0\le i \le L+1}N_i$.

    An $\sigma$-activated network $\phi:\mathbb R^d\to\mathbb R$ with width $N$ and depth $L$ is described by
    \[
        \bm{x}=\widetilde{\bm{h}}_0
        \xrightarrow{\bm{W}_1,\bm{b}_1}
        \bm{h}_1
        \xrightarrow{\bm{\sigma}}
        \widetilde{\bm{h}}_1
        \xrightarrow{\bm{W}_2,\bm{b}_2}
        \cdots
        \xrightarrow{\bm{W}_L,\bm{b}_L}
        \bm{h}_L
        \xrightarrow{\bm{\sigma}}
        \widetilde{\bm{h}}_L
        \xrightarrow{\bm{W}_{L+1},\bm{b}_{L+1}}
        \phi(\bm{x})=\bm{h}_{L+1}.
    \]
    Here
    \[
        \bm{W}_\ell\in\mathbb R^{N_\ell\times N_{\ell-1}},
        \qquad 
        \bm{b}_\ell\in\mathbb R^{N_\ell},
    \]
    are the weight matrix and bias vector in the $\ell$-th affine transform, respectively. Namely,
    \[
        \bm{h}_\ell
        :=
        \bm{W}_\ell \widetilde{\bm{h}}_{\ell-1}+\bm{b}_\ell,
        \qquad 
        \ell=1,\ldots,L+1,
    \]
    and
    \[
        \widetilde{\bm{h}}_\ell
        :=
        \bm{\sigma}(\bm{h}_\ell),
        \qquad 
        \ell=1,\ldots,L.
    \]
\end{enumerate}

\subsection{Sobolev spaces}

Let $\Omega\subset\mathbb R^d$ be an open set. For a multi-index
$\bm{\alpha}=(\alpha_1,\ldots,\alpha_d)\in\mathbb N_0^d$, we write
\[
    |\bm{\alpha}|:=\alpha_1+\cdots+\alpha_d,
    \qquad
    D^{\bm{\alpha}}
    :=
    D_1^{\alpha_1}\cdots D_d^{\alpha_d},
\]
where $D_j$ denotes the weak derivative with respect to the $j$-th variable.

\begin{definition}[Sobolev spaces]
Let $s\in\mathbb N_0$ and $1\le p\le \infty$. The Sobolev space
$W^{s,p}(\Omega)$ is defined by
\[
    W^{s,p}(\Omega)
    :=
    \left\{
        f\in L^p(\Omega):
        D^{\bm{\alpha}}f\in L^p(\Omega)
        \text{ for all }
        \bm{\alpha}\in\mathbb N_0^d
        \text{ with } |\bm{\alpha}|\le s
    \right\}.
\]
It is equipped with the norm
\[
\|f\|_{W^{s,p}(\Omega)}
:=
\left\{
\begin{array}{ll}
\displaystyle
\biggl(
\sum_{|\bm{\alpha}|\le s}
\|D^{\bm{\alpha}}f\|_{L^p(\Omega)}^p
\biggr)^{1/p},
& 1\le p<\infty,\\[1.1em]

\displaystyle
\max_{|\bm{\alpha}|\le s}
\|D^{\bm{\alpha}}f\|_{L^\infty(\Omega)},
& p=\infty.
\end{array}
\right.
\]
The corresponding Sobolev semi-norm is defined by
\[
|f|_{W^{s,p}(\Omega)}
:=
\left\{
\begin{array}{ll}
\displaystyle
\biggl(
\sum_{|\bm{\alpha}|=s}
\|D^{\bm{\alpha}}f\|_{L^p(\Omega)}^p
\biggr)^{1/p},
& 1\le p<\infty,\\[1.1em]

\displaystyle
\max_{|\bm{\alpha}|=s}
\|D^{\bm{\alpha}}f\|_{L^\infty(\Omega)},
& p=\infty.
\end{array}
\right.
\]
\end{definition}

\subsubsection*{Averaged Taylor polynomials.}
Let $\Omega\subset\mathbb R^d$ be an open subset,
$s\in\mathbb N$, and $f\in W^{s,\infty}(\Omega)$. For a ball
$B:=\overline{B}_{r,|\cdot|}(\bm{x}_0)\subset \Omega$, let $b_B\in C_c^\infty(B)$ satisfy
$b_B\ge 0$ and
\[
    \int_B b_B(\bm{y})\,d\bm{y}=1.
\]
For $\bm{y}\in B$, define
\[
    T_{\bm{y}}^s f(\bm{x})
    :=
    \sum_{|\bm{\alpha}|\le s-1}
    \frac{D^{\bm{\alpha}}f(\bm{y})}{\bm{\alpha}!}
    (\bm{x}-\bm{y})^{\bm{\alpha}},
    \qquad
      Q_B^s f(\bm{x})
    :=
    \int_B T_{\bm{y}}^s f(\bm{x})b_B(\bm{y})\,d\bm{y}.
\]
Then $Q_B^s f$ is a polynomial of degree at most $s-1$, namely
\[
    Q_B^s f(\bm{x})
    =
    \sum_{|\bm{\alpha}|\le s-1}
    c_{f,\bm{\alpha}}\bm{x}^{\bm{\alpha}}.
\]
Moreover, if $B\subset (0,1)^d$, then
\[
    |c_{f,\bm{\alpha}}|
    \le
    C_{2}(s,d)\,
    \|f\|_{W^{s-1,\infty}(B)},
    \quad 
    |\bm{\alpha}|\le s-1,
    \qquad
        C_{2}(s,d)
    :=
    \sum_{|\bm{\alpha}+\bm{\beta}|\le s-1}
    \frac{1}{\bm{\alpha}!\bm{\beta}!}.
\]
$Q_B^s f$ is called the averaged Taylor polynomial of order $s$ over $B$.

\subsubsection*{Bramble--Hilbert estimate.}
Let $\Omega\subset\mathbb R^d$ be open and bounded. We say that $\Omega$ is
star-shaped with respect to a set $B\subset\Omega$ if
\[
    \operatorname{conv}(\{\bm{x}\}\cup B)\subset \Omega,
    \qquad \forall\,\bm{x}\in\Omega,
\]
where $\operatorname{conv}$ denotes the convex hull. Define
\[
    r_{\max}^*
    :=
    \sup\left\{
        r>0:
        \Omega \text{ is star-shaped with respect to }
        \overline{B}_{r,|\cdot|}(\bm{x}_0)
        \text{ for some } \bm{x}_0\in\Omega
    \right\}.
\]
We call
\[
    \gamma
    :=
    \frac{\operatorname{diam}(\Omega)}{r_{\max}^*}
\]
the chunkiness parameter of $\Omega$.

\begin{lemma}[Bramble--Hilbert Lemma {\normalfont\citep[Theorem~1]{brambleHilbert70}}]\label{lem:bramble-hilbert}
Let $\Omega\subset\mathbb R^d$ be open and bounded, and suppose that $\Omega$ is
star-shaped with respect to $B:=\overline{B}_{r,|\cdot|}(\bm{x}_0)\subset\Omega$ with
$r\ge \frac12 r_{\max}^*$. Let $\gamma$ be the chunkiness parameter of $\Omega$ and
$h:=\operatorname{diam}(\Omega)$. Then, for every $s\in\mathbb N$ and
$1\le p\le\infty$, there exists a constant $C_{\mathrm{BH}}(s,d,\gamma)>0$ such that
for every $f\in W^{s,p}(\Omega)$,
\[
    |f-Q_B^s f|_{W^{r,p}(\Omega)}
    \le
    C_{\mathrm{BH}}(s,d,\gamma)\,
    h^{s-r}
    |f|_{W^{s,p}(\Omega)},
    \qquad
    r=0,1,\ldots,s.
\]
Here $Q_B^s f$ denotes the averaged Taylor polynomial of order $s$ over $B$.
\end{lemma}

\subsection{Proof ideas and roadmap}

In dimension one the proof can be viewed as a fixed-size pipeline; see
Figure~\ref{fig:proof-idea-1d}.  Let
\[
    C_k=\Big[\frac{k}{K},\frac{k+1}{K}\Big],\qquad
    I_{1,k}=\Big[\frac{k}{K},\frac{2k+1}{2K}\Big],
    \qquad
    \Omega_1=\bigcup_{k=0}^{K-1} I_{1,k}.
\]
Only the left halves \(I_{1,k}\) are used in one chart.  This separation is
essential: the coefficient row must change from \(k\) to \(k+1\), and these
jumps are kept inside the unused half-cells rather than in the region where the
Taylor polynomial is evaluated.

\begin{figure}[htbp]
\centering
\includegraphics[width=0.98\textwidth]{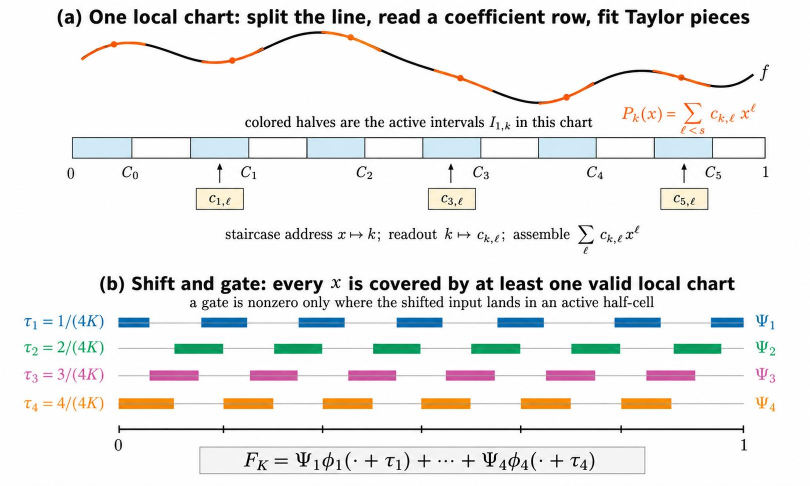}
\caption{One-dimensional cartoon of the construction.  A single local chart fits
averaged Taylor pieces only on active half-intervals.  Four shifted charts,
weighted by realizable gates, cover the target interval and glue these local
pieces into a global approximant}
\label{fig:proof-idea-1d}
\end{figure}

The local chart has three \DUAF-realizable steps.  First, on each active interval,
Bramble--Hilbert gives an averaged Taylor piece
\[
    P_k(x)=\sum_{\ell=0}^{s-1} c_{k,\ell}x^\ell,
    \qquad
    \|f-P_k\|_{W^{r,\infty}(I_{1,k})}\le C K^{r-s},
    \quad 0\le r\le s .
\]
Second, a staircase subnet realizes the address \(x\mapsto k\) on \(I_{1,k}\),
and a fixed-width point-fitting readout realizes
\[
    k\longmapsto c_{k,\ell},\qquad 0\le \ell\le s-1 .
\]
Third, multiplication subnets assemble
\[
    \phi_1(x)\approx \sum_{\ell<s} c_{k,\ell}x^\ell
    \quad\text{on } I_{1,k}.
\]
Choosing the readout accuracy \(K^{-s}\) gives
\[
    \|f-\phi_1\|_{W^{r,\infty}(\Omega_1)}
    \le C K^{r-s},
    \qquad 0\le r\le s-1 .
\]
Thus \(K\) enlarges the coefficient table, but the address, readout, and assembly
modules have fixed architecture.

The unused halves are covered by four shifted copies.  Set
\(\tau_j=j/(4K)\), \(j=1,2,3,4\).  For each shift, construct a local network
\(\phi_j\) for \(f(\cdot-\tau_j)\) on the shifted active set.  The gate
\(\Psi_{j,n,K}\) is a \DUAF network with
\[
    \sum_{j=1}^4 \Psi_{j,n,K}=1,\qquad
    \operatorname{supp}\Psi_{j,n,K}\subset\{x:\ x+\tau_j\in\Omega_1\},
    \qquad
    \|\Psi_{j,n,K}^{(m)}\|_\infty\le C K^m .
\]
The global approximant is
\[
    F_K(x)=\sum_{j=1}^4
        \Psi_{j,n,K}(x)\,\phi_j(x+\tau_j).
\]
By the support property, each local error is only used where the corresponding
shifted chart is valid.  Leibniz's rule then gives, for \(0\le r\le s-1\),
\[
    \|f-F_K\|_{W^{r,\infty}}
    \le
    C\sum_{m=0}^r K^m K^{r-m-s}
    \le C K^{r-s}\to0 .
\]

The formal $d$-dimensional argument is the same construction with rectangles,
multi-indices, and $4^d$ shifts.  The full proof architecture is summarized as
follows.

\begin{figure}[htbp]
\centering
\footnotesize
\setlength{\tabcolsep}{3.2pt}
\renewcommand{\arraystretch}{1.15}
\fbox{%
\begin{minipage}{0.985\textwidth}
\vspace{0.25em}
\begin{tabularx}{\textwidth}{@{}>{\bfseries\raggedright\arraybackslash}p{0.155\textwidth}>{\raggedright\arraybackslash}p{0.325\textwidth}>{\raggedright\arraybackslash}p{0.315\textwidth}>{\raggedright\arraybackslash}p{0.16\textwidth}@{}}
\rowcolor{mygray}
\textbf{Step} & \textbf{Concrete operation} & \textbf{Tool and estimate} & \textbf{Delivers}\\
\midrule
Taylor data &
Cut $[0,1]^d$ into $K^d$ cells.  On each active half-cell $Q_{\bm{i},\bm{m}_*}$,
store the coefficients $c_{\bm{i},\bm{\alpha}}$ of an averaged Taylor polynomial. &
Bramble--Hilbert and Theorem~\ref{thm:taylor-g-alpha-K} give the local error
$K^{r-s}$ and bounded coefficient tables $g_{f,\bm{\alpha},\bm{m}_*}$. &
Taylor data for Proposition~\ref{prop:local}.\\
\addlinespace[0.2em]
Address &
For an input $\bm{x}$, the staircase subnet returns the cell indices
$(i_1,\ldots,i_d)$ and encodes them as one integer
$p=\sum_j i_jK^{j-1}$. &
Lemma~\ref{lemma:steps}.  The map has fixed width and depth; only its slopes and
biases depend on $K$. &
A table row selected by $\bm{x}$.\\
\addlinespace[0.2em]
Readout &
For each multi-index $\bm{\alpha}$, a width-one subnet maps the integer $p+1$ to
the prescribed coefficient $c_{\bm{i},\bm{\alpha}}$. &
Lemmas~\ref{lemma:dense} and~\ref{lemma:point-fit}; for $\DUAF_n$, use
Lemma~\ref{lem:sigma_n-properties}. &
Fixed-size realization of $g_{f,\bm{\alpha},\bm{m}_*}$.\\
\addlinespace[0.2em]
Assemble &
Multiply the read coefficient by $\bm{x}^{\bm{\alpha}}$ and sum over
$|\bm{\alpha}|\le s-1$. &
Lemmas~\ref{lemma:multi} and~\ref{lemma:poly}, plus the cellwise product bound
\eqref{eq:cellwise-product-bound}. &
Local network: Proposition~\ref{prop:local}.\\
\addlinespace[0.2em]
Shift--glue &
The local lookup is valid only on active half-cells, so run it on $4^d$ shifted
grids and glue the copies by gates $\Psi_{\bm{i},n,K}$. &
Lemma~\ref{lem:gn-properties}: partition of unity, localization, and
$\|D^{\bm{\gamma}}\Psi_{\bm{i},n,K}\|_\infty\lesssim K^{|\bm{\gamma}|}$.  With local tolerance $K^{-s}$,
Leibniz gives an $O(K^{-1})$ Sobolev error. &
Theorems~\ref{thm:shift-W1-d} and~\ref{thm:shift-Ws-d}.\\
\addlinespace[0.2em]
Rescale &
Extend and normalize $f$ on $(0,1)^d$, approximate on $(0,9/10)^d$, then apply an
affine change of variables back to $(a,b)^d$. &
Sobolev extension, normalization, and affine rescaling. &
Theorems~\ref{thm:global-w2}, \ref{thm:global-ws}; Corollary~\ref{cor:global-winfty}.\\
\midrule
Structured &
Do not build a $K^d$ table.  Approximate the univariate functions $h_{q,p}$ and
$g_q$, then compose $g_q(\sum_p h_{q,p})$. &
Corollary~\ref{cor:global-winfty} in dimension one; Fa\`a di Bruno controls
Sobolev errors through the compositions. &
Theorem~\ref{thm:QKST}.\\
\addlinespace[0.2em]
Sigmoidal &
Do not redo the grid proof.  First simulate the scalar activation $\DUAF_n$ by a
fixed $\widetilde{\DUAF}_n$-subnet, then replace every activation layer. &
Theorem~\ref{thm:sigmoidal-loc}: the middle branch is exact up to scaling and the
tails are recovered from the integral definition by finite differences; Lemma~\ref{lem:old-to-new-1} performs layer replacement. &
Theorem~\ref{thm:sigmoidal}.\\
\bottomrule
\end{tabularx}
\vspace{0.25em}
\end{minipage}}
\caption{Proof blueprint.  The number of cells grows with $K$, but the modules in
this table do not; only their parameters change}
\label{fig:proof-roadmap}
\end{figure}

The key estimate is in the gluing step.  If a derivative $D^{\bm{\alpha}}$ falls
partly on the local error and partly on a gate, the typical term is
$(K^{-s}+K^{|\bm{\beta}|-s})K^{|\bm{\alpha}|-|\bm{\beta}|}$, which is $O(K^{-1})$
for $|\bm{\alpha}|\le s-1$.  Hence increasing $K$ improves the Sobolev accuracy
without changing the architecture.  The structured theorem replaces the
$d$-dimensional coefficient table by univariate tables, while the sigmoidal theorem
keeps the same proof and only replaces the activation.

\section{Algebraic and analytic properties of \DUAF{s}} \label{sec:fundamental-lemmas}

This section collects the fixed-size primitives used in the proof roadmap.  The
lemmas are ordered according to their role in the construction: first we encode
finite tables by one scalar parameter, then we turn this encoding into neural
network primitives, then we build algebraic operations, and finally we record the
smooth gates needed for the local-to-global Sobolev estimates.

The first lemma is the arithmetic/topological source of the whole coefficient
encoding mechanism.  It says that a single parameter $w$ can drive a periodic
function through a dense subset of an arbitrary finite-dimensional cube.  Later,
this is used to encode all Taylor-coefficient values on a grid without increasing
the network width.

\begin{lemma}[Denseness of one-parameter periodic tables]\label{lemma:dense}
	Let $g:\mathbb R \to \mathbb R$ be a continuous periodic function with $\min_{x\in\mathbb R}g(x)=0$ and $\max_{x\in\mathbb R}g(x)=1$. Then, for any $N\in \mathbb N$, the following point set
	\begin{equation*}
	\Bigg\{\Big[g\Big(\tfrac{w}{2^{1/N}+1+\frac{1}{N}}\Big),\   
	g\Big(\tfrac{w}{2^{1/N}+1+\frac{2}{N}}\Big),\   \cdots,\  
	g\Big(\tfrac{w}{2^{1/N}+2}\Big)\Big]^T\,  :\, 
	w\in\mathbb R   \Bigg\}\subseteq [0,1]^N
\end{equation*}
is dense in $[0,1]^N$.
\end{lemma}

\begin{proof}  
We split the proof into three steps. The first step shows the rational independence of a particular sequence. The second step demonstrates the general denseness phenomenon of the fractional curve generated by an arbitrary rational independent sequence on the torus. The last step generalizes the fractional curve to arbitrary continuous periodic functions and shows the denseness result of their images in the unit cube.

\noindent\textbf{Step 1}: Rational Independence of $\Big(\tfrac{1}{2^{1/N}+1+\frac{1}{N}}, \tfrac{1}{2^{1/N}+1+\frac{2}{N}},\ldots,\tfrac{1}{2^{1/N}+2}\Big)$

Let $\alpha_i=\tfrac{1}{2^{1/N}+1+\frac{i}{N}}$, we claim that $(\alpha_1,\dots,\alpha_N)$ are linearly independent over $\mathbb Q$.
Indeed, suppose $c_1,\dots,c_N\in\mathbb Q$ satisfy
$\sum_{i=1}^N c_i\alpha_i=0$, i.e.
$\sum_{i=1}^N \frac{c_i}{2^{1/N}+1+i/N}=0$.
Multiplying both sides by $\prod_{k=1}^N(2^{1/N}+1+\frac{k}{N})$ yields
$P(2^{1/N})=0$, where the polynomial $P\in\mathbb Q[x]$ is given explicitly by
\[
P(x)
=
\sum_{j=1}^{N}
c_j
\prod_{\substack{k=1\\ k\ne j}}^{N}\Big(x+1+\frac{k}{N}\Big).
\]
Clearly $\deg P\le N-1$.
Since the polynomial $x^{N}-2$ is irreducible over $\mathbb Q$ by Eisenstein's criterion, we must have $[\mathbb Q(2^{1/N}):\mathbb Q]=N$ and hence $P\equiv 0$.  For $1\le i\le N$, we have $c_i\frac{1}{N^{N-1}}\prod_{k\neq i}(k-i)=P(-1-\frac{i}{N})=0$. Since $\prod_{k\neq i}(k-i)\neq0$, this forces $c_i=0$. 

\smallskip
\noindent\textbf{Step 2}: Denseness on $\mathbb T^N$.

Suppose now $\bm{a}:=(a_1,\dots,a_N)\in\mathbb R^N$ is rationally independent. For $r\in\mathbb R$, let $\{r\}\in[0,1)$ denote its fractional part. We claim that the following set
\[
\Big\{\big[\{a_1x\},\dots,\{a_Nx\}\big]^{T}\,  :\, 
	x\in\R   \Big\}\subseteq [0,1)^N
\]
is dense in $[0,1]^N$.

In order to prove this claim, we translate the original space into $N$-torus and analyze the topological properties there. Let $\mathbb T^N:=\mathbb R^N/\mathbb Z^N$ be the $N$-torus and let $\pi:\mathbb R^N\to\mathbb T^N$ be the quotient map.
We define the group homomorphism and its induced homomorphism
\[
L:\mathbb R\to\mathbb R^N,\qquad L(x):=x\bm{a}, \qquad \phi:=\pi\circ L:\mathbb R\to\mathbb T^N.
\]
Set
\[
H:=\overline{\phi(\mathbb R)}\subset\mathbb T^N.
\]
Since $\phi$ is a group homomorphism, the set $\phi(\mathbb R)$ is a subgroup of the topological group $\mathbb T^N$, and therefore
$H$ is a closed subgroup of $\mathbb T^N$.

For each $\bm{m}\in\mathbb Z^N$, define a character $\chi_{\bm{m}}:\mathbb T^N\to \mathbb S^1$ by $\chi_{\bm{m}}(\pi(\bm{t})):=e^{2\pi i\, \bm{m}\cdot \bm{t}}$, for $\bm{t}\in\mathbb R^N$.
This is well defined since $\bm{t}-\bm{t'}\in\mathbb Z^N$ implies $\bm{m}\cdot(\bm{t}-\bm{t'})\in\mathbb Z$.

Assume, for contradiction, that $H\subsetneq \mathbb T^N$.
Consider the quotient compact abelian group $Q:=\mathbb T^N/H$, and let $\rho:\mathbb T^N\to Q$ be the quotient homomorphism.
Since $H$ is a proper closed subgroup, the quotient $Q$ is nontrivial.
For a locally compact abelian group $G$, define its Pontryagin dual by
\[
\widehat G
:=
\{\chi:G\to \mathbb S^1 \mid \chi \text{ is a continuous group homomorphism}\},
\]
equipped with the compact--open topology. Pontryagin duality \citep[Theorem~1.7.2, \S1.7]{rudin62fourier}, which goes back to Pontrjagin \citep{pontryagin34}, says that the canonical evaluation map
\[
\iota:G\to\widehat{\widehat G},\qquad
\iota(x)(\chi):=\chi(x),
\]
is a topological group isomorphism.
In particular, if $G$ is a compact abelian group, then
$G$ is trivial if and only if $\widehat G$ is trivial.
By applying this result to our compact abelian group $Q$, it follows immediately that
$\widehat Q\neq\{1\}$, and hence we may choose a nontrivial character $\psi\in\widehat Q$, i.e., $\psi\neq 1$.
Define $\chi:=\psi\circ\rho\in\widehat{\mathbb T^N}$.
Then $\chi\neq 1$ and $\chi(h)=\psi(\rho(h))=\psi(0)=1$ for all $h\in H$, so $\chi$ is a nontrivial character of $\mathbb T^N$
that is identically $1$ on $H$.

It is standard (follows by lifting characters to $\mathbb R^N$ and using $\mathbb Z^N$-periodicity) that
\[
\widehat{\mathbb T^N}\cong \mathbb Z^N,\qquad
\chi=\chi_{\bm{m}} \ \text{for a unique } \bm{m}\in\mathbb Z^N,
\]
and $\chi$ is nontrivial if and only if $\bm{m}\neq \bm{0}$.
Thus there exists $\bm{m}\in\mathbb Z^N\setminus\{0\}$ such that $\chi_{\bm{m}}(h)=1$ for all $h\in H$.

In particular, for every $x\in\mathbb R$ we have $\phi(x)\in H$, hence
\[
1=\chi_{\bm{m}}(\phi(x))
=\chi_{\bm{m}}(\pi(x\bm{a}))
=e^{2\pi i\, \bm{m}\cdot(x\bm{a})}
=e^{2\pi i\, x(\bm{m}\cdot \bm{a})}.
\]
This forces $\bm{m}\cdot \bm{a}=0$.
By the rational independence of $\bm{a}=(a_1,\dots,a_N)$, this implies $\bm{m}=\bm{0}$, a contradiction.
Hence $H=\mathbb T^N$, i.e. $\phi(\mathbb R)$ is dense in $\mathbb T^N$ (Figure~\ref{fig:torus-density} illustrates the denseness phenomenon).

Next, we consider the map $f:\mathbb R\to[0,1)^N$,where $f(x):=[\{a_1x\},\dots,\{a_Nx\}]^{T}$.
By the definitions, $L(x)-f(x)\in\mathbb Z^N$ for every $x\in\mathbb R$, hence $\phi=\pi\circ f$. 

Let $U\subset(0,1)^N$ be an arbitrary nonempty open set.
Since the restriction $\pi|_{[0,1)^N}$ is injective, we have
\[
\pi^{-1}(\pi(U))\cap[0,1)^N = U.
\]
Because $\pi(f(\mathbb R))=\phi(\mathbb R)$ is dense in $\mathbb T^N$ and $\pi$ is a covering map and so an open map, there exists
$x\in\mathbb R$ such that $\pi(f(x))\in\pi(U)$.
As $f(x)\in[0,1)^N$, it follows that
\[
f(x)\in \pi^{-1}(\pi(U))\cap[0,1)^N = U.
\]
Hence $f(\mathbb R)$ intersects every nonempty open subset of $(0,1)^N$, and therefore
$f(\mathbb R)$ is dense in $(0,1)^N$, and thus also dense in $[0,1]^N$. This proves the claim.

\begin{figure}[htbp]
  \centering
  \includegraphics[width=0.33\textwidth]{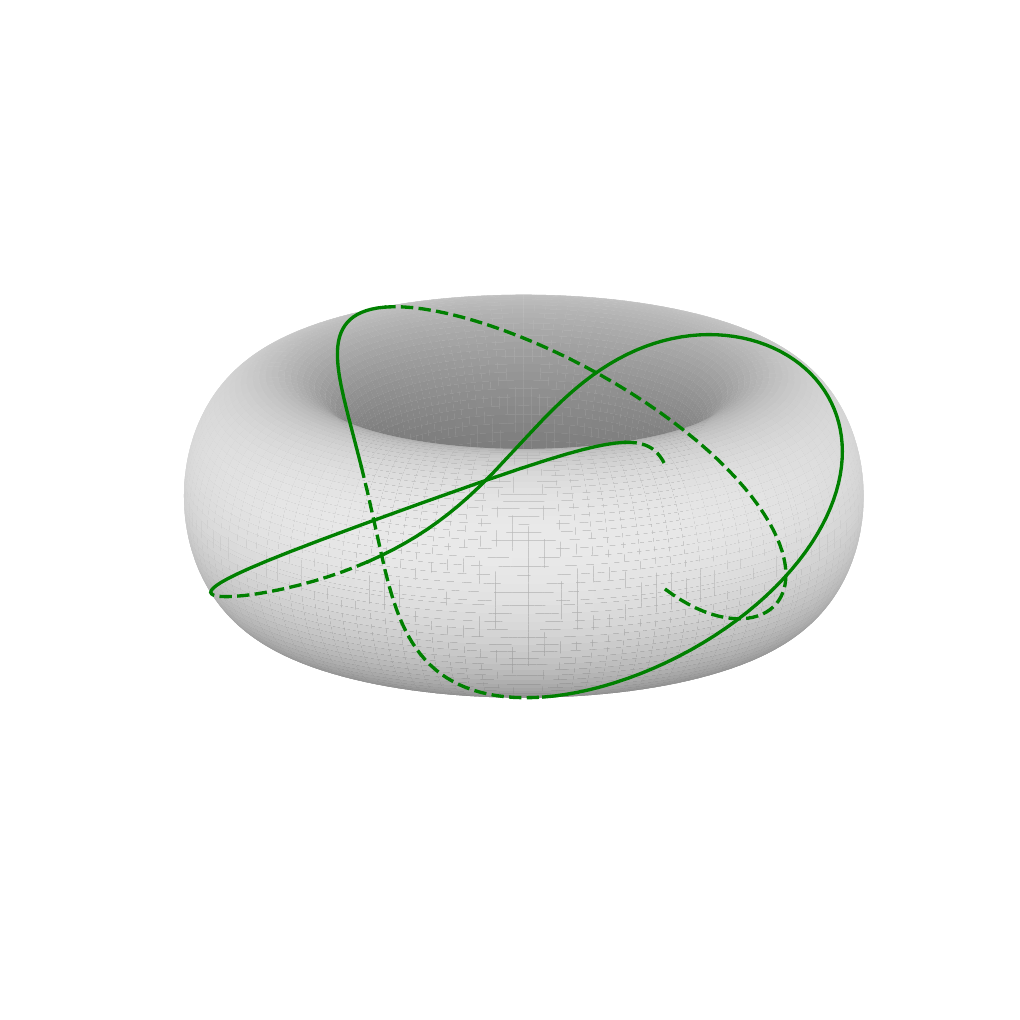}\hfill
  \includegraphics[width=0.33\textwidth]{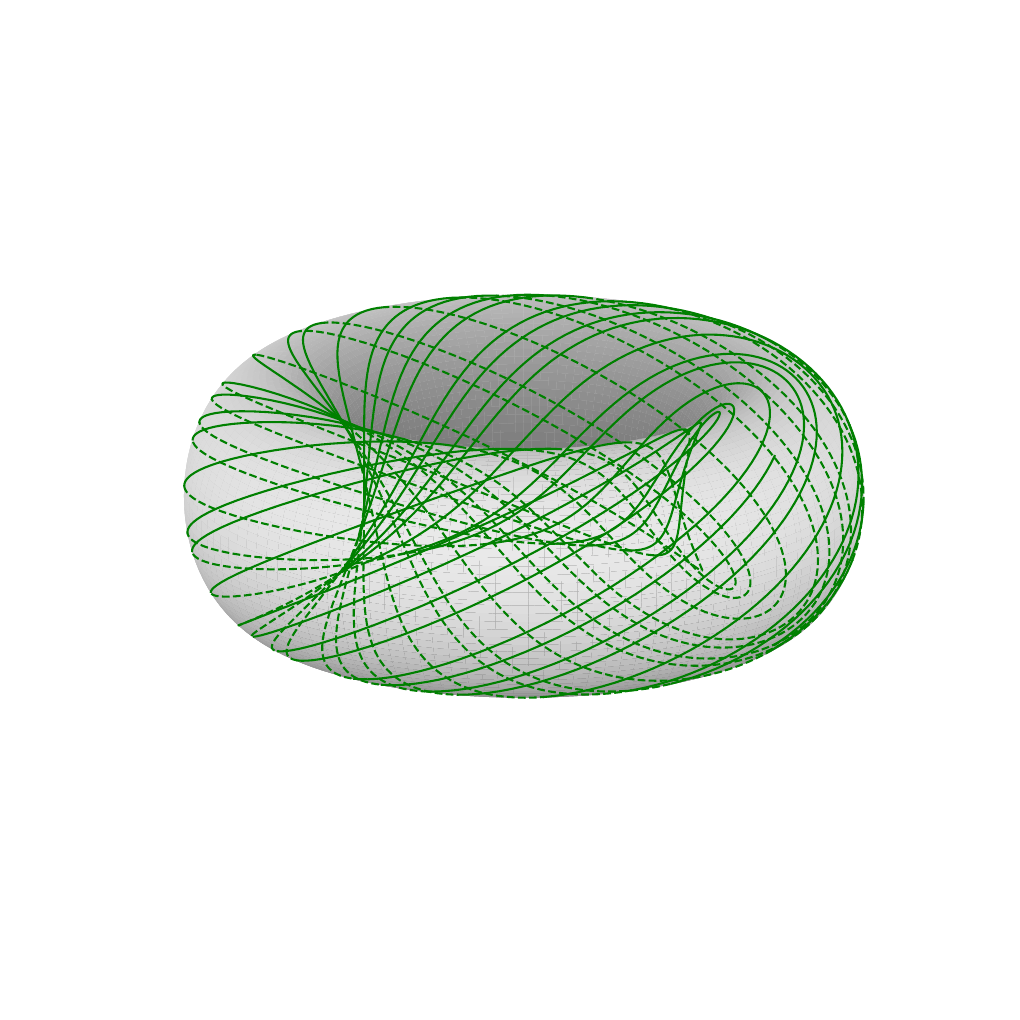}\hfill
  \includegraphics[width=0.33\textwidth]{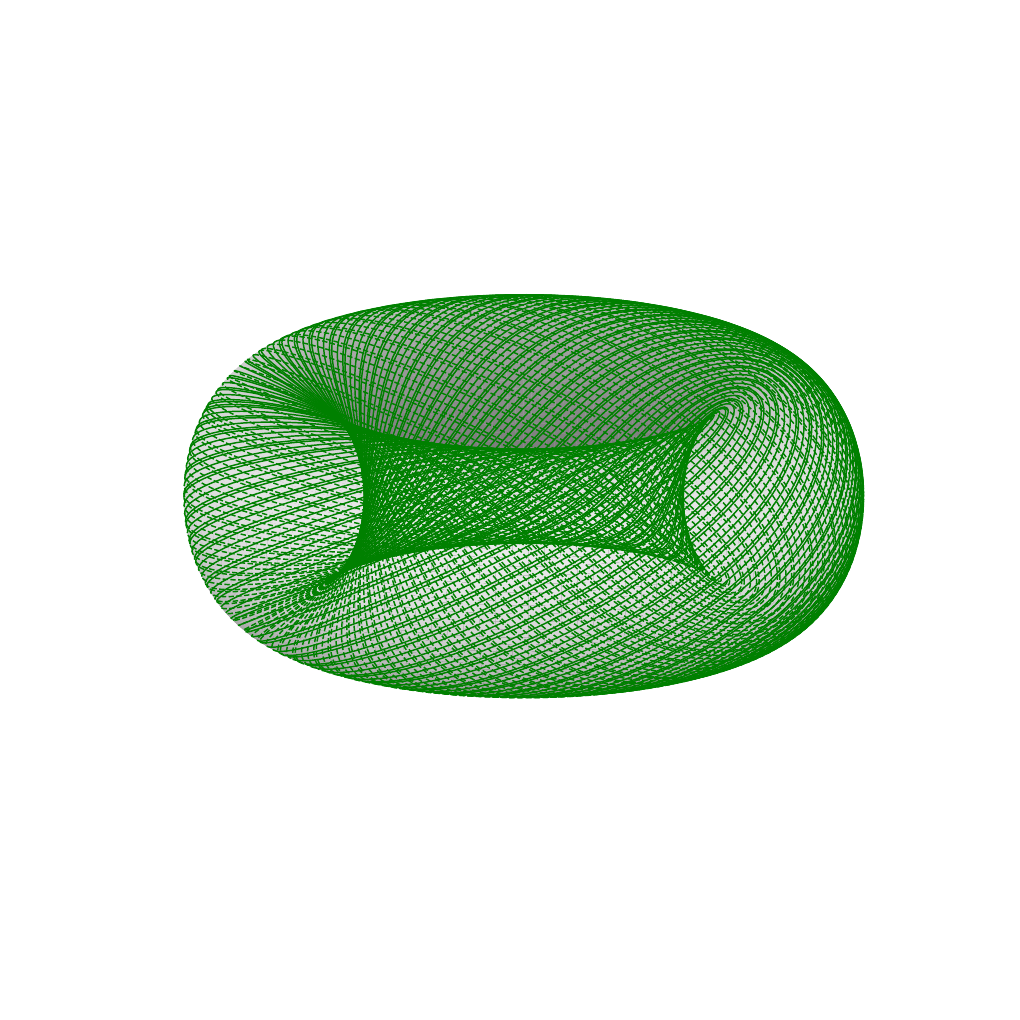}
  \vspace{-6pt}
  \caption{The green curves represent $\phi([-1,1])$, $\phi([-10,10])$, and $\phi([-50,50])$ for $\bm{a}=(1,\sqrt{2})$ on $\mathbb T^2$ canonically embedded into $\mathbb R^3$}
  \label{fig:torus-density}
\end{figure}

\smallskip
\noindent\textbf{Step 3}: Denseness on $[0,1]^N$. Let $P>0$ be a period of $g$ and set $h(u):=g(Pu)$. Then $h$ is continuous and $1$-periodic, so $h(u)=h(\{u\})$. For $k=1,\dots,N$, set $a_k:=1/\bigl(P(2^{1/N}+1+k/N)\bigr)$. By Step 1, $(a_1,\dots,a_N)$ is rationally independent, and hence Step 2 gives that \[ \mathcal D:=\{[\{a_1w\},\dots,\{a_Nw\}]^T:w\in\mathbb R\} \] is dense in $[0,1]^N$. Moreover, \[ g\Bigl(\frac{w}{2^{1/N}+1+k/N}\Bigr) =h\Bigl(\frac{w}{P(2^{1/N}+1+k/N)}\Bigr) =h(a_kw)=h(\{a_kw\}), \] so the point set in the lemma is $\mathcal H(\mathcal D)$, where $\mathcal H(u_1,\dots,u_N):=[h(u_1),\dots,h(u_N)]^T$. Since $\mathcal H$ is continuous and $\mathcal D$ is dense in the compact cube $[0,1]^N$, \[ \overline{\mathcal H(\mathcal D)} =\mathcal H([0,1]^N). \] Finally, $h([0,1])=g([0,P])=[0,1]$ by continuity and the assumptions $\min g=0$, $\max g=1$. Therefore $\mathcal H([0,1]^N)=h([0,1])^N=[0,1]^N$, and the desired point set is dense in $[0,1]^N$.
\end{proof}

The preceding denseness statement is not yet a network construction.  The next
lemma converts it into the finite-table fitting primitive used in
Proposition~\ref{prop:local}: once a grid cell is represented by an integer index,
a width-one, depth-two network can assign an essentially arbitrary value to each
index.

\begin{lemma}[Finite point fitting by a fixed-size network]\label{lemma:point-fit}
For any $\varepsilon>0$, $\xi_i\in[0,1]$, $i=1,2,\dots,N$, 
there exists an $\EUAF$-activated network $\varphi$ of width $1$ and depth $2$  such that
\[
    \max_{1\le i \le N}\bigl|\varphi(i)-\xi_i\bigr| < \varepsilon.
\]
\end{lemma}
\begin{proof}
Let
\[
  g_{\EUAF}(t):=\EUAF(t),\qquad t\ge0.
\]
On the nonnegative half-line, $g_{\EUAF}$ is continuous, $2$-periodic, and has range $[0,1]$. By Lemma~\ref{lemma:dense} applied to $g_{\EUAF}$, there exists a $\tilde{w}\in\mathbb R$ such that
\[
  \max_{1\le i \le N}\left|g_{\EUAF}\!\left(\frac{\tilde{w}}{2^{1/N}+1+i/N}\right)-\xi_i\right| < \varepsilon.
\]
We claim that
\[
\varphi(x)=\EUAF\Big(\tilde{w}\cdot \EUAF\Big(-2^{1/N}-\frac{1}{N}x\Big)+\tilde{w}+2\lfloor|\tilde{w}|\rfloor+2\Big)
\]
is the network we want. We first observe that $-2^{1/N}-x/N<0$ for $x\ge0$. This leads to the following computation
\[
\tilde{w}\cdot\EUAF\Big(-2^{1/N}-\frac{1}{N}x\Big)+\tilde{w}=\tilde{w}\left(\frac{-2^{1/N}-\frac{1}{N}x}{|-2^{1/N}-\frac{1}{N}x|+1}+1\right)=\frac{\tilde{w}}{2^{1/N}+1+\frac{x}{N}}.
\]
Notice that $\frac{\tilde{w}}{2^{1/N}+1+x/N}+2(\lfloor|\tilde{w}|\rfloor+1)>0$ for $x\ge 0$. It follows from the $2$-periodicity of the nonnegative branch $g_{\EUAF}$ that
\[
\varphi(x)=\EUAF\left(\frac{\tilde{w}}{2^{1/N}+1+x/N}+2(\lfloor|\tilde{w}|\rfloor+1)\right)=g_{\EUAF}\left(\frac{\tilde{w}}{2^{1/N}+1+x/N}\right), \qquad x\in [0, \infty).
\]
Therefore,
\[
\max_{1\le i \le N}\bigl|\varphi(i)-\xi_i\bigr|
=\max_{1\le i \le N}\left|g_{\EUAF}\!\left(\frac{\tilde{w}}{2^{1/N}+1+i/N}\right)-\xi_i\right|< \varepsilon.
\]
This completes the proof.
\end{proof}

After the coefficient values have been encoded, the network must know which grid
cell the input belongs to.  The following staircase lemma provides exactly this
cell-indexing device on the active half of each grid interval.  It is the bridge
between the geometric grid and the one-dimensional table in Lemma~\ref{lemma:point-fit}.

\begin{lemma}[Staircase map]\label{lemma:steps} 
There exists an $\EUAF$-activated network $\varphi_1$ of width $2$ and depth $1$  such that
\[
    \varphi_1(x)=i,  \qquad x\in 
    \left[\,\frac{i}{N},\ \frac{2i+1}{2N}\,\right], \qquad i=0,\ldots,N-1.
\]
\end{lemma}
\begin{proof}
$\varphi_1(x)=N\cdot\EUAF(x)-\frac{1}{2}\EUAF(2Nx)$ is a $\EUAF$--NN of width 2 and depth 1 which satisfies this property.
\end{proof}

The local Taylor models contain products of coefficient tables and monomials.  The
next lemma shows that multiplication itself can be implemented exactly on bounded
sets with fixed width and depth.  This algebraic primitive is used repeatedly: for
monomials, for local polynomial terms, and later for gate products.

\begin{lemma}[Exact multiplication on bounded intervals]\label{lemma:multi}
For any $M>0$, there exists a function $\varphi_4$ generated by an $\EUAF$-activated network of width $6$ and depth $2$  such that
\[
    \varphi_4(x,y)=xy, 
    \quad \text{for any } x,y\in[-M,M].
\]
\end{lemma}
\begin{proof}
Observe that for any $x\in[-1,1]$
\begin{align*}
\tildevarphi_4(x)
&:=-4\,\EUAF\Bigl(1-\bigl(\EUAF(x-1)+\EUAF(-x-1)+2\bigr)\Bigr)
\notag\\
&=-4\,\EUAF\left(
1-\left(\frac{x-1}{2-x}+\frac{-x-1}{x+2}+2\right)
\right)
\notag\\
&=-4\,\EUAF\left(
1-\frac{4}{4-x^2}
\right)
=-4\,\EUAF\left(
-\frac{x^2}{4-x^2}
\right)
\notag\\
&=-4\cdot
\frac{-\frac{x^2}{4-x^2}}{1+\frac{x^2}{4-x^2}}
\notag\\
&=x^2.
\end{align*}
Using polarization identity, for any $x,y\in[-M,M]$, we have
\begin{align*}
\varphi_4 (x,y):=2M^2\left(\tilde{\varphi}_4\left(\frac{x+y}{2M}\right)-\tilde{\varphi}_4\left(\frac{x}{2M}\right)-\tilde{\varphi}_4\left(\frac{y}{2M}\right)\right)=xy.
\end{align*}
Since $\tilde{\varphi_4}$ is a $\EUAF$-NN with width 2 depth 2, $\varphi_4$ is a $\EUAF$-NN with width $6$ and depth $2$.

\end{proof}

With multiplication available, we can build monomials by repeated products.  The
following lemma gives the monomial subnet that will be paired with the coefficient
subnet in Proposition~\ref{prop:local}.

\begin{lemma}[Exact monomial realization]\label{lemma:poly} 
If $\bm{\alpha}\in\mathbb{N}_0^d$, $|\bm{\alpha}|=s>0$, then there exists an $\EUAF$--activated network 
of width $\le s+3d$ and depth $\le 2s+2d-4$  such that
\[
    \varphi_{\bm{\alpha}}(\bm{x})=\bm{x}^{\bm{\alpha}}, \quad  \bm{x}\in[0,1]^d.
\]
The constant monomial $\bm{x}^{\bm{0}}=1$ is represented by an affine output layer; zero components of $\bm{\alpha}$ are treated as constant factors.
\end{lemma}
\begin{proof}
We first consider the case $d=1$. From the proof of Lemma~\ref{lemma:multi}, there is an $\EUAF$--activated network of width $2$ and depth $2$ that represents $x^2$ on $[0,1]$. Let $W_s$ and $L_s$ denote the width and depth required to generate $x^s$. For $s>2$, since $x^s=\varphi_4(x,x^{s-1})$ and $x\mapsto(x,x^{s-1})$ is an $\EUAF$--activated network of width $W_{s-1}+1$ and depth $L_{s-1}$, we have $W_s=\max\left\{6, W_{s-1}+1\right\}=s+3$ and $L_s=L_{s-1}+2=2s-2$.
For arbitrary dimension, we first compute the architecture of $x_1x_2\cdots x_d$. If $d>2$, then $x_1x_2\cdots x_d=\varphi_4(x_d,x_1\cdots x_{d-1})$. Using the same recursion pattern as above, $x_1x_2\cdots x_d$ can be represented by an $\EUAF$--activated network with width $4+d$ and depth $2d-2$.
By applying the computation from $d=1$, we conclude that the map
\[
(x_1,\ldots,x_d)\mapsto (x_1^{\alpha_1},\ldots,x_d^{\alpha_d})
\]
has width $(\alpha_1+3)+\cdots+(\alpha_d+3)=s+3d$ and depth $2\max_{1\le i \le d}\alpha_i-2\le2s-2$. Finally, composing with the $d$-tuple multiplication map, $\bm{x}^{\bm{\alpha}}$ can be generated by an $\EUAF$--activated network of width $\le s+3d$ and depth $\le 2s+2d-4$.
\end{proof}

The previous lemmas build the local polynomial pieces.  To pass from local pieces
to a global Sobolev approximant, we need smooth gates that form a partition of
unity after suitable shifts.  The next definition introduces those gates.

\begin{definition}[Multidimensional $C^n$-gates]
Let $K, d\in\mathbb N$ and $n\in\mathbb N\cup\{\infty\}$. For 
$\bm{i}=[i_1,\dots,i_d]^{T}\in\{1,2,3,4\}^d$ and 
$\bm{x}=[x_1,\dots,x_d]^{T}\in\mathbb R^d$ define
\[
  \Psi_{\bm{i},n,K}(\bm{x})
  :=
  \prod_{j=1}^d
  g_n\Bigl(2K x_j + \frac{i_j}{2}\Bigr),
\]
where $g_n$ is the $C^n$--periodic gate defined by \eqref{eq:def-qs}.
\end{definition}

The first five lemmas were proved directly for \EUAF\@.  Higher-order Sobolev
approximation requires the smoother activation $\DUAF_n$.  The next lemma records
that the same finite-data, staircase, multiplication, and monomial primitives also
exist for $\DUAF_n$, and that the multidimensional gates are themselves realizable
by fixed-size $\DUAF_n$ networks.

{\begin{lemma}[Basic properties of $\DUAF_n$ ]\label{lem:sigma_n-properties}
Let $n\in \mathbb N \cup\{\infty\}$, we have the following properties:
\begin{enumerate} 

  \item[\textnormal{(i)}] 
  Lemma~\ref{lemma:point-fit}, Lemma~\ref{lemma:steps}, Lemma~\ref{lemma:multi} and Lemma~\ref{lemma:poly} hold for $\DUAF_n$-activated networks.  Moreover, for any given $ m, L\in\mathbb N$, $M>0$, $\DUAF_n$-activated networks can represent the identity map in $[-M,M]^m$ with width $m$ and depth $L$.

  \item[\textnormal{(ii)}]
  $\Psi_{\bm{i},n,K}$ can be realized as  a $\DUAF_n$--activated network with width $d+4$ and depth $2d-1$ on $[0, \infty)^d$.

\end{enumerate}
\end{lemma}}
\begin{proof}

(i) Since $g_n$ is continuous and $2$-periodic, with $\min_{x\in\mathbb R}g_n(x)=0$ and $\max_{x\in\mathbb R}g_n(x)=1$, Lemma~\ref{lemma:dense} holds for $g_n$. In addition, by the defining left branch of $\DUAF_n$, we have $\DUAF_n(x)=h_n(x)$ for $x\le0$ and $h_n(x)=-2-x^{-1}$ on $[-4,-13/11]$. Hence $\DUAF_n(x)+2=-\frac{1}{x}$ when $-4\le x\le -13/11$. Observe that 
\[
x\in[0, N]\Longrightarrow -2^{1/N}-1-\frac{1}{N}x\in[-2^{1/N}-2, -2^{1/N}-1]\subset [-4, -2]\subset\Big[-4, -\frac{13}{11}\Big].
\]
Therefore, for $x\in[0, N]$, the following $\DUAF_n$-activated network with width 1 and depth 2
\begin{align*}
\varphi(x)
&:=
\DUAF_n\Big(\tilde{w}_n\cdot\DUAF_n\Big(-2^{1/N}-1-\frac{1}{N}x\Big)+2\tilde{w}_n+2\lfloor|\tilde{w}_n|\rfloor+2\Big)
\notag\\
&=\DUAF_n\Big(\frac{\tilde{w}_n}{2^{1/N}+1+x/N}+2\lfloor|\tilde{w}_n|\rfloor+2\Big)
\notag\\
&=g_n\Big(\frac{\tilde{w}_n}{2^{1/N}+1+x/N}+2\lfloor|\tilde{w}_n|\rfloor+2\Big)
\notag\\
&=g_n\Big(\frac{\tilde{w}_n}{2^{1/N}+1+x/N}\Big),
\end{align*}
where $\tilde{w}_n\in\mathbb R$ was chosen such that $\max_{1\le i \le N}\bigl|g_n(\tfrac{\tilde{w}_n}{2^{1/N}+1+i/N})-\xi_i\bigr| < \varepsilon$, will approximate $\xi_i$ at integer points $x=1,\ldots,N$. This proves Lemma~\ref{lemma:point-fit}.

Moreover, using the property that $\DUAF_n(x)+2=-\frac{1}{x}$ for $-4\le x\le -13/11$ once again, it's not difficult to check that the following $\DUAF_n$--activated network of width $2$ and depth $2$ satisfies
\[
-12\,\DUAF_n\Bigl(
-2\,\DUAF_n(x-3)
-2\,\DUAF_n(-x-3)
-8
\Bigr)
-15
=x^2,
\quad
x\in[-1, 1].
\]
Hence, the multiplication map follows from applying the polarization identity. This proves Lemma~\ref{lemma:multi}. In general, Lemma~\ref{lemma:poly} will follow from Lemma~\ref{lemma:multi}.

Next, we show that Lemma~\ref{lemma:steps} holds for $\DUAF_n$-activated networks. Let $\varrho_n:\mathbb R\to\mathbb R$ be the function defined by \eqref{eq:def-qs}. We claim that the following identity holds (see Figure~\ref{fig:staircase-s} for an illustration):
\[
S_{n, N}(t):=2N\,\varrho_n\!\left(\frac t2\right)-\varrho_n(Nt)=i,
\qquad
t\in\left[\frac{i}{N},\frac{2i+1}{2N}\right],
\quad i=0,1,\dots,N-1.
\]

\begin{figure}[htbp] 
    \centering
    \includegraphics[width=0.48\textwidth]{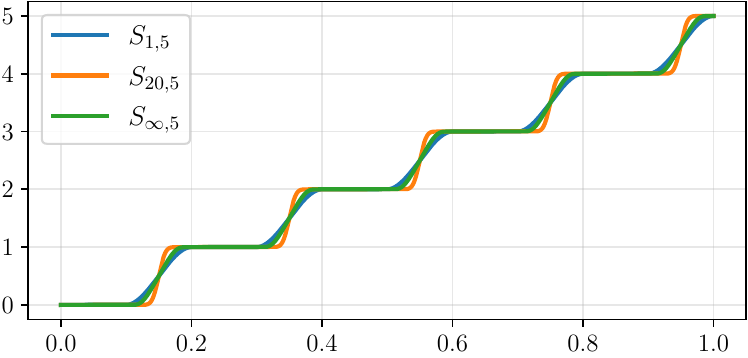}\hfill
    \includegraphics[width=0.48\textwidth]{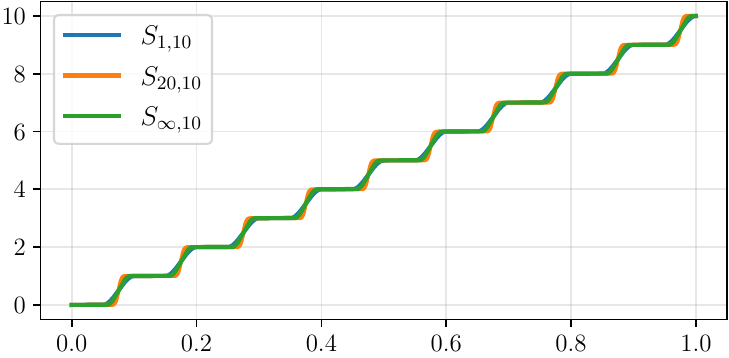}
    \caption{Comparison of $S_{n,N}$ for $N=5$ (left) and $N=10$ (right)}
    \label{fig:staircase-s}
\end{figure}

To verify this, note that \(t/2\in[0,\tfrac12]\), so \(\varrho_n(t/2)=t/2\). Also,
\[
Nt\in\left[i,\frac{2i+1}{2}\right]\implies \{Nt\}=Nt-i\in \big[0,\tfrac12\big].
\]
By the \(1\)-periodicity of \(\varrho_n\) and the fact that \(\varrho_n(x)=x\) on \([0,1/2]\), we have
\[
\varrho_n(Nt)=\varrho_n(\{Nt\})=\{Nt\}=Nt-i.
\]
Hence
\[
S_{n, N}(t)=2N\,\varrho_n\!\left(\frac t2\right)-\varrho_n(Nt)=Nt-(Nt-i)=i.
\]
From the definition of $\DUAF_n$, it is straightforward to check that $-\DUAF_n(-x-5)-\frac{14}{5}=\varrho_n(x)$ whenever $x\ge0$. Therefore, we can construct a $\DUAF_n$--activated network of width $2$ and depth $1$ such that
\[
\DUAF_n(-Nx-5)-2N\cdot \DUAF_n\Big(-\frac{x}{2}-5\Big)-\frac{14(2N-1)}{5}= S_{n, N}(x), \qquad x\in[0, \infty).
\]
This proves Lemma~\ref{lemma:steps}.

The last statement from (i) is trivial due to the fact that $\DUAF_n(x)=x$ when $x\in [-1, -4/11]$.

(ii) Observe that
\[
\Big[g_n(2Kx_1+\frac{i_1}{2}),\ldots,g_n(2Kx_d+\frac{i_d}{2})\Big]^{T}=\DUAF_n \circ \Big(2K\cdot \bm{I}_d+\frac{\bm{i}}{2}\Big)\bm{x},
\qquad \bm{x}\in [0, \infty)^d, 
\]
together with $0\le g_n \le 1$. We can compose the above map with the $d$--fold multiplication map, which is a width $d+4$ and depth $2d-2$ $\DUAF_n$--network in $[0,1]^d$ by Lemma~\ref{lemma:poly} from (i) (the proof of Lemma~\ref{lemma:poly}). The resulting architecture is then a $\DUAF_n$--activated network of width $\max\{d+4, d\}=d+4$ and depth $(2d-2)+1=2d-1$ that exactly represents $ \Psi_{\bm{i},n,K}(\bm{x})$ for $\bm{x}\in [0, \infty)^d$.

\end{proof}

\medskip
\noindent
In the sequel, whenever we refer to Lemma~\ref{lemma:point-fit}, 
Lemma~\ref{lemma:steps}, Lemma~\ref{lemma:multi}, or Lemma~\ref{lemma:poly}
in the context of $\DUAF_n$-activated networks, we implicitly invoke 
Lemma~\ref{lem:sigma_n-properties} and use the corresponding realizations 
without further mention.

The last ingredient is analytic rather than algebraic.  In the gluing argument,
derivatives fall both on the local approximants and on the gates.  The following
lemma records boundedness, smoothness, partition-of-unity, and localization
properties of $\Psi_{\bm{i},n,K}$, which are exactly the facts needed when applying
Leibniz's rule in Theorems~\ref{thm:shift-W1-d} and~\ref{thm:shift-Ws-d}.

\begin{lemma}[Basic properties of $g_n$ and $\Psi_{\bm{i},n,K}$ ]\label{lem:gn-properties}
Let $n\in \mathbb N \cup\{\infty\}$ and let $g_n$ be defined as in \eqref{eq:def-qs}. Then:
\begin{enumerate}

  \item[\textnormal{(i)}] \textbf{$L^{\infty}$-bounds.}
  For $t\in\mathbb R$ and $\bm{x}\in\mathbb R^d$, we have
  \[
    0 \,\le\, g_n(t) \,\le\, 1,\qquad 0\le\Psi_{\bm{i},n,K}(\bm{x})\le 1.
  \]

  \item[\textnormal{(ii)}] \textbf{$C^n$--regularity.}
  The function $g_n$ is of class $C^n$ on $\mathbb R$. Therefore, by construction, $\Psi_{\bm{i},n,K}\in C^n(\mathbb R^d)$

  \item[\textnormal{(iii)}] \textbf{derivative bounds.}
 Let $\bm{\gamma}\in\mathbb N^d$ be a multi–index with $|\bm{\gamma}|\le n$. Then
  \[
    \|D^{\bm{\gamma}} \Psi_{\bm{i},n,K}\|_{L^\infty(\mathbb R^d)}
    \;\le\;
    C_{n,d}\,K^{|\bm{\gamma}|}, \qquad
        \|D^{\bm{\gamma}} \Psi_{\bm{i},\infty,K}\|_{L^\infty(\mathbb R^d)}
    \;\le\;
    C_{\infty,d}\,K^{|\bm{\gamma}|},
  \]
  where $C_{n,d}>0$ depends only on $n$ and $d$ (not on $K$ or $\bm{i}$).

  \item[\textnormal{(iv)}] \textbf{partition of unity.}
  for all $\bm{x}\in\mathbb R^d$, we have
  \[
    \sum_{\bm{i}\in\{1,2,3,4\}^d} \Psi_{\bm{i},n,K}(\bm{x})=1.
  \]
  \item[\textnormal{(v)}] \textbf{localization}
  Let $\bm{\tau}_{\bm{i}}:= \frac{1}{4K}(i_1,\ldots,i_d)\in\mathbb{R}^d$,  $|\alpha|\le n$, and $K>10$, then for $\bm{x}\in[0,\frac{9}{10}]^d$ we always have
  \[
  D^\alpha \Psi_{\bm{i},n,K}(\bm{x})\neq0
  \ \Longrightarrow\
  \bm{x} + \bm{\tau}_{\bm{i}}\in\Omega_{\bm{m}_*}.
\]
where $\Omega_{\bm{m}_*}$ is the region defined by \eqref{eq:omega} below.
\end{enumerate}
\end{lemma}

\begin{proof}[Sketch of proof] 
Properties (i) and (ii) are straightforward from the definitions. 
For (iii), one applies the chain rule and the product rule to the product representation of $\Psi_{\bm{i},n,K}$. 
Property (iv) follows from the one-dimensional identity (see Figure~\ref{fig:gn-partition} for an illustration)
\[
g_n\Bigl(t+\frac{1}{2}\Bigr)+g_n\Bigl(t+1\Bigr)+g_n\Bigl(t+\frac{3}{2}\Bigr)+g_n\Bigl(t+2\Bigr)=1,
\qquad t\in\mathbb R,
\]
and then taking products over the coordinates. 
Finally, (v) follows from the support properties of $g_n$, since if $\bm{x} + \bm{\tau}_{\bm{i}}\notin\Omega_{\bm{m}_*}$ then $x_j+\tau_{\bm{i},j}\in\Omega_2-\Omega_1$ for some $1\le j\le d$. This is due to the fact that $\bm{x} + \bm{\tau}_{\bm{i}}\in[0, 1)^d$ whenever $K>10$ and $\Omega_1\cup \Omega_2=[0, 1]$. But then the $j$-th component of $ \Psi_{\bm{i},n,K}$ satisfies
\[
g_n\Bigl(2Kx_j+\tfrac{i_j}{2}\Bigr)=g_n\bigl(2K(x_j+\tau_{\bm{i},j})\bigr)= 0,
\]
and $g_n\bigl(2Kx+\tfrac{i_j}{2}\bigr)\equiv 0$ on a small neighborhood of $x_j$. This forces $\Psi_{\bm{i},n,K}\equiv 0$ on a small neighborhood of $\bm{x}$, which leads to all vanishing derivatives. This proves the contrapositive of (v).
\begin{figure}[htbp]
    \centering
    \includegraphics[width=\textwidth]{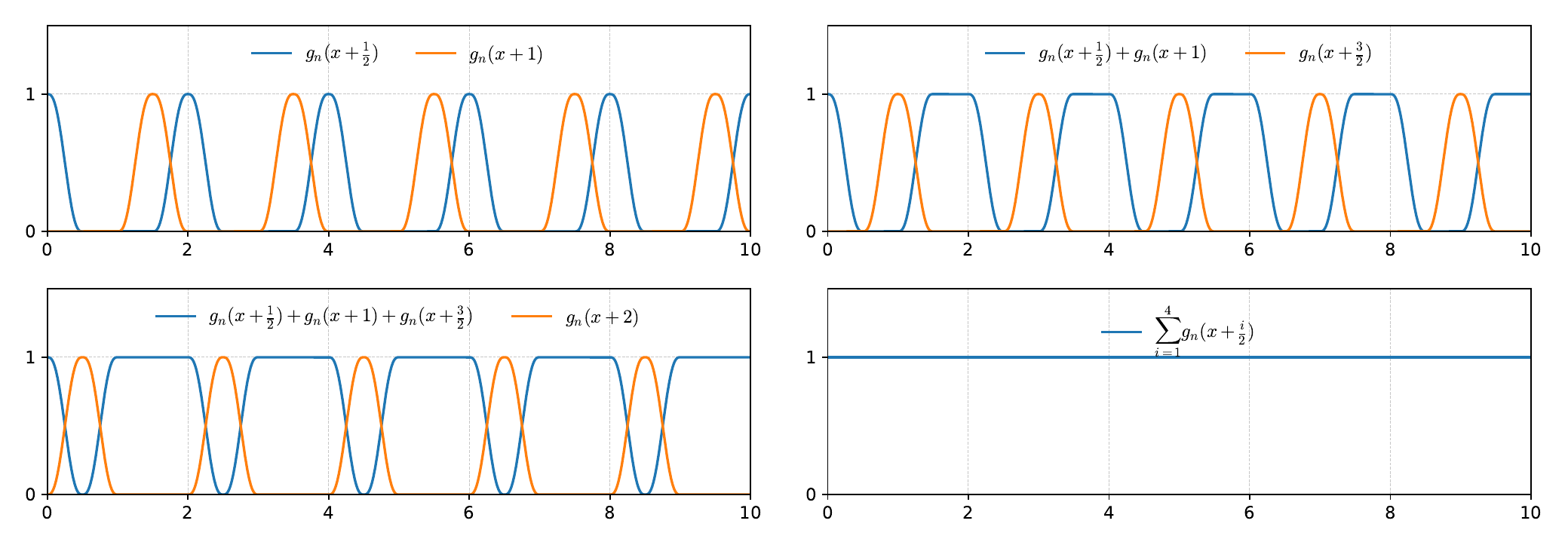}
    \caption{Illustration of the partition-of-unity structure generated by the gates $g_n$ for $n=2$}
    \label{fig:gn-partition}
\end{figure}
\end{proof}

\section{Proof of Theorem~\ref{thm:global-w2} and Theorem~\ref{thm:global-ws}}\label{sec:proof-main}

The proof is organized around a local-to-global construction.  On each half-grid
region $\Omega_{\bm{m}}$, we approximate $f$ by a Taylor-type polynomial whose
coefficients are constant on the grid cells.  The Bramble--Hilbert lemma gives the
Sobolev error $K^{r-s}$ for this local polynomialization.  The nonstandard
activation is then used to encode the coefficient tables: a step network identifies
the relevant cell, the denseness lemma fits arbitrary cell values, and multiplication
networks attach the monomials $\bm{x}^{\bm{\alpha}}$ (see Section~\ref{sec:loc-omega}).

The local approximants only work on disjoint half-grid regions, so we cover
$(0,9/10)^d$ by $4^d$ shifted copies and glue them with the gates
$\Psi_{\bm{i},n,K}$.  The key point is that the derivatives of the gates grow like
powers of $K$, while the local error is chosen of order $K^{-s}$; after Leibniz's
rule the total $W^{s-1,\infty}$ error is still $O(K^{-1})$.  Taking $K$ large gives
arbitrary accuracy, but $K$ enters only through the network parameters, not through
the width or depth (see Section~\ref{sec:sobolev-subcube}).

The final passage from $(0,9/10)^d$ to $(a,b)^d$ is an affine
change of variables and a normalization of the Sobolev norm (see Section~\ref{sec:global-sobolev-cube}).

\begin{definition}\label{def:Omega-m}
Let $K,d\in\mathbb{N}$. 
For $k=0,1,\ldots,K-1$ define
\[
  I_{1,k}
  := \Big[\frac{k}{K},\,\frac{2k+1}{2K}\Big],\quad 
  I_{2,k}
  := \Big[\frac{2k+1}{2K},\,\frac{2k+2}{2K}\Big], \qquad
    \Omega_1 := \bigcup_{k=0}^{K-1} I_{1,k},
  \quad
  \Omega_2 := \bigcup_{k=0}^{K-1} I_{2,k}.
\]
For any $\bm{m}=(m_1,\ldots,m_d)\in\{1,2\}^d$, define
\begin{equation}\label{eq:omega}
  \Omega_{\bm{m}} := \prod_{j=1}^d \Omega_{m_j},\ \ \ \ \ \ \ \ \ \Omega_{\bm{m}_*} := \Omega_1 \times\cdots\times \Omega_1=\Omega_1^d.
\end{equation}
For any $\bm{i}=(i_1,\ldots,i_d)\in\{0,\ldots,K-1\}^d$, define the sub-rectangle
\[
  Q_{\bm{i},\bm{m}}
  := I_{m_1,i_1}\times\cdots\times I_{m_d,i_d}\subset \Omega_{\bm{m}}.
\]
\end{definition}

Since the sets $\Omega_{\bm{m}}$ are finite unions of half-grid rectangles, Sobolev norms on $\Omega_{\bm{m}}$ are understood cellwise throughout this section:
\[
\|u\|_{W^{r,\infty}(\Omega_{\bm{m}})}
:=\max_{\bm{i}\in\{0,\ldots,K-1\}^d}
\|u\|_{W^{r,\infty}(\operatorname{int}Q_{\bm{i},\bm{m}})},
\]
with the analogous convention for Sobolev seminorms.

We shall repeatedly use the following elementary consequence of this convention.  If
$a$ is constant on each cell $\operatorname{int}Q_{\bm{i},\bm{m}}$ and $P$ is a polynomial,
then, for every integer $s\ge1$ and $0\le r\le s-1$,
\begin{equation}\label{eq:cellwise-product-bound}
\|aP\|_{W^{r,\infty}(\Omega_{\bm{m}})}
\le C(r,s,d)\|a\|_{L^\infty(\Omega_{\bm{m}})}
\max_{|\beta|\le r}\|D^\beta P\|_{L^\infty((0,1)^d)}.
\end{equation}
Indeed, on each open cell the distributional derivatives reduce to ordinary
polynomial derivatives, and the cellwise norm then takes the maximum over cells.

\subsection{\texorpdfstring{Local Sobolev approximation on $\Omega_{\bm{m}}$}{Local approximation on Omega m}}\label{sec:loc-omega}

\begin{theorem}\label{thm:taylor-g-alpha-K} 

Let $s\in\mathbb{N}$ and $f\in W^{s,\infty}((0,1)^d)$ with 
\[
M := \|f\|_{W^{s,\infty}((0,1)^d)}
    = \max_{|\gamma|\le s}\|\partial^\gamma f\|_{L^\infty((0,1)^d)}.
\]
For any $K\in\mathbb{N}$, there exists a function on the half-grid region
$\Omega_{\bm{m}_*}$ of the form
\[
f_{K,\bm{m}_*} (\bm{x})=\sum_{|\bm{\alpha}|\le s-1} g_{f,\bm{\alpha},\bm{m}_*}(\bm{x})\,\bm{x}^{\bm{\alpha}},
\]
where each $g_{f,\bm{\alpha},\bm{m}_*}$ is piecewise constant on the grid cubes
$Q_{\bm{i},\bm{m}_*}$, such that:
\[
    |g_{f,\alpha,\bm{m}_*}(\bm{x})|
    \;\le\; C_0\,M,
    \quad \forall\,|\alpha|\le s-1, 
    \qquad
     \|f-f_{K,\bm{m}_*}\|_{W^{r,\infty}(\Omega_{\bm{m}_*})}
    \;\le\;
    C_{0}\,M\,K^{r-s}, \quad \forall\,r\in\mathbb N_{0}, r\le s,
\]
where the constant $C_0$ depends only on $s,d$ and is independent of $K$.
\end{theorem} 
\begin{proof}
Let $E:W^{s,\infty}((0,1)^d)\to W^{s,\infty}(\mathbb R^d)$ be an extension operator and set
$\tilde f:=Ef$. Denote by $C_E$ the operator norm of $E$. In particular,
\[
\|\tilde{f}\|_{W^{s,\infty}(\mathbb R^d)}\le C_E\,\|f\|_{W^{s,\infty}((0,1)^d)}=C_E M.
\]

For $\bm{i}=(i_1,\ldots,i_d)\in\{0,1,\ldots,K-1\}^d$, set the center $\bm{x}_{\bm{i}}:=\Big(\frac{4i_1+1}{4K},\ldots,\frac{4i_d+1}{4K}\Big)$. Then we have
\[
\overline{B}_{\frac{1}{4K},\|\cdot\|_{\ell^\infty}}(\bm{x}_{\bm{i}})
=\prod_{j=1}^d\Big[\frac{i_j}{K},\frac{2i_j+1}{2K}\Big]= I_{1,i_1}\times\cdots\times I_{1,i_d}=Q_{\bm{i},\bm{m}_*}.
\]

Define the averaging ball $B_{\bm{i},K}$ and the averaged Taylor polynomial $p_{f,\bm{i}}$ of order $s$ (hence of degree at most $s-1$) on $B_{\bm{i},K}$ by
\[
B_{\bm{i},K}:=\overline{B}_{\frac{1}{8K},|\cdot|}(\bm{x}_{\bm{i}}),\qquad
p_{f,\bm{i}}
:=\int_{B_{\bm{i},K}} T^{\,s}_{\bm{y}}\tilde f(\bm{x})\,b_{\frac{1}{8K}}(\bm{y})\,d\bm{y}.
\]

\begin{figure}[h]
    \centering
    \includegraphics[width=0.4\textwidth]{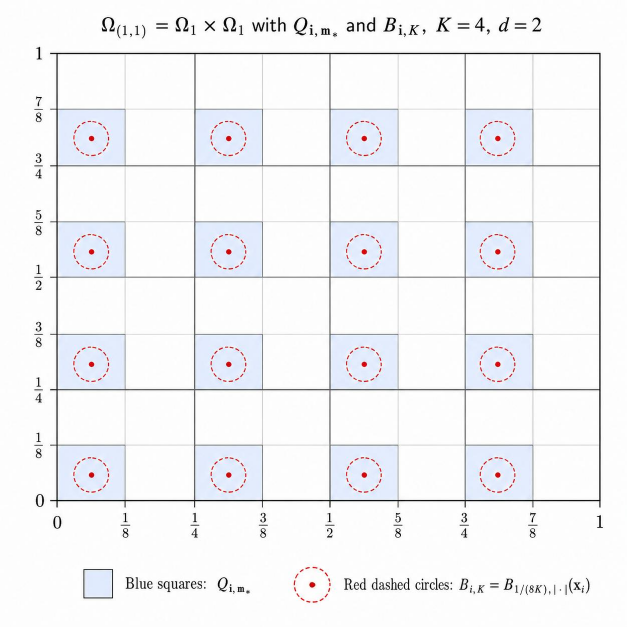}
    \caption{The region $\Omega_{(1,1)}$ with local rectangles and averaging balls}
    \label{fig:omega-11}
\end{figure}

Then $p_{f,\bm{i}}$ is a polynomial of degree at most $s-1$ and can be written as
\[
p_{f,\bm{i}}(\bm{x})=\sum_{|\alpha|\le s-1} c_{f,\bm{i},\alpha}\,\bm{x}^\alpha, \qquad \text{where}  \ \
|c_{f,\bm{i},\alpha}|\le C_2(s,d)\,C_E M.
\]

The above choice of $B_{\bm{i},K}$ is made so that the averaged Taylor polynomial
gives a uniform local approximation on the cube $Q_{\bm{i},\bm{m}_*}$ at scale $K^{-1}$.
Since
\[
\operatorname{diam}(Q_{\bm{i},\bm{m}_*})=\frac{\sqrt d}{2K},
\qquad
r_{\max}^\ast(Q_{\bm{i},\bm{m}_*})=\frac{1}{4K},
\qquad
\gamma(Q_{\bm{i},\bm{m}_*})
:=\frac{\operatorname{diam}(Q_{\bm{i},\bm{m}_*})}{r_{\max}^\ast(Q_{\bm{i},\bm{m}_*})}
=2\sqrt d,
\]
and $\frac{1}{8K}=\frac{1}{2}\,r_{\max}^\ast(Q_{\bm{i},\bm{m}_*})$, applying the Bramble--Hilbert lemma (Lemma~\ref{lem:bramble-hilbert}), we obtain that for every
$0\le r\le s$,
\[
|\tilde f-p_{f,\bm{i}}|_{W^{r,\infty}(Q_{\bm{i},\bm{m}_*})} \le C_{\mathrm{BH}}(s,d,2\sqrt d)\,C_E M\,K^{r-s}.
\]
Consequently,
\[
\|f-p_{f,\bm{i}}\|_{W^{r,\infty}(Q_{\bm{i},\bm{m}_*})}
\le C_{\mathrm{BH}}(s,d,2\sqrt d)\,C_E M\,K^{r-s},\qquad 0\le r\le s.
\]

For $i=0,1,\ldots,K-1$, define $h_i:[0,1]\to[0,1]$ by
\[
h_i(x):=
\begin{cases}
1, & x\in I_{1,i}:=\big[\frac{i}{K},\frac{2i+1}{2K}\big],\\[0.4em]
2Kx-(2i-1), & x\in I_{2,i-1}:=\big[\frac{2i-1}{2K},\frac{i}{K}\big]\ \ (i\ge1),\\[0.4em]
(2i+2)-2Kx, & x\in I_{2,i}:=\big[\frac{2i+1}{2K},\frac{i+1}{K}\big]\ \ (i\le K-2),\\[0.4em]
1, & x\in I_{2,K-1}:=\big[\frac{2K-1}{2K},1\big]\ \ (i=K-1),\\[0.4em]
0, & \text{otherwise.}
\end{cases}
\]
A direct computation shows that $\{h_i\}_{i=0}^{K-1}$ forms a partition of unity on $[0,1]$, together with $h_i(x)=1, \text{ for }x\in I_{1,i}$ and $h_i(x)=0, \text{ for }x\in I_{1,j}$ when $j\neq i$.

For $\bm{i}=(i_1,\ldots,i_d)\in\{0,\ldots,K-1\}^d$ define
\[
h_{\bm{i}}(\bm{x}):=\prod_{j=1}^d h_{i_j}(x_j).
\label{partition1-new}
\]
Then $\big\{h_{\bm{i}}: \bm{i}\in\{0,\ldots,K-1\}^d\big\}$ forms a partition of unity on $[0,1]^d$, together with $h_{\bm{i}}=1$ on $Q_{\bm{i},\bm{m}_*}$ and $h_{\bm{i}}=0$ on $Q_{\bm{j},\bm{m}_*}$ for all $\bm{j} \neq \bm{i}$.
In particular,
\[
\|h_{\bm{i}}(f-p_{f,\bm{i}})\|_{W^{r,\infty}(Q_{\bm{i},\bm{m}_*})}
= \|f-p_{f,\bm{i}}\|_{W^{r,\infty}(Q_{\bm{i},\bm{m}_*})}
\le C_{\mathrm{BH}}(s,d,2\sqrt d)\,C_E M\,K^{r-s},
\qquad 0\le r\le s,
\]
which is due to $h_{\bm{i}}=1$ on $Q_{\bm{i},\bm{m}_*}$.

Finally, define
\[
f_{K,\bm{m}_*}(\bm{x}):=\sum_{\bm{i}\in\{0,\ldots,K-1\}^d} h_{\bm{i}}(\bm{x})\,p_{f,\bm{i}}(\bm{x}).
\]
Using the mutual disjointness of $Q_{\bm{i},\bm{m}_*}$, $\bigcup_{\bm{i}}Q_{\bm{i},\bm{m}_*}=\Omega_{\bm{m}_*}$, and
${\rm supp}\,h_{\bm{i}}\cap\Omega_{\bm{m}_*}=Q_{\bm{i},\bm{m}_*}$, we have
\begin{align*}
\|f-f_{K,\bm{m}_*}\|_{W^{r,\infty}(\Omega_{\bm{m}_*})}
&\le \max_{\bm{i}\in\{0,\ldots,K-1\}^d}
\|h_{\bm{i}}( f-p_{f,\bm{i}})\|_{W^{r,\infty}(Q_{\bm{i},\bm{m}_*})}\notag\\
&\le C_{\mathrm{BH}}(s,d,2\sqrt d)\,C_E M\,K^{r-s},\qquad 0\le r\le s.
\end{align*}

Last of all,
\begin{align*}
f_{K,\bm{m}_*}(\bm{x})
&=\sum_{\bm{i}} h_{\bm{i}}(\bm{x})\,p_{f,\bm{i}}(\bm{x})
=\sum_{\bm{i}}\sum_{|\alpha|\le s-1} h_{\bm{i}}(\bm{x})\,c_{f,\bm{i},\alpha}\,\bm{x}^\alpha \notag\\
&=\sum_{|\alpha|\le s-1}\Big(\sum_{\bm{i}} h_{\bm{i}}(\bm{x})\,c_{f,\bm{i},\alpha}\Big)\bm{x}^\alpha
=:\sum_{|\alpha|\le s-1} g_{f,\alpha,\bm{m}_*}(\bm{x})\,\bm{x}^\alpha,
\end{align*}
with $|g_{f,\alpha,\bm{m}_*}(\bm{x})|\le C_2(s,d)\,C_E M$ for $\bm{x}\in\Omega_{\bm{m}_*}$.
Moreover, $g_{f,\alpha,\bm{m}_*}$ is a step function on $\Omega_{\bm{m}_*}$:
for $\bm{x}\in Q_{\bm{i},\bm{m}_*}$ one has $h_{\bm{i}}(\bm{x})=1$ and
$h_{\bm{j}}(\bm{x})=0$ for $\bm{j}\neq \bm{i}$, hence
\[
g_{f,\alpha,\bm{m}_*}(\bm{x})=c_{f,\bm{i},\alpha},
\qquad \bm{x}\in Q_{\bm{i},\bm{m}_*}.
\]
Choosing $C_0(s,d)=C_E\max\{C_2(s,d),C_{\mathrm{BH}}(s,d,2\sqrt d)\}$, this completes the proof.
\end{proof}

\begin{proposition}\label{prop:local}
Let $\epsilon>0$, $K\in \mathbb{N}$, and $f\in W^{s,\infty}((0,1)^d)$ with $\|f\|_{W^{s,\infty}((0,1)^d)}\le 1$. In the $\EUAF$ case assume $s=2$ and take $r=0,1$; in the $\DUAF_n$ case assume $n\in\mathbb N\cup\{\infty\}$, $1\le s\le n+1$, and take $0\le r\le s-1$. Then there exists a network $\psi_{\bm{m}_*}$ on the half-grid region $\Omega_{\bm{m}_*}$, activated by the corresponding activation, with width $\le 4^{-d}N_{s,d}-d-4$ and depth $\le \max\left\{ 2s+2d-4, 5\right\}$ such that
\begin{align*}
\|f-\psi_{\bm{m}_*}\|_{W^{r,\infty}(\Omega_{\bm{m}_*})}
    \;\le\;
    \epsilon+C_{0}\,K^{r-s},
\end{align*}
for every admissible $r$ above, where the constant $C_0$ depends only on $s,d$.
    
\end{proposition} 

\medskip
\noindent
The realizations $\psi_{\bm{m}_*}$ for $\EUAF$- and 
$\DUAF_n$-activated networks are in general different, 
but satisfy the same approximation guarantees.
\medskip
\noindent
\begin{proof}  
Let $\epsilon_1=\min\left\{1,\frac{\epsilon}{2^{s}\max\{1,(s-1)^{s-1}\}s^{d}C_0\binom{s+d-1}{d}}\right\}$, where $C_0$ is the constant in Theorem~\ref{thm:taylor-g-alpha-K}.
By Theorem~\ref{thm:taylor-g-alpha-K} and the choice of $K$, we have
\begin{align*}
    \|f-f_{K,\bm{m}_*}\|_{W^{r,\infty}(\Omega_{\bm{m}_*})}
    \;\le\;
    C_0K^{r-s},
\end{align*}
for every admissible $r$.
It remains to approximate $f_{K,\bm{m}_*}$ by a $\DUAF_n$-activated (or $\EUAF$-activated) network.

To approximate $g_{f,\bm{\alpha},\bm{m}_*}(\bm{x})$ we recall that by Lemma~\ref{lemma:steps} there exists a $\DUAF_n$-activated (or $\EUAF$-activated) network of width $2$ and depth $1$ such that
\[
    \varphi_1(x)=k,\ \ x\in \left[\,\frac{k}{K},\ \frac{2k+1}{2K}\,\right], \ \ k=0, 1,\ldots,K-1.
\]
We define
\[
  \bm{\varphi}_2(\bm{x})=\Big[\frac{\varphi_1(x_1)}{K},\frac{\varphi_1(x_2)}{K},\ldots,\frac{\varphi_1(x_d)}{K}\Big]^\top.
\]
For each $p=0,1,\ldots,K^d-1$, there is a bijection
\[
  \bm{\eta}(p)=[\eta_1,\eta_2,\ldots,\eta_d]\in\{0,1,\ldots,K-1\}^d,
\]
such that $\sum_{j=1}^d\eta_jK^{j-1}=p$. Define
\[
  \xi_{\bm{\alpha},p}
  :=\frac{g_{f,\bm{\alpha},\bm{m}_*}(\frac{\bm{\eta}(p)}{K})+C_0}{2C_0}\in[0,1],
\]
where $C_0$ is the bound of $g_{f,\bm{\alpha},\bm{m}_*}$ in Theorem~\ref{thm:taylor-g-alpha-K}. By
Lemma~\ref{lemma:point-fit}, there is a neural network $\tilde\phi_{\bm{\alpha}}$ of width
1 and depth 2 such that
\[
  |\tilde\phi_{\bm{\alpha}}(p+1)-\xi_{\bm{\alpha},p}|\le \epsilon_1,
  \qquad p=0,1,\ldots,K^d-1.
\]
Set
\[
  \phi_{\bm{\alpha}}(\bm{x})
  =2C_0\,\tilde\phi_{\bm{\alpha}}\Big(1+\sum_{j=1}^dx_jK^j\Big)-C_0,
\]
then
\[
  \bigl|\phi_{\bm{\alpha}}(\frac{\bm{\eta}(p)}{K})-g_{f,\bm{\alpha},\bm{m}_*}(\frac{\bm{\eta}(p)}{K})\bigr|
  = 2C_0|\tilde\phi_{\bm{\alpha}}(p+1)-\xi_{\bm{\alpha},p}|\le2C_0\epsilon_1.
\]
Since $\varphi_1$ has constant steps on $[\frac{k}{K},\ \frac{2k+1}{2K}\,]$, $  \bm{\varphi}_2(\bm{x})=\Big[\frac{\varphi_1(x_1)}{K},\frac{\varphi_1(x_2)}{K},\ldots,\frac{\varphi_1(x_d)}{K}\Big]^\top$ is a piecewise constant function on each disjoint grid of $\Omega_{\bm{m}_*}$, and so does $\phi_{\bm{\alpha}}\circ\bm{\varphi}_2$. Hence both $\phi_{\bm{\alpha}}\circ\bm{\varphi}_2$ and $g_{f,\bm{\alpha},\bm{m}_*}$ are piecewise constant functions on $\Omega_{\bm{m}_*}$ with

\begin{align*}
\|\phi_{\bm{\alpha}}\circ\bm{\varphi}_2-g_{f,\bm{\alpha},\bm{m}_*}\|_{W^{r,\infty}(\Omega_{\bm{m}_*})}
&=\max_{p\in\{0,1,\ldots,K^d-1\}}\bigl|\phi_{\bm{\alpha}}(\frac{\bm{\eta}(p)}{K})-g_{f,\bm{\alpha},\bm{m}_*}(\frac{\bm{\eta}(p)}{K})\bigr|\le 2C_0\epsilon_1.
\end{align*}
Choose the multiplication subnet in Lemma~\ref{lemma:multi} with
$M_{\mathrm{mult}}:=\max\{1,3C_0\}$. On every active cell of $\Omega_{\bm{m}_*}$ we have
$|\phi_{\bm{\alpha}}\circ\bm{\varphi}_2|\le C_0+2C_0\epsilon_1\le 3C_0$ and
$|\bm{x}^{\bm{\alpha}}|\le1$, so the exact multiplication identity is valid throughout
the cellwise estimates below.
We define
\begin{align*}
    \psi_{\bm{m}_*}(\bm{x}):=\sum_{|\bm{\alpha}|\le s-1}\varphi_4(\phi_{\bm{\alpha}}\circ\bm{\varphi}_2(\bm{x}), \varphi_{\bm{\alpha}}(\bm{x}))\overset{[0,1]^d}{=}
\sum_{|\bm{\alpha}|\le s-1} \phi_{\bm{\alpha}}\circ\bm{\varphi}_2(\bm{x})\,\bm{x}^{\bm{\alpha}},
\end{align*}
where $\varphi_4$ and $\varphi_{\bm{\alpha}}$ are $\DUAF_n$-activated (or $\EUAF$-activated) networks defined in Lemma~\ref{lemma:multi} and Lemma~\ref{lemma:poly} respectively.  Since the norm on $\Omega_{\bm{m}_*}$ is cellwise, \eqref{eq:cellwise-product-bound} applies to each product of a piecewise constant coefficient and a monomial. Hence,

\begin{align*}
 \|\psi_{\bm{m}_*}-f_{K,\bm{m}_*}\|_{W^{r,\infty}(\Omega_{\bm{m}_*})}
 &\le
 \sum_{|\bm{\alpha}|\le s-1}
   \|\phi_{\bm{\alpha}}(\bm{\varphi}_2(\bm{x}))\,\bm{x}^{\bm{\alpha}}
    -g_{f,\bm{\alpha},\bm{m}_*}(\bm{x})\,\bm{x}^{\bm{\alpha}}\|_{W^{r,\infty}(\Omega_{\bm{m}_*})}
 \notag\\
 &\le
 2^{s-1}(s-1)^{s-1}s^{d}\,\sum_{|\bm{\alpha}|\le s-1}
   \|\phi_{\bm{\alpha}}(\bm{\varphi}_2(\bm{x}))
     -g_{f,\bm{\alpha},\bm{m}_*}(\bm{x})\|_{W^{r,\infty}(\Omega_{\bm{m}_*})} \notag\\
 &\le 2^{s}(s-1)^{s-1}s^{d}C_0\binom{s+d-1}{d}\epsilon_1 \notag\\
 &\le\epsilon.
\end{align*}
By the triangle inequality, 
\begin{align*}
 \|\psi_{\bm{m}_*}-f\|_{W^{r,\infty}(\Omega_{\bm{m}_*})}
 &\le
 \|\psi_{\bm{m}_*}-f_{K,\bm{m}_*}\|_{W^{r,\infty}(\Omega_{\bm{m}_*})}+\|f-f_{K,\bm{m}_*}\|_{W^{r,\infty}(\Omega_{\bm{m}_*})}
 \notag\\
 &\le
\epsilon+C_{0}\,K^{r-s}.
\end{align*}
Lastly, we go back to each individual construction to compute the width and depth of our architecture $\psi_{\bm{m}_*}$. $\bm{\varphi}_2$ has width $2d$ and depth $1$, $\phi_{\bm{\alpha}}$ has width $d$ and depth $2$, and hence $\phi_{\bm{\alpha}}\circ\bm{\varphi}_2$ has width $2d$ and depth $3$.

Since $\varphi_{\bm{\alpha}}$ has width $\le |\alpha|+3d$ together with depth $\le 2|\alpha|+2d-4\le2s+2d-6$ and $\varphi_4$ has width $6$ and depth $2$, $\varphi_4(\phi_{\bm{\alpha}}\circ\bm{\varphi}_2(\bm{x}), \varphi_{\bm{\alpha}}(\bm{x}))$ has width $\le |\alpha|+5d$ and depth $\le\max\left\{ 2s+2d-6, 3\right\}+2$. Taking summation, we obtain a $\DUAF_n$-activated (or $\EUAF$-activated) network bound at:
\[
\begin{aligned}
\text{Width}(\psi_{\bm{m}_*})
&\le \sum_{|\bm{\alpha}|\le s-1}\big(|\bm{\alpha}|+5d\big)
 = \frac{d}{d+1}(s+5d+4)\,\binom{s+d-1}{d}=4^{-d}N_{s,d}-d-4,\\
\text{Depth}(\psi_{\bm{m}_*})
&\le \max\left\{ 2s+2d-6, 3\right\}+2
=\max\!\left\{ 2s+2d-4,\; 5 \right\}.\\
\end{aligned}
\]
\end{proof}

\subsection{Sobolev approximation on a subcube}\label{sec:sobolev-subcube}

\begin{theorem}[Shift-based $W^{1,\infty}$-approximation in $(0,1)^d$ for $s=2$]
\label{thm:shift-W1-d}
Let $d\in\mathbb{N}$ and assume $f\in W^{2,\infty}((0,1)^d)$ with $\|f\|_{W^{2,\infty}((0,1)^d)}\le 1$.
Then, for any $\varepsilon>0$, there exists a function $\Phi$
generated by an $\EUAF$--activated network with width no greater than $4^d(5d^2+8d+3)$ and depth no greater than $2d+5$, such that
\[
  \|\Phi-f\|_{W^{1,\infty}((0,9/10)^d)}<\varepsilon.
\]
\end{theorem}

\begin{proof}
We split the proof into several steps.

\noindent\textbf{Step 1}: Extension and Shifts.

Extend $f$ from $(0,1)^d$ to $\mathbb R^d$ by using Sobolev Extension Theorem. This yields 
$\tilde f\in W^{2,\infty}(\mathbb R^d)$ with 
$\|\tilde f\|_{W^{2,\infty}(\mathbb R^d)}\le C$.
where C is a constant only depends on $d$. For notational simplicity, we still denote the extension by $f$.

Fix an integer $K\in\mathbb{N}$, with $K>10$.
For each multi-index $\bm{i} = (i_1,\ldots,i_d)\in\{1,2,3,4\}^d$,
define the shift vector and the shifted function
\[
  \bm{\tau}_{\bm{i}}
  := \frac{1}{4K}(i_1,\ldots,i_d)\in\mathbb{R}^d, \qquad
  f_{\bm{i}}(\bm{x}) := f(\bm{x} - \bm{\tau}_{\bm{i}}),
  \qquad \bm{x}\in[0,1]^d.
\]
Each $f_{\bm{i}}\in W^{2,\infty}((0,1)^d)$ satisfies
$\|f_{\bm{i}}\|_{W^{2,\infty}((0,1)^d)}\le C$ uniformly in $\bm{i}$ and $K$. In the local approximation step below, Proposition~\ref{prop:local} is applied to $C^{-1}f_{\bm{i}}$, and the resulting output network is multiplied by $C$; enlarging the constant $C_0$ if necessary, this yields the displayed estimates for $f_{\bm{i}}$.

\smallskip
\noindent\textbf{Step 2}: Multidimensional Gate Functions.

Define the one-dimensional gate function
\[
  \psi(t) := \EUAF\big(t+1-\EUAF(t+1)\big),\qquad t\in\mathbb{R}.
\]
Then $\psi$ satisfies:
\begin{itemize}
  \item $0\le\psi(t)\le 1$ for all $t\in\mathbb{R}$;
  \item  $|\psi'(t)|\le 2$ for almost every $x\geq0$;
  \item for all $x\ge 0$,
        \[
           \psi\Big(x+\frac{1}{2}\Big)+\psi\Big(x+1\Big)+\psi\Big(x+\frac{3}{2}\Big)+\psi\Big(x+2\Big) = 1.
        \]
    (See Figure~\ref{fig:psi1234} for illustration.)
\end{itemize}

\begin{figure}[htbp]        
	\centering
	\begin{minipage}{0.985\textwidth}
		\centering
		\begin{minipage}[b]{0.48\textwidth}
			\centering            
			\includegraphics[width=0.985\textwidth]{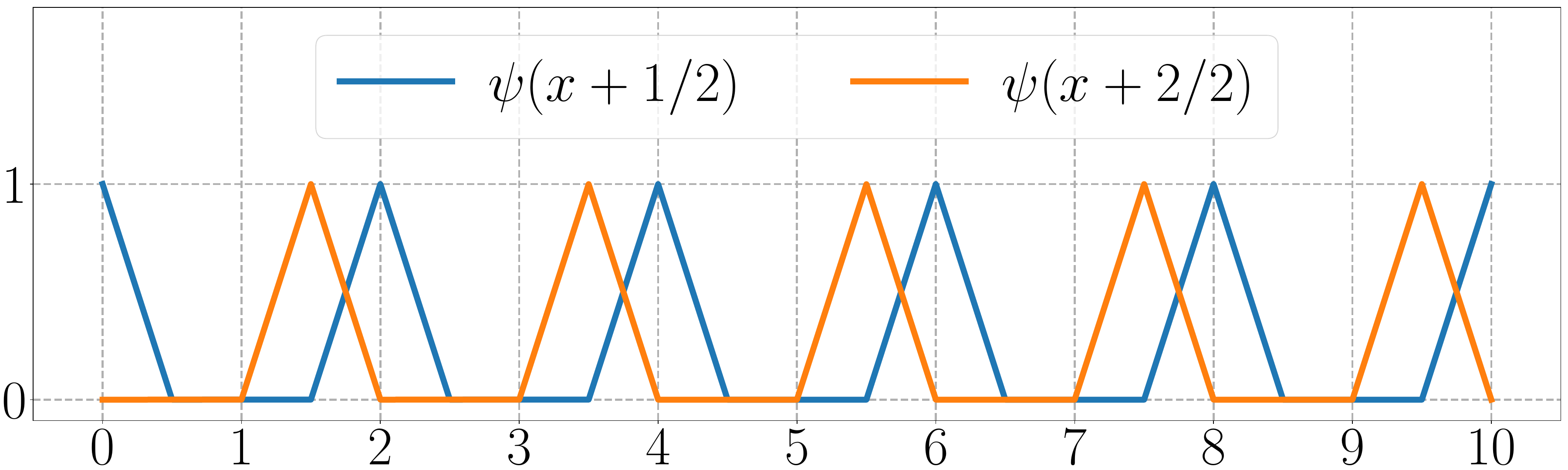}
		\end{minipage}
		\begin{minipage}[b]{0.48\textwidth}
			\centering            
			\includegraphics[width=0.985\textwidth]{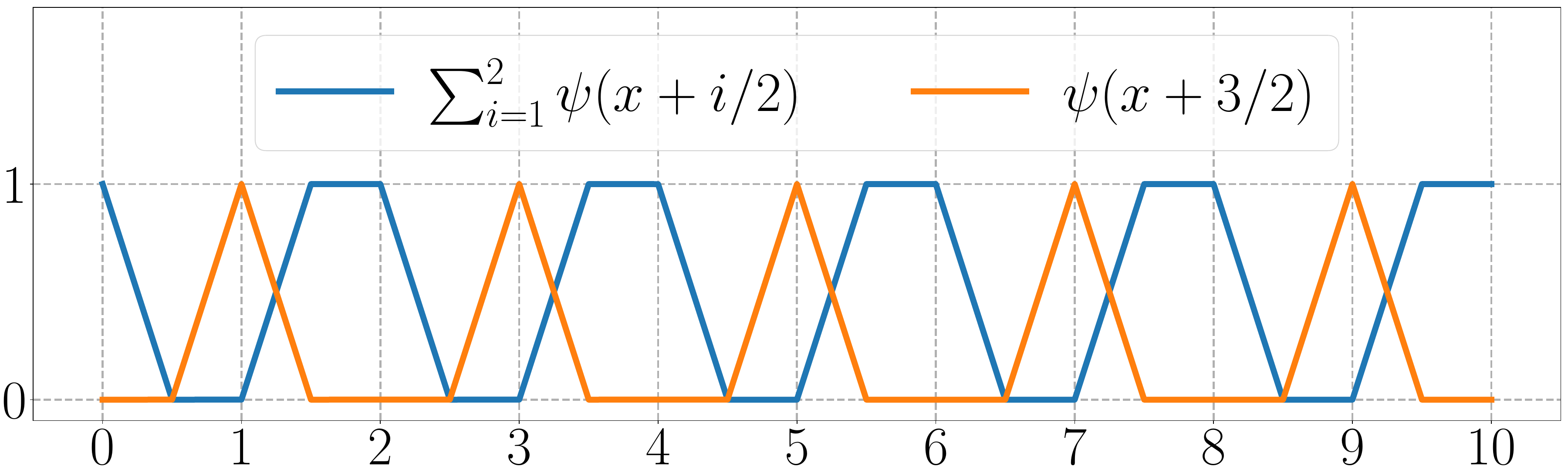}
		\end{minipage}
		
		\vspace*{4pt}
			\begin{minipage}[b]{0.48\textwidth}
		\centering            
		\includegraphics[width=0.985\textwidth]{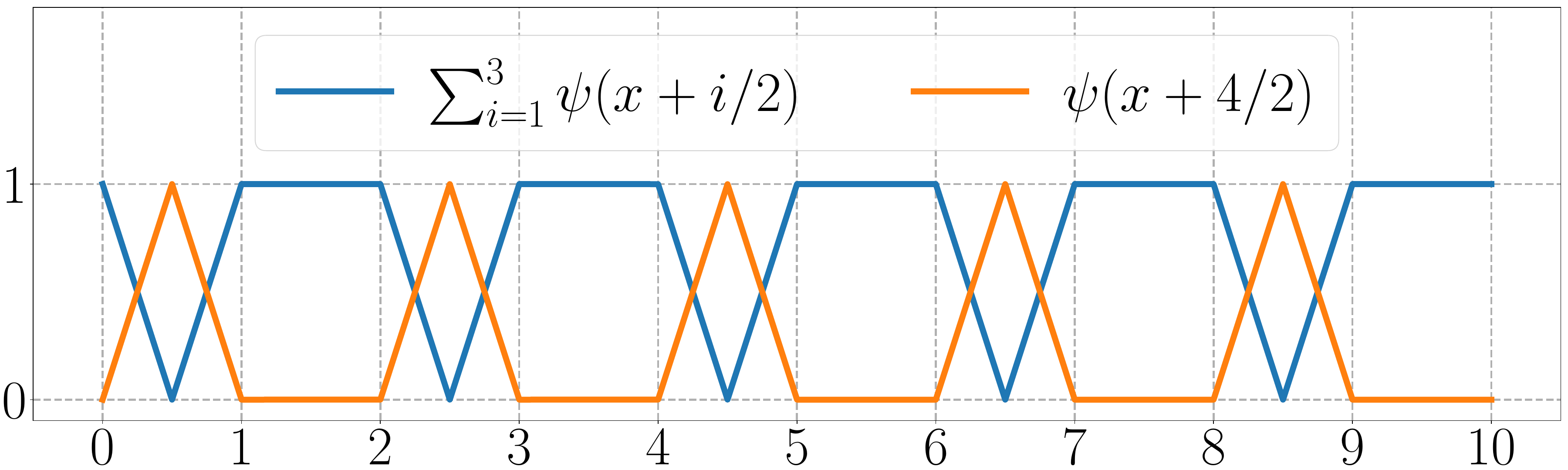}
	\end{minipage}
			\begin{minipage}[b]{0.485\textwidth}
	\centering            
	\includegraphics[width=0.985\textwidth]{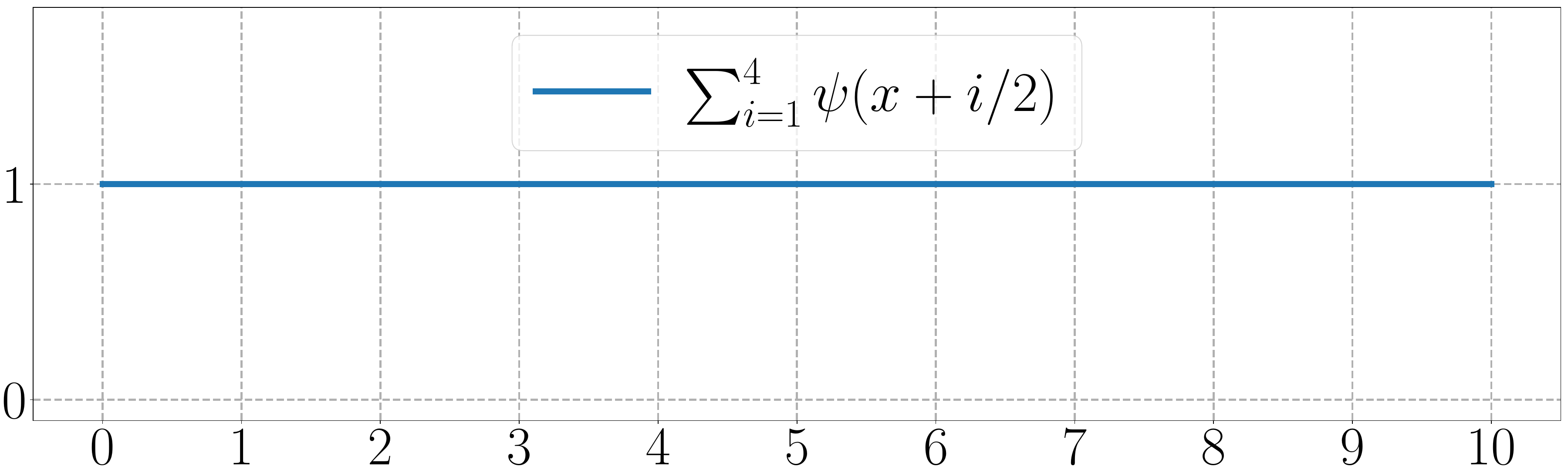}
\end{minipage}
	\end{minipage}
	\caption{Illustration of $\sum_{i=1}^{4}\psi(x+i/2)=1$ for any  $x\in [0,10]$}
	\label{fig:psi1234}
\end{figure}

\smallskip

Define, for $\bm{i}=(i_1,\ldots,i_d)\in\{1,2,3,4\}^d$ and 
$\bm{x}=(x_1,\ldots,x_d)\in\mathbb{R}^d$,
\[
  \Psi_{\bm{i}}(\bm{x})
  := \prod_{j=1}^d \psi\Big(2Kx_j + \frac{i_j}{2}\Big).
\]
Then $0\le \Psi_{\bm{i}}\le 1$, and for all $\bm{x}\in[0,\infty)^d$,
\[
  \sum_{\bm{i}\in\{1,2,3,4\}^d} \Psi_{\bm{i}}(\bm{x})
  = \sum_{\bm{i}\in\{1,2,3,4\}^d}\prod_{j=1}^d \psi\Big(2Kx_j + \frac{i_j}{2}\Big)=\prod_{j=1}^d \Big(\sum_{s=1}^4 \psi(2Kx_j + \tfrac{s}{2})\Big)
  = \prod_{j=1}^d 1=1.
\]
Thus $(\Psi_{\bm{i}})_{\bm{i}}$ is a partition of unity on $[0,\infty)^d$ together with
\[
  \big\|\Psi_{\bm{i}}\big\|_{L^{\infty}(\mathbb{R}^d)}\le 1, \qquad \big\|\Psi_{\bm{i}}\big\|_{W^{1,\infty}([0,\infty)^d)}
  \le 4K\,
\]
\smallskip
\noindent\textbf{Step 3}: Local Approximation for each $f_{\bm{i}}$.

Let $ \varepsilon_0 := \frac{\varepsilon}{4^d(8+10C_0)K} $.  Apply Proposition~\ref{prop:local} with $s=2$ and 
$\bm{m}_*=(1,\ldots,1)\in\{1,2\}^d$ to the normalized function $C^{-1}f_{\bm{i}}$, with local tolerance $C^{-1}\varepsilon_0$, and then multiply the output by $C$.
We obtain, for each $\bm{i}\in\{1,2,3,4\}^d$, an $\EUAF$--activated network 
$\varphi_{\bm{i}}$ such that
\begin{equation}\label{eq:local-phi-i}
  \|f_{\bm{i}} - \varphi_{\bm{i}}\|_{W^{r,\infty}(\Omega_{\bm{m}_*})}
  \le \varepsilon_0 + C_0 K^{\,r-2}, \qquad r=0,1,
\end{equation}
where 
\[
  \Omega_{\bm{m}_*} := \Omega_1 \times \cdots \times \Omega_1
  \subset [0,1]^d,
\]
and $\Omega_1$ is defined as in Definition~\ref{def:Omega-m}.

By construction of the gates, whenever 
$\Psi_{\bm{i}}(\bm{x})\neq 0$ or $\partial_\ell\Psi_{\bm{i}}(\bm{x})\neq 0$ for $\bm{x}\in[0,9/10]^d$, one has 
$\bm{x} + \bm{\tau}_{\bm{i}}\in\Omega_{\bm{m}_*}$, so
\eqref{eq:local-phi-i} applies at $\bm{x}+\bm{\tau}_{\bm{i}}$.

\smallskip
\noindent\textbf{Step 4}: Network Realization.

Define, for $\bm{x}\in[0,9/10]^d$,
\[
  F_{\bm{i}}(\bm{x})
  := \varphi_{\bm{i}}(\bm{x}+\bm{\tau}_{\bm{i}})\,\Psi_{\bm{i}}(\bm{x}),
  \qquad
  G_{\bm{i}}(\bm{x})
  := f_{\bm{i}}(\bm{x}+\bm{\tau}_{\bm{i}})\,\Psi_{\bm{i}}(\bm{x})
   = f(\bm{x})\,\Psi_{\bm{i}}(\bm{x}).
\]
Set
\[
  \Phi(\bm{x}) := \sum_{\bm{i}\in\{1,2,3,4\}^d} F_{\bm{i}}(\bm{x}).
\]
Since $\sum_{\bm{i}}\Psi_{\bm{i}}(\bm{x})=1$ for $\bm{x}\in[0,\frac{9}{10}]^d$, we have
\[
  f(\bm{x}) = \sum_{\bm{i}} G_{\bm{i}}(\bm{x}),\qquad \bm{x}\in[0,9/10]^d.
\]
Notice that on $[0,\frac{9}{10}]$, any $\psi\Big(2Kx+\frac{i}{2}\Big), i=1,2,3,4$ can be realized as 
\[
\psi\Big(2Kx+\frac{i}{2}\Big)=\EUAF\left((2K+1)\cdot\EUAF\left(\frac{2Kx+\frac{i}{2}+1}{2K+1}\right)-\EUAF\Big(2Kx+\frac{i}{2}+1\Big)\right),
\]
which is an $\EUAF$--activated network of width $2$ and depth $2$. This is due to the fact that $2Kx+\frac{i}{2}\le 2(\frac{9K}{10}+1)\le 2K$ because $K\ge 10$.

Hence, on $[0, \frac{9}{10}]^d$, $\Psi_{\bm{i}}$ can be realized by composing these subnetworks with the $d$-tuple multiplication map, which results in a network of width $\max\left\{2d, 7+d\right\}$ and depth $2d$. $F_{\bm{i}}(\bm{x})=\varphi_4\left( \varphi_{\bm{i}}(\bm{x}+\bm{\tau}_{\bm{i}}), \Psi_{\bm{i}}(\bm{x})\right)$, where $\varphi_{\bm{i}}(\bm{x}+\bm{\tau}_{\bm{i}})$ is a network of width $\le (5d+6)d$ and depth $\max\left\{ 2d, 5\right\}$ by Proposition~\ref{prop:local} and $\varphi_4$ is the multiplication map. Summing over the $4^d$ multi-indices $\bm{i}$, $\Phi$ is an $\EUAF$--activated network with width $\le 4^d\max\{\,5d^2 + 8d,\; 5d^2 + 7d + 4\,\}\le 4^d(5d^2+8d+3)$ and depth $\le \max\left\{ 2d+2, 7\right\}\le2d+5$.

\smallskip
\noindent\textbf{Step 5}: $L^\infty$-error estimate.

For any $\bm{x}\in[0,9/10]^d$,
\[
  |\Phi(\bm{x})-f(\bm{x})|
  = \Big|\sum_{\bm{i}} \big(F_{\bm{i}}(\bm{x})-G_{\bm{i}}(\bm{x})\big)\Big|
  \le \sum_{\bm{i}} |F_{\bm{i}}(\bm{x})-G_{\bm{i}}(\bm{x})|.
\]
Whenever $\Psi_{\bm{i}}(\bm{x})\neq 0$, we have 
$\bm{x}+\bm{\tau}_{\bm{i}}\in\Omega_{\bm{m}_*}$, and hence, by 
\eqref{eq:local-phi-i} with $r=0$,
\[
  \big|\varphi_{\bm{i}}(\bm{x}+\bm{\tau}_{\bm{i}})
       - f_{\bm{i}}(\bm{x}+\bm{\tau}_{\bm{i}})\big|
  \le \varepsilon_0 + C_0 K^{-2}.
\]
Using $|\Psi_{\bm{i}}|\le 1$,
\[
  |F_{\bm{i}}(\bm{x})-G_{\bm{i}}(\bm{x})|
  \le \varepsilon_0 + C_0 K^{-2}.
\]
There are $4^d$ choices of $\bm{i}$, so taking the supremum over $[0,9/10]^d$ gives

\begin{equation}\label{eq:Linf-error}
  \|\Phi-f\|_{L^\infty([0,9/10]^d)}
  \le 4^d\varepsilon_0 + 4^d C_0 K^{-2}.
\end{equation}

\smallskip
\noindent\textbf{Step 6}: Gradient error estimate.

Since $f_{\bm{i}}$ and $\varphi_{\bm{i}}$ are in $W^{1,\infty}$ on 
$\Omega_{\bm{m}_*}$ and each $\Psi_{\bm{i}}$ is Lipschitz, the products
$F_{\bm{i}}$ and $G_{\bm{i}}$ are Lipschitz on $[0,9/10]^d$.
A direct computation  shows that for any $\ell=1,\ldots,d$

\[
  \frac{\partial \Phi}{\partial x_\ell}(\bm{x})-\frac{\partial f}{\partial x_\ell}(\bm{x})
  = \sum_{\bm{i}} \left(\big(\frac{\partial \varphi_{\bm{i}}}{\partial x_\ell}
         -\frac{\partial f_{\bm{i}}}{\partial x_\ell} \big)(\bm{x}+\bm{\tau}_{\bm{i}})\,
      \Psi_{\bm{i}}(\bm{x})+
     \big(\varphi_{\bm{i}}-f_{\bm{i}}\big)(\bm{x}+\bm{\tau}_{\bm{i}})\,
     \frac{\partial \Psi_{\bm{i}}}{\partial x_\ell}(\bm{x})\right).
\]

When $\Psi_{\bm{i}}(\bm{x})\neq 0$, we have $\bm{x}+\bm{\tau}_{\bm{i}}\in \Omega_{\bm{m}_*}$,
and by \eqref{eq:local-phi-i} with $r=1$,
\[
  \left|\frac{\partial \varphi_{\bm{i}}}{\partial x_\ell}(\bm{x}+\bm{\tau}_{\bm{i}})
       -\frac{\partial f_{\bm{i}}}{\partial x_\ell}(\bm{x}+\bm{\tau}_{\bm{i}})\right|
  \le \varepsilon_0 + C_0 K^{-1}.
\]
Since $|\Psi_{\bm{i}}|\le 1$, this gives
\[
  \left|\big(\frac{\partial \varphi_{\bm{i}}}{\partial x_\ell}
         -\frac{\partial f_{\bm{i}}}{\partial x_\ell}\big)(\bm{x}+\bm{\tau}_{\bm{i}})
        \,\Psi_{\bm{i}}(\bm{x})\right|
  \le \varepsilon_0 + C_0 K^{-1}.
\]
Similarly, by \eqref{eq:local-phi-i} with $r=0$ and 
$|\frac{\partial \Psi_{\bm{i}}}{\partial x_\ell}(\bm{x})|\le 4K$,
\[
  \left|\big(\varphi_{\bm{i}}-f_{\bm{i}}\big)(\bm{x}+\bm{\tau}_{\bm{i}})
        \,\frac{\partial \Psi_{\bm{i}}}{\partial x_\ell}(\bm{x})\right|
  \le 4K  \big(\varepsilon_0 + C_0 K^{-2}\big)
  = 4 K\varepsilon_0 + 4C_0  K^{-1}.
\]
Summing over the $4^d$ multi-indices $\bm{i}$ and using the triangle inequality, we obtain for almost every $\bm{x}\in[0,9/10]^d$
\[
  |\Phi-f|_{W^{1,\infty}((0,9/10)^d)}
  \le 4^d(4+5C_0)\Big(\varepsilon_0 + K\varepsilon_0 + K^{-1}\Big).
\]

\smallskip
\noindent\textbf{Step 7}: Choice of  $K$.

Recall that we chose
  $\varepsilon_0 = \frac{\varepsilon}{4^d(8+10C_0)K}$.
Now, we set $K:=\max\{11,\lceil4+4^d8(5C_0+4)\varepsilon^{-1}\rceil\}$.
Substituting into \eqref{eq:Linf-error},
\[
  \|\Phi-f\|_{L^\infty((0,9/10)^d)}
  \le \frac{\varepsilon }{(8+10C_0)K}+ \frac{4^d C_0} {K^{2}}
  \le \frac{\varepsilon }{(8+10C_0)}+ \frac{4^d C_0} {4\frac{4^d8(5C_0+4)}{\epsilon}}\le \frac{\varepsilon }{8}+\frac{\varepsilon }{8}=\frac{\varepsilon }{4},
\]
Similarly,
\[
  |\Phi-f|_{W^{1,\infty}((0,9/10)^d)}
  \le \frac{\varepsilon}{2K} + \frac{\varepsilon}{2} + \frac{4^d(4+5C_0)}{K}
  \le \frac{\varepsilon}{8} + \frac{\varepsilon}{2} + \frac{\varepsilon}{8}=\frac{3\varepsilon}{4}.
\]
Therefore,
\[
  \|\Phi-f\|_{W^{1,\infty}((0,9/10)^d)}
  \le \|\Phi-f\|_{L^\infty((0,9/10)^d)} + |\Phi-f|_{W^{1,\infty}((0,9/10)^d)}
  < \frac{\varepsilon}{4} + \frac{3\varepsilon}{4}
  = \varepsilon.
\]
This completes the proof.
\end{proof}

\begin{theorem} [Shift-based $W^{s,\infty}$-approximation]
\label{thm:shift-Ws-d}
Let $d, s\in\mathbb{N}$ and $n\in\mathbb N\cup\{\infty\}$ with $1\le s\le n+1$. Assume $f\in W^{s,\infty}((0,1)^d)$ with 
$\|f\|_{W^{s,\infty}((0,1)^d)}\le 1$. Then, for any $\varepsilon>0$, there exists a $\DUAF_n$--activated network $\Phi$ of total width no greater than $N_{s,d}$ and depth no greater than $ L_{s,d}$ such that
\[
  \|\Phi-f\|_{W^{s-1,\infty}((0,9/10)^d)}<\varepsilon.
\]
\end{theorem}
\begin{proof}
We use the same shifted-grid construction as in Theorem~\ref{thm:shift-W1-d}. Fix $K>10$, to be chosen at the end. Extend $f$ to $W^{s,\infty}(\mathbb R^d)$ and set
\[
\bm{\tau}_{\bm{i}}:=\frac{1}{4K}\bm{i},
\qquad
f_{\bm{i}}(\bm{x}):=f(\bm{x}-\bm{\tau}_{\bm{i}}),
\qquad
\bm{i}\in\{1,2,3,4\}^d.
\]
The extension theorem gives $C=C(s,d)$ with
$\|f_{\bm{i}}\|_{W^{s,\infty}((0,1)^d)}\le C$ uniformly in $\bm{i}$ and $K$.

Let $\varepsilon_0:=K^{-s}$. Applying Proposition~\ref{prop:local} and Lemma~\ref{lem:sigma_n-properties}(i) to $C^{-1}f_{\bm{i}}$ on $\Omega_{\bm{m}_*}$, and then multiplying the output by $C$, gives, after enlarging $C_0$ if necessary, networks $\varphi_{\bm{i}}$ of width $<4^{-d}N_{s,d}-d-4$ and depth $<\max\{2s+2d-4,5\}$ such that
\begin{equation}\label{eq:shift-Ws-local-error}
\|f_{\bm{i}}-\varphi_{\bm{i}}\|_{W^{r,\infty}(\Omega_{\bm{m}_*})}
\le \varepsilon_0+C_0K^{r-s},
\qquad r=0,1,\ldots,s-1.
\end{equation}
Define
\[
F_{\bm{i}}(\bm{x}):=\varphi_{\bm{i}}(\bm{x}+\bm{\tau}_{\bm{i}})\Psi_{\bm{i},n,K}(\bm{x}),
\quad
G_{\bm{i}}(\bm{x}):=f_{\bm{i}}(\bm{x}+\bm{\tau}_{\bm{i}})\Psi_{\bm{i},n,K}(\bm{x}),
\quad
\Phi:=\sum_{\bm{i}\in\{1,2,3,4\}^d}F_{\bm{i}}.
\]
Here $F_{\bm{i}}\in W^{s-1,\infty}((0,9/10)^d)$ without extending from the cellwise set: $\varphi_{\bm{i}}$ is a globally defined $C^n$ network, $\Psi_{\bm{i},n,K}\in C^n(\mathbb R^d)$ by Lemma~\ref{lem:gn-properties}(ii), and $s-1\le n$. The same Sobolev Leibniz rule gives $G_{\bm{i}}\in W^{s-1,\infty}$.

The architecture of each $F_{\bm{i}}$ is obtained by running the local approximant and the gate in parallel and multiplying them. The parallel pair has width
$(4^{-d}N_{s,d}-d-4)+(d+4)=4^{-d}N_{s,d}$ and depth at most $\max\{2s+2d-4,2d-1,5\}$; the final width-$6$, depth-$2$ multiplier gives width $4^{-d}N_{s,d}$ and depth
$\max\{2s+2d-2,2d+1,7\}=L_{s,d}$. Padding the $4^d$ branches to equal depth and summing affinely gives width at most $N_{s,d}$ and depth at most $L_{s,d}$.

For $|\bm{\alpha}|\le s-1$, Lemma~\ref{lem:gn-properties}(v) ensures that whenever $D^{\bm{\alpha}-\bm{\beta}}\Psi_{\bm{i},n,K}(\bm{x})\ne0$, the shifted point $\bm{x}+\bm{\tau}_{\bm{i}}$ lies in $\Omega_{\bm{m}_*}$. Therefore, by \eqref{eq:shift-Ws-local-error}, Lemma~\ref{lem:gn-properties}(iii), and Leibniz's rule,
\[
\begin{aligned}
|D^{\bm{\alpha}}(F_{\bm{i}}-G_{\bm{i}})(\bm{x})|
&\le \sum_{\bm{\beta}\le\bm{\alpha}}\binom{\bm{\alpha}}{\bm{\beta}}
\bigl(\varepsilon_0+C_0K^{|\bm{\beta}|-s}\bigr)C_{s,d}K^{|\bm{\alpha}-\bm{\beta}|} \\
&\le C(d,s)\bigl(\varepsilon_0K^{|\bm{\alpha}|}+K^{|\bm{\alpha}|-s}\bigr)
\le 2C(d,s)K^{-1}.
\end{aligned}
\]
Since $\sum_{\bm{i}}G_{\bm{i}}=f$ on $(0,9/10)^d$ by the partition-of-unity property of the gates,
\[
\|D^{\bm{\alpha}}(\Phi-f)\|_{L^\infty((0,9/10)^d)}
\le \sum_{\bm{i}\in\{1,2,3,4\}^d}
\|D^{\bm{\alpha}}(F_{\bm{i}}-G_{\bm{i}})\|_{L^\infty((0,9/10)^d)}
\le 2^{2d+1}C(d,s)K^{-1}.
\]
Choose $K:=\max\{11,\lceil2^{2d+2}C(d,s)\varepsilon^{-1}\rceil\}$. Taking the maximum over $|\bm{\alpha}|\le s-1$ gives
\[
\|\Phi-f\|_{W^{s-1,\infty}((0,9/10)^d)}<\varepsilon.
\]
This completes the proof.
\end{proof}

\subsection{Global Sobolev approximation on $(a,b)^d$}\label{sec:global-sobolev-cube}

\begin{proof}[Proof of Theorem~\ref{thm:global-w2} and~\ref{thm:global-ws}]\label{proof:thm-global-w2-ws}

 Define the affine transformation that maps $[0, \frac{9}{10}]^d$ to $[a, b]^d$:
 \[
 \bm{T}(\bm{x})=\bm{A}\bm{x}+\bm{b},\qquad \text{where } \bm{A}=\tfrac{10}{9}(b-a)\bm{I}_d,\qquad \bm{b}=a\bm{1}.
 \]
 Leaving out the trivial case $f=0$, it follows that $f\circ\bm{T} \in W^{s,\infty}((0,9/10)^d)$. Let $Ef$ be the Sobolev extension of $f\circ\bm{T}$ in $W^{s,\infty}(\mathbb R^d)$ and set $\mu=\frac{10(b-a)}{9}$. Then there exists a constant $C_{s,d}>0$, depending only on $s$ and $d$, such that 
 \[
 \|Ef\|_{W^{s,\infty}(\mathbb R^d)}\le C_{s,d} \|f\circ\bm{T}\|_{W^{s,\infty}((0,9/10)^d)}\le C_{s,d} \max\{1,\mu^s\} \|f\|_{W^{s,\infty}((a,b)^d)}.
 \]
 Denote the right-hand side $\lambda$. Since $\lambda>0$, it follows that $\|\lambda^{-1} Ef\|_{W^{s,\infty}((0,1)^d)} \le1$. For Theorem~\ref{thm:global-w2}, we take $s=2$ and use Theorem~\ref{thm:shift-W1-d}; for Theorem~\ref{thm:global-ws}, we use Theorem~\ref{thm:shift-Ws-d}. In either case, there exists an architecture $\Phi$ with the prescribed width and depth bounds such that 
 \[
 \|\Phi-\lambda^{-1} Ef\|_{W^{s-1,\infty}((0,9/10)^d)}< \frac{\varepsilon}{\lambda\max\{1, \mu^{1-s}\}}.
 \]
 Notice that in $(a,b)^d$, we always have $f(x)=Ef\circ\bm{T}^{-1}(x)$. This implies that
 \[
 \begin{aligned}
 \|\lambda\Phi\circ\bm{T}^{-1}-f\|_{W^{s-1,\infty}((a,b)^d)}
 &=\lambda\|(\Phi-\lambda^{-1} Ef)\circ\bm{T}^{-1}\|_{W^{s-1,\infty}((a,b)^d)} \\
 &\le \lambda \max\{1, \mu^{1-s}\}\|\Phi-\lambda^{-1} Ef\|_{W^{s-1,\infty}((0,9/10)^d)} \\
 &<\varepsilon,
 \end{aligned}
 \]
 where the architecture $\lambda\Phi\circ\bm{T}^{-1}$ is exactly what we need.
 \end{proof}

\section{Proof of Theorem~\ref{thm:QKST}}\label{sec:proof-qkst}

\begin{proof}[Proof of Theorem~\ref{thm:QKST}]
Set $r:=s-1$ and $\Omega:=(a,b)^d$. By the definition of $K_{d,Q}^s(\Omega)$, choose a superposition approximant
\[
  f_{\varepsilon}=\sum_{q=1}^Q g_q(S_q),
  \qquad
  S_q(\bm{x}):=\sum_{p=1}^d h_{q,p}(x_p),
  \qquad
  g_q,h_{q,p}\in W^{s,\infty}_{\mathrm{loc}}(\mathbb R),
\]
such that $\|f-f_{\varepsilon}\|_{W^{r,\infty}(\Omega)}<\varepsilon/2$.  Put
\[
M:=1+\max_{1\le q\le Q}\|S_q\|_{W^{r,\infty}(\Omega)},
\qquad I:=(-M-2,M+2),
\]
and
\[
K_g:=1+\max_{1\le q\le Q}\|g_q\|_{W^{s,\infty}(I)},
\qquad
\theta(t):=\sum_{j=0}^{r}t^j,
\qquad
J:=\theta(M+1)+\theta'(M+1).
\]
For each multi-index $\bm{\alpha}$ with $|\bm{\alpha}|\le r$, the multivariate Fa\`a di Bruno formula gives polynomials $B_{\bm{\alpha},k}$, $0\le k\le |\bm{\alpha}|$, with nonnegative coefficients, depending only on $d$ and $\bm{\alpha}$, such that
\[
D^{\bm{\alpha}}(g\circ S)
=
\sum_{k=0}^{|\bm{\alpha}|}
B_{\bm{\alpha},k}\big(D^{\bm{\beta}}S:1\le |\bm{\beta}|\le |\bm{\alpha}|\big)
\,g^{(k)}(S),
\]
for smooth $g$ and $S$; for $\bm{\alpha}=0$ we use the convention $B_{\bm{0},0}=1$. The same identity and the estimates below hold for the present Sobolev functions in the weak sense: extend the one-dimensional functions slightly inside $I$, mollify $g$ and $S$, apply the smooth formula with constants depending only on the displayed $W^{s,\infty}$ and $W^{r,\infty}$ bounds, and pass to weak derivatives. Let $C_B\ge1$ be chosen so that, for all these finitely many Bell polynomials,
\[
|B_{\bm{\alpha},k}(\bm{Y})|\le C_B\|\bm{Y}\|_\infty^k,
\qquad
|B_{\bm{\alpha},k}(\bm{Y})-B_{\bm{\alpha},k}(\bm{Z})|
\le C_B k A^{k-1}\|\bm{Y}-\bm{Z}\|_\infty,
\]
whenever $\|\bm{Y}\|_\infty,\|\bm{Z}\|_\infty\le A$; when $k=0$, the second right-hand side is interpreted as $0$. Define
\[
C_1:=C_BK_gJ,
\qquad
C_2:=C_B\theta(M+1),
\qquad
\delta:=\min\left\{\frac1d,\frac{\varepsilon}{4QC_1d},
\frac{\varepsilon}{4QC_2}\right\}.
\]

By Corollary~\ref{cor:global-winfty} in dimension one, for every pair $(q,p)$ and every $q$ there exist $\DUAF_\infty$--activated networks $\varphi_{q,p}$ and $\Phi_q$ with one-dimensional width and depth
\[
W_0=2(s^2+9s+10),
\qquad
L_0=\max\{2s,7\},
\]
such that
\[
\|\varphi_{q,p}-h_{q,p}\|_{W^{r,\infty}((a,b))}<\delta,
\qquad
\|\Phi_q-g_q\|_{W^{r,\infty}(I)}<\delta.
\]
Set $R_q(\bm{x}):=\sum_{p=1}^d\varphi_{q,p}(x_p)$. Then
\[
\|S_q-R_q\|_{W^{r,\infty}(\Omega)}
\le\sum_{p=1}^d\|h_{q,p}-\varphi_{q,p}\|_{W^{r,\infty}((a,b))}
\le d\delta\le1,
\]
and hence $\|R_q\|_{W^{r,\infty}(\Omega)}\le M+1$. In particular, both $S_q$ and $R_q$ take values in $[-M-1,M+1]\subset I$.

We now compare $g_q(S_q)$ and $g_q(R_q)$. For $|\bm{\alpha}|\le r$, the above formula and the Lipschitz bound $|g_q^{(k)}(S_q)-g_q^{(k)}(R_q)|\le K_g|S_q-R_q|$ for $0\le k\le r$ give, almost everywhere on $\Omega$,
\[
\begin{aligned}
&\big|D^{\bm{\alpha}}(g_q(S_q)-g_q(R_q))\big| \\
&\le
\sum_{k=0}^{|\bm{\alpha}|}
\Big|B_{\bm{\alpha},k}(D^{\bm{\beta}}S_q)-B_{\bm{\alpha},k}(D^{\bm{\beta}}R_q)\Big|
\,|g_q^{(k)}(R_q)| \\
&\quad+
\sum_{k=0}^{|\bm{\alpha}|}
\Big|B_{\bm{\alpha},k}(D^{\bm{\beta}}S_q)\Big|
\,|g_q^{(k)}(S_q)-g_q^{(k)}(R_q)| \\
&\le
\sum_{k=0}^{|\bm{\alpha}|}
\left(C_B k(M+1)^{k-1}d\delta\,K_g
     +C_B(M+1)^kK_gd\delta\right) \\
&\le
C_BK_g\bigl(\theta'(M+1)+\theta(M+1)\bigr)d\delta
=C_1d\delta
\le \frac{\varepsilon}{4Q}.
\end{aligned}
\]
Taking the maximum over $|\bm{\alpha}|\le r$, we obtain
\[
\|g_q(S_q)-g_q(R_q)\|_{W^{r,\infty}(\Omega)}\le\frac{\varepsilon}{4Q}.
\]
For the outer-network error, apply the same formula to $H_q:=g_q-\Phi_q$. Since $\|H_q\|_{W^{r,\infty}(I)}<\delta$ and $\|R_q\|_{W^{r,\infty}(\Omega)}\le M+1$, for $|\bm{\alpha}|\le r$ we have
\[
\begin{aligned}
\big|D^{\bm{\alpha}}(H_q(R_q))\big|
&\le
\sum_{k=0}^{|\bm{\alpha}|}
\Big|B_{\bm{\alpha},k}(D^{\bm{\beta}}R_q)\Big|\,
\|H_q^{(k)}\|_{L^\infty(I)} \\
&\le
\sum_{k=0}^{|\bm{\alpha}|} C_B(M+1)^k\delta
\le C_2\delta
\le \frac{\varepsilon}{4Q}.
\end{aligned}
\]
Therefore
\[
\|g_q(R_q)-\Phi_q(R_q)\|_{W^{r,\infty}(\Omega)}\le\frac{\varepsilon}{4Q}.
\]
Combining the last two estimates and summing over $q$ gives
\[
\left\|f_\varepsilon-\sum_{q=1}^Q \Phi_q(R_q)\right\|_{W^{r,\infty}(\Omega)}
\le
\sum_{q=1}^Q\left(\frac{\varepsilon}{4Q}+\frac{\varepsilon}{4Q}\right)
=\frac{\varepsilon}{2}.
\]
Together with $\|f-f_\varepsilon\|_{W^{r,\infty}(\Omega)}<\varepsilon/2$, this proves the approximation estimate.

It remains to count the architecture. Run the $dQ$ inner networks $\varphi_{q,p}$ in parallel after affine coordinate-selection maps; this has width $dQW_0$ and depth $L_0$. The affine output map of this block forms the $Q$ sums $R_q=\sum_{p=1}^d\varphi_{q,p}(x_p)$ and can be absorbed into the input affine maps of the $Q$ outer networks. Running the $Q$ outer networks $\Phi_q$ in parallel has width $QW_0\le dQW_0$ and depth $L_0$, and the final summation over the $Q$ channels is affine. Hence the total width is
\[
 dQW_0=2dQ(s^2+9s+10),
\]
and the total depth is
\[
2L_0=2\max\{2s,7\}=\max\{4s,14\}.
\]
This completes the proof.
\end{proof}

\section{Proof of Theorem~\ref{thm:sigmoidal}}\label{D}

The sigmoidal theorem is proved by converting the already-constructed
$\DUAF_n$ networks into $\widetilde{\DUAF}_n$ networks.  The conversion has two
parts.  First, Lemma~\ref{lem:old-to-new-1} below is a general replacement principle:
if one activation can approximate another activation in Sobolev norm on arbitrary
compact intervals using a fixed-size subnet, then every network using the original
activation can be replaced layer by layer by a network using the new activation.
The width and depth increase multiplicatively by the size of the activation-simulation
subnet.

Second, Theorem~\ref{thm:sigmoidal-loc}, which is the technical heart of this section, supplies the required activation simulation for
$\varrho=\DUAF_n$ and $\tilde\varrho=\widetilde{\DUAF}_n$.  After applying
Theorem~\ref{thm:global-ws} to obtain a fixed-size $\DUAF_n$ network that
approximates $f$, the replacement lemma gives a
$\widetilde{\DUAF}_n$ network approximating this intermediate network on a compact
set containing $(a,b)^d$.  A final triangle inequality splits the total error into
the approximation error of the $\DUAF_n$ network and the replacement error.

\begin{lemma}\label{lem:old-to-new-1}
Let $\varrho, \tilde\varrho:\mathbb R\to\mathbb R$ be two functions with $\varrho\in C^n(\mathbb R)$, where $n\in\mathbb N$. Suppose that, for each $M>0$ and $\eta\in(0,1)$, there exists a function $\varrho_{M,\eta}$ realized by a $\tilde\varrho$-activated network with width $\tilde N$ and depth $\tilde L$ ($\tilde N$ and $\tilde L$ independent of $M$ and $\eta$) such that
\[
  \|\varrho_{M,\eta}-\varrho\|_{W^{m,\infty}([-M,M])}\xrightarrow{\,\eta\to0\,}0,
\]
for some integer $m$ with $0\le m\le n-1$. Then, for any $K>0$ and any $d$-input $\varrho$-activated network $\psi$ with width $N$ and depth $L$, there exist $\tilde\varrho$-activated networks $\psi_{K,\eta}$ with width $N\tilde N$ and depth $L\tilde L$ such that
\[
  \|\psi_{K,\eta}-\psi\|_{W^{m,\infty}([-K,K]^d)}\xrightarrow{\,\eta\to0\,}0.
\]
\end{lemma}
\begin{proof}
Write the original network as
\[
\psi(\bm{x})=\bm{A}_L\circ\bm{\varrho}\circ\bm{A}_{L-1}\circ\cdots\circ\bm{\varrho}\circ\bm{A}_1\circ\bm{\varrho}\circ\bm{A}_0(\bm{x}),
\]
where each $\bm{A}_i$ is affine and $\bm{\varrho}$ denotes componentwise application of
$\varrho$. Let $U_\ell$ be the exact intermediate map just before the $\ell$-th
activation layer. Since $[-K,K]^d$ is compact and $\varrho\in C^n$, all $U_\ell$ are
bounded in $W^{m,\infty}$ and all their ranges are contained in a common compact
interval. Choose $M$ so large that these ranges, enlarged by one, are contained in
$[-M,M]$.

Replacing each scalar activation by the componentwise subnet $\varrho_{M,\eta}$ gives
\[
\psi_{K,\eta}(\bm{x})=\bm{A}_L\circ\bm{\varrho}_{M,\eta}\circ\bm{A}_{L-1}\circ\cdots\circ
\bm{\varrho}_{M,\eta}\circ\bm{A}_1\circ\bm{\varrho}_{M,\eta}\circ\bm{A}_0(\bm{x}),
\]
with width at most $N\tilde N$ and depth at most $L\tilde L$ after the usual parallel
realization and identity padding.

We prove by induction over layers that the perturbed intermediate maps $U_{\ell,\eta}$
satisfy
\[
U_{\ell,\eta}\to U_\ell\quad\text{in }W^{m,\infty}([-K,K]^d),
\qquad
\sup_\eta\|U_{\ell,\eta}\|_{W^{m,\infty}([-K,K]^d)}<\infty.
\]
The claim is immediate for the input layer and is preserved by affine layers. For an
activation layer, assume the induction hypothesis for $U_{\ell,\eta}$ and restrict to
small enough $\eta$ so that the range of $U_{\ell,\eta}$ remains in $[-M,M]$. Then
\[
\varrho_{M,\eta}(U_{\ell,\eta})-\varrho(U_\ell)
=\bigl(\varrho_{M,\eta}-\varrho\bigr)(U_{\ell,\eta})
+\bigl(\varrho(U_{\ell,\eta})-\varrho(U_\ell)\bigr).
\]
We use the following standard Fa\`a di Bruno estimates on the compact interval
$[-M,M]$: if $H\in W^{m,\infty}([-M,M])$ and $G$ has range in $[-M,M]$ with a
uniform $W^{m,\infty}$ bound, then
\[
\|H\circ G\|_{W^{m,\infty}}
\le C\|H\|_{W^{m,\infty}([-M,M])};
\]
and if $h\in W^{m+1,\infty}([-M,M])$ and $G,\widetilde G$ have ranges in
$[-M,M]$ and a common $W^{m,\infty}$ bound, then
\[
\|h\circ G-h\circ\widetilde G\|_{W^{m,\infty}}
\le C\|G-\widetilde G\|_{W^{m,\infty}}.
\]
For smooth functions these are obtained by differentiating the compositions and
using the multivariate Fa\`a di Bruno formula; the Sobolev case follows by
mollification and passage to weak derivatives. Applying the first estimate to
$H=\varrho_{M,\eta}-\varrho$ and $G=U_{\ell,\eta}$ gives
\[
\| (\varrho_{M,\eta}-\varrho)(U_{\ell,\eta})\|_{W^{m,\infty}}
\le C\|\varrho_{M,\eta}-\varrho\|_{W^{m,\infty}([-M,M])}\to0.
\]
Applying the second estimate with $h=\varrho$ is legitimate because $m\le n-1$ and
$\varrho\in C^n$, hence $\varrho\in W^{m+1,\infty}([-M,M])$; it yields
\[
\|\varrho(U_{\ell,\eta})-\varrho(U_\ell)\|_{W^{m,\infty}}
\le C\|U_{\ell,\eta}-U_\ell\|_{W^{m,\infty}}\to0.
\]
The same two estimates also give a uniform $W^{m,\infty}$ bound for the next layer.
After finitely many layers we obtain
$\|\psi_{K,\eta}-\psi\|_{W^{m,\infty}([-K,K]^d)}\to0$.
\end{proof}

\begin{theorem}\label{thm:sigmoidal-loc}
Assume $n\in \mathbb N$ and $M>0$. For any $\eta\in (0,1)$, there exists an $\widetilde{\DUAF}_n$-activated network $\Phi_{M, \eta}$ with width no greater than $S_n$, where $S_n:=\max\{4n^2+19,26\}$, and depth $3n^2+5$, whose parameters depend on $n$, $M$, and $\eta$, such that
\[
\|\Phi_{M, \eta}-\DUAF_n\|_{W^{n-1,\infty}([-M, M])}\xrightarrow{\,\eta\to 0\,}
0.
\]
\end{theorem}
The proof of Theorem~\ref{thm:sigmoidal-loc} is technically involved.
For clarity of exposition, we defer the complete and detailed proof to
Section~\ref{sec:proof-sigmoidal-loc}. Together with Lemma~\ref{lem:old-to-new-1}, this result provides all the ingredients
needed to prove Theorem~\ref{thm:sigmoidal}. The remaining argument is
straightforward.

\begin{proof}[Proof of Theorem~\ref{thm:sigmoidal}]
Fix an integer $n\ge1$ and assume that $1\le s\le n$. By Theorem~\ref{thm:global-ws}, we know that for any $\varepsilon>0$ and $f\in W^{s, \infty}((a, b)^d)$, there exists a $\DUAF_n$--activated network $\Phi_{\varepsilon}$ with width less than $N_{s,d}$ and depth less than $L_{s,d}$ such that 
\[
\|\Phi_{\varepsilon}-f\|_{W^{s-1, \infty}((a, b)^d)}<\frac{\varepsilon}{2}. 
\]
Let $K>0$ such that $(a, b)^d\subseteq (-K, K)^d$. Lemma~\ref{lem:old-to-new-1} together with Theorem~\ref{thm:sigmoidal-loc} shows that there exists a $\eta(\varepsilon)>0$ small enough and an $\widetilde{\DUAF}_n$--activated network $\widetilde{\Phi}_{\varepsilon}:=\psi_{K,\eta(\varepsilon)}$ with width no greater than $\max\{4n^2+19,26\}N_{s,d}$ and depth no greater than $(3n^2+5)L_{s,d}$ such that 
\[
\|\widetilde{\Phi}_{\varepsilon}-\Phi_{\varepsilon}\|_{W^{s-1,\infty}((a, b)^d)}
\le
\|\widetilde{\Phi}_{\varepsilon}-\Phi_{\varepsilon}\|_{W^{s-1,\infty}([-K, K]^d)}<\frac{\varepsilon}{2}.
\]
Hence, applying the triangle inequality, we have
\[
\|\widetilde{\Phi}_{\varepsilon}-f\|_{W^{s-1, \infty}((a, b)^d)}
\le
\|\widetilde{\Phi}_{\varepsilon}-\Phi_{\varepsilon}\|_{W^{s-1,\infty}((a, b)^d)}
+
\|\Phi_{\varepsilon}-f\|_{W^{s-1, \infty}((a, b)^d)}
<\varepsilon.
\]
This completes the proof.
\end{proof}

\subsection{Proof of Theorem~\ref{thm:sigmoidal-loc}}\label{sec:proof-sigmoidal-loc}

This is the technical core of the sigmoidal reduction.  On the middle interval
$[-9/2,0]$, the activation $\widetilde{\DUAF}_n$ is exactly a scalar multiple of
$\DUAF_n$, so no approximation is needed there.  On the positive side, the integral
identity defining $\widetilde{\DUAF}_n$ implies that $\DUAF_n$ can be recovered from
$\widetilde{\DUAF}_n'$ after multiplying by $(x+1)^2/c_n$; finite differences of
$\widetilde{\DUAF}_n$ therefore approximate $\DUAF_n$ in Sobolev norm.  The far-left
region is treated analogously using the left integral branch and a factor $x^2/d_n$.

The remaining task is to glue these three pieces without losing derivative control.
We carefully build a smooth partition of unity from
$G(t)=t^{n^2}/(t^{n^2}+(1-t)^{n^2})$, which is realizable by
$\widetilde{\DUAF}_n$ networks using multiplication and reciprocal subnets.  The
transition width is $O(\eta)$, and the exponent $n^2$ together with the scaling
$\alpha=1+2/n$ ensures that the derivatives of the cutoffs are small enough away
from the transition layers.  Near the transition layers, the endpoint flatness of
$\DUAF_n$ compensates for the derivative growth of the cutoffs, yielding convergence
in $W^{n-1,\infty}([-M,M])$. We start by introducing the following lemma, which will be used later to construct partition functions.

\begin{lemma}\label{lem:reciprocal-dyadic}
Let $n, N\in\mathbb N$. There exists an $\widetilde{\DUAF}_n$--activated network $\varphi:\mathbb R\to\mathbb R$ with width $1$ and depth $N$ such that
\[
\varphi(x)=\frac{1}{x},\qquad x\in[2^{1-N},1].
\]
\end{lemma}

\begin{proof}
We first record a basic observation. If $0<\ell\le u$ satisfy $u/\ell\le \frac{44}{13}$, 
then $1/x$ can be represented exactly on $[\ell,u]$ by a single-layer $\widetilde{\DUAF}_n$--activated network. Indeed, setting $\lambda:=-\frac{13}{11\ell}$, we have for every $x\in[\ell,u]$,
\[
\lambda x\in\left[-\frac{13u}{11\ell},-\frac{13}{11}\right]\subset\left[-4,-\frac{13}{11}\right],
\]
and hence, by the rational branch of the activation,
\[
d_{n}^{-1}\widetilde{\DUAF}_n(\lambda x)=\DUAF_{n}(\lambda x)=-2-\frac{1}{\lambda x}.
\]
Therefore,
\[
-\lambda\bigl(d_{n}^{-1}\widetilde{\DUAF}_n(\lambda x)+2\bigr)=\frac{1}{x},
\qquad x\in[\ell,u].
\]

According to the above observation, we conclude that the following two $\widetilde{\DUAF}_n$--activated networks with width $1$ and depth $1$
\[
R(x):=\frac{26}{11}\left(d_{n}^{-1}\widetilde{\DUAF}_n\left(-\frac{26}{11}x\right)+2\right),
\quad
T(x):=
2-\frac{26}{11}\left(d_{n}^{-1}\widetilde{\DUAF}_n\left(-\frac{13}{11}(1+x)\right)+2\right)
\]
coincides with $1/x$ on $[1/2, 1]$ and $2x/(x+1)$ on $[0,1]$ respectively. We claim that for any $m\in\mathbb N_0$, 
\[
T^{m}(x)
:=
\underbrace{T\circ T\circ \cdots \circ T}_{m\text{ times}}(x)
=
\frac{2^{m}x}{1+(2^{m}-1)x},
\qquad x\in [0, 1].
\]
The case of $m=0, 1$ is obvious. Assume now the equality holds for $m-1\ge0$. Since $T([0,1])=[0,1]$, for $x\in[0,1]$, we have
\[
T^{m}(x)=T^{m-1}(T(x))=\frac{2^{m-1}T(x)}{1+(2^{m-1}-1)T(x)}=\frac{2^{m-1}\frac{2x}{x+1}}{1+(2^{m-1}-1)\frac{2x}{x+1}}=\frac{2^{m}x}{1+(2^{m}-1)x}.
\]
This proves the claim. Next, since $T^{m}$ is continuously increasing on $[0,1]$, we have
\[
T^{N-1}\left(\left[2^{1-N},1\right]\right)=\left[T^{N-1}\left(2^{1-N}\right),T^{N-1}(1)\right]=\left[\frac{1}{2-2^{1-N}}, 1\right]\subset [1/2, 1].
\]
Therefore, the following $\widetilde{\DUAF}_n$--activated network with width $1$ and depth $N$
\[
\varphi(x):=2^{N-1}R(T^{N-1}(x))+(1-2^{N-1})=\frac{2^{N-1}}{T^{N-1}(x)}+1-2^{N-1}=\frac{1}{x},\quad x\in[2^{1-N},1]
\]
has the desired property.

\end{proof}

\begin{proof}[Proof of Theorem~\ref{thm:sigmoidal-loc}.]
The proof is somewhat technical. We first define two constants
\[
A:=\|\DUAF_n\|_{W^{n, \infty}(\mathbb R)}, \qquad B:= \|\widetilde{\DUAF}_n\|_{W^{n, \infty}(\mathbb R)}.
\]
It suffices to consider the case that $M\ge \max\{B, 6\}$. This is a valid assumption, since an architecture that approximates $\DUAF_n$ well on the larger interval also approximates well on its sub-intervals. Dropping the index $M$ for simplicity, we intend to construct an $\widetilde{\DUAF}_n$--activated network $\Phi_{\eta}$ with desired architectural bounds such that $\|\Phi_{\eta}-\DUAF_n\|_{W^{n-1,\infty}([-M, M])}\rightarrow0$ as $\eta\rightarrow0^{+}$. 
We divide the proof into three steps.

\noindent\textbf{Step 1}: Reproducing $xy$ on $[-\tilde M, \tilde M]^2$ for any $\tilde M>0$.

We first recall the definition of $\widetilde{\DUAF}_n$. Since $d_{n}^{-1}\widetilde{\DUAF}_n(x)=\DUAF_n(x)=x$ on $[-1, -4/11]$, $\widetilde{\DUAF}_n$ can reproduce the identity map on any bounded interval and hence on any arbitrary dimensional compact set. Also notice that
\[
\widetilde{\DUAF}_n(x)=d_{n}\DUAF_n(x)=-2d_{n}-\frac{d_{n}}{x}, \qquad x\in\left[-4, -\frac{13}{11}\right].
\]
Therefore, the following function
\[
\Psi(x):=
-12\,d_{n}^{-1}\widetilde{\DUAF}_n\Bigl(
-2\,d_{n}^{-1}\widetilde{\DUAF}_n(x-3)
-2\,d_{n}^{-1}\widetilde{\DUAF}_n(-x-3)
-8
\Bigr)
-15
\]
coincides with $x^2$ on $[-1,1]$. Let $\Psi_{\tilde M}(x):=\tilde M^2\Psi\big(\tilde M^{-1}x)$, then $\Psi_{\tilde M}$ is an $\widetilde{\DUAF}_n$--activated network with width $2$ and depth $2$ and coincides with $x^2$ on $[-\tilde M, \tilde M]$. Next, we define
\[\Gamma_{\tilde M}(x,y):=2\Psi_{\tilde M}\Big(\frac{x+y}{2}\Big)-2\Psi_{\tilde M}\Big(\frac{x}{2}\Big)-2\Psi_{\tilde M}\Big(\frac{y}{2}\Big).
\]
Clearly, $\Gamma_{\tilde M}(x,y)=xy$ on $[-\tilde M, \tilde M]^2$ and represents an $\widetilde{\DUAF}_n$--activated network with width $6$ and depth $2$.

\noindent\textbf{Step 2}: Sobolev Approximation of $\DUAF_n$ on $[0, M]$ and $[-M, -9/2]$.

Recall that $\widetilde{\DUAF}_n(x)=\int_0^x \frac{c_n\,\DUAF_n(t)}{(t+1)^2}\,dt$ on $[0, \infty)$, where $c_n=(\int_0^\infty \frac{\DUAF_n(t)}{(t+1)^2}\,dt)^{-1}$. For $\delta\in(0, 1)$, we define the approximator $\psi_{\delta}$ as follows:
\[
\psi_{\delta}(x):=\frac{M+1}{c_n}\Gamma_{M+1}\left(\frac{1}{M+1}\Psi_{M+1}(x+1), 
\frac{\widetilde{\DUAF}_n(x+\delta)-\widetilde{\DUAF}_n(x)}{\delta}
\right).
\]
Since $\widetilde{\DUAF}_n$--activated networks can reproduce identity maps, $\psi_{\delta}$ can be regarded as an $\widetilde{\DUAF}_n$--activated network of width $6$ and depth $4$.

For $x\in [-M,M]$, since $x+1\in [-M+1,M+1]\subset [-M-1, M+1]$, we have $\Psi_{M+1}(x+1)=(x+1)^2$ and $\big|\frac{1}{M+1}\Psi_{M+1}(x+1)\big|\le M+1$. Moreover, the mean value theorem implies that $\big|\frac{1}{\delta}(\widetilde{\DUAF}_n(x+\delta)-\widetilde{\DUAF}_n(x))\big|\le \sup_{t\in \mathbb R}|\widetilde{\DUAF}_n^{(1)}(t)|\le B \le M$. The above observation together with the properties of $\Psi_{M+1}$ and $\Gamma_{M+1}$ suggests that we can rewrite $\psi_{\delta}$ as follows:
\[
\begin{aligned}
\psi_{\delta}(x)
&=\frac{M+1}{c_n}\frac{1}{M+1}\Psi_{M+1}(x+1)\,
\frac{\widetilde{\DUAF}_n(x+\delta)-\widetilde{\DUAF}_n(x)}{\delta} \\
&=\frac{1}{\delta}\int_{0}^{\delta} K(u,x)\,\DUAF_n(u+x)\,du,
\qquad x\in[0,M].
\end{aligned}
\]

Here $K(u,x)=\frac{(x+1)^2}{(u+x+1)^2}$. From the definition, it is clear that $\psi_{\delta}\in C^n(\mathbb R)$. Now, we are ready to prove that $\|\psi_\delta-\DUAF_n\|_{W^{n,\infty}([0,M])}\xrightarrow{\,\delta\to 0\,}0$.

For $x\in[0,M]$ and $u\in[0,1]$, we have $x+u+1\ge1$ and a direct computation shows
\[
K(u,x)-1
=\frac{(x+1)^2-(x+u+1)^2}{(x+u+1)^2}
=-\frac{u(2x+u+2)}{(x+u+1)^2}.
\]
Hence,
\begin{equation}\label{K-1}
|K(u,x)-1|\le \frac{u(2x+2u+2)}{(x+u+1)^2}=\frac{2u}{(x+u+1)}\le 2u.
\end{equation}
Moreover,
\[
\partial_x K(u,x)
=\partial_x\!\left(\frac{(x+1)^2}{(x+u+1)^2}\right)
=\frac{2(x+1)u}{(x+u+1)^3} \implies |\partial_x K(u,x)|\le 2(M+1)\,u.
\]
By repeatedly differentiating $(x+1)^2(x+u+1)^{-2}$, one obtains similarly that for
each integer $r\ge1$ there exists $C_r:=C(M,r)>0$ such that
\begin{equation}\label{eq:partial-k}
|\partial_x^{\,r}K(u,x)|\le C_r\,u,
\qquad (u,x)\in[0,1]\times[0,M].
\end{equation}

Fix any $m\in\{0,1,\dots,n\}$. Since $\DUAF_n\in C^n(\mathbb R)$, differentiation under the
integral sign is justified and Leibniz's rule yields
\[
\psi_\delta^{(m)}(x)
=\frac{1}{\delta}\int_0^\delta
\sum_{j=0}^m\binom{m}{j}\,
\partial_x^{\,m-j}K(u,x)\,\DUAF_n^{(j)}(x+u)\,du.
\]
Subtracting $\DUAF_n^{(m)}(x)$ and taking the absolute values, the triangle inequality gives
\begin{equation}\label{eq:derivative-estimate}
\begin{aligned}
|\psi_\delta^{(m)}(x)-\DUAF_n^{(m)}(x)|
&\le \frac{1}{\delta}\int_0^\delta
|\DUAF_n^{(m)}(x+u)-\DUAF_n^{(m)}(x)|\,du
\\
&\quad +\frac{1}{\delta}\int_0^\delta
|K(u,x)-1|\,|\DUAF_n^{(m)}(x+u)|\,du
\\
&\quad +\sum_{j=0}^{m-1}\binom{m}{j}\frac{1}{\delta}\int_0^\delta
|\partial_x^{\,m-j}K(u,x)|\,|\DUAF_n^{(j)}(x+u)|\,du.
\end{aligned}
\end{equation}
Define
\[
\omega(\delta)
:=
\max_{0\le m\le n}
\sup_{\substack{x,y\in[0,M+1]\\ |x-y|\le\delta}}
\bigl|\DUAF_n^{(m)}(x)-\DUAF_n^{(m)}(y)\bigr|,
\]
Notice that $\omega(\delta)$ satisfies $\omega(\delta)\to0$ as $\delta\to0$ because $\DUAF_n\in C^n(\mathbb R)$. Then for all
$x\in[0,M]$ and $u\in[0,\delta]$,
\begin{equation}\label{difference}
|\DUAF_n^{(m)}(x+u)-\DUAF_n^{(m)}(x)|
\le \omega(u)\le \omega(\delta),
\qquad 0\le m\le n.
\end{equation}
Combining \eqref{K-1}, \eqref{eq:partial-k} and \eqref{difference} in \eqref{eq:derivative-estimate}, we obtain a constant $\tilde C(M,n)>0$ such that
\[
\begin{aligned}
\sup_{x\in[0,M]}\bigl|\psi_\delta^{(m)}(x)-\DUAF_n^{(m)}(x)\bigr|
&\le \frac{1}{\delta}\int_0^\delta \omega(u)\,du + \frac{A}{\delta}\int_0^\delta
u\Bigl(2+\sum_{j=0}^{m-1}\binom{m}{j}C_{m-j}\Bigr)\,du \\
&\le \omega(\delta)+\tilde C(M,n)\,\delta,
\qquad 0\le m\le n.
\end{aligned}
\]
Taking the maximum over $m=0,1,\dots,n$, we conclude that
\[
\|\psi_\delta-\DUAF_n\|_{W^{n,\infty}([0,M])}
\le
\omega(\delta)
+
\tilde C(M,n)\delta
\xrightarrow{\,\delta\to 0\,}
0.
\]

Similarly, define 
\[
\psi^{*}_{\delta}(x):=-\frac{M}{d_n}\Gamma_{M}\left(\frac{1}{M}\Psi_{M}(x), 
\frac{\widetilde{\DUAF}_n(x)-\widetilde{\DUAF}_n(x-\delta)}{\delta}
\right)-\frac{9}{5}.
\]
We verify the left-tail estimate explicitly.  For $x\in[-M,-9/2]$ and
$t\in[-\delta,0]$, using the definition of $\widetilde{\DUAF}_n$ on the left
region gives
\[
\widetilde{\DUAF}_n'(x+t)
=-d_n\frac{\DUAF_n(x+t)+9/5}{(x+t)^2}.
\]
Hence
\[
\psi^{*}_{\delta}(x)
=\frac1\delta\int_{-\delta}^{0}
K_-(t,x)\bigl(\DUAF_n(x+t)+9/5\bigr)\,dt-\frac95,
\qquad
K_-(t,x):=\frac{x^2}{(x+t)^2}.
\]
Assume $0<\delta\le1$.  Since $x\le-9/2$ and $t\in[-\delta,0]$, we have
$|x+t|\ge 9/2$ and the elementary derivatives of $K_-$ satisfy
\[
|K_-(t,x)-1|\le C_M|t|,
\qquad
|\partial_x^rK_-(t,x)|\le C_{M,r}|t|,
\quad 1\le r\le n.
\]
Set
\[
\omega_-(\delta):=
\max_{0\le j\le n}\sup_{\substack{y,z\in[-M-1,-9/2]\\ |y-z|\le\delta}}
|\DUAF_n^{(j)}(y)-\DUAF_n^{(j)}(z)|.
\]
Then $\omega_-(\delta)\to0$ as $\delta\to0$ because $\DUAF_n\in C^n$ on the
compact interval $[-M-1,-9/2]$.  Differentiating the preceding integral formula
under the integral sign and using the product rule, the terms with
$K_-(t,x)-1$ or $\partial_x^rK_-(t,x)$ are bounded by $C(M,n)\delta$, while the
translation error $\DUAF_n^{(j)}(x+t)-\DUAF_n^{(j)}(x)$ is bounded by
$\omega_-(\delta)$.  Therefore, for each $0\le m\le n$,
\[
\|\partial_x^m(\psi^{*}_{\delta}-\DUAF_n)\|_{L^\infty([-M,-9/2])}
\le \omega_-(\delta)+C(M,n)\delta,
\]
and consequently
\[
\|\psi^{*}_\delta-\DUAF_n\|_{W^{n,\infty}([-M,-9/2])}
\xrightarrow{\,\delta\to 0\,}0.
\]

\noindent\textbf{Step 3}: The global approximator $\Phi_{\eta}$ of $\DUAF_n$ on $[-M, M]$.

On $[-M, M]$, it can be seen from the discussion in Step 2 that 
$\psi_\delta$ and $\psi^{*}_{\delta}$ always have the following representations:
\[
\psi_\delta(x)=c_n^{-1}(x+1)^2\frac{\widetilde{\DUAF}_n(x+\delta)-\widetilde{\DUAF}_n(x)}{\delta}=\frac{(x+1)^2}{c_n\delta}\int_0^{\delta}\widetilde{\DUAF}_n^{(1)}(t+x)\,dt,
\]
\[
\psi^{*}_{\delta}(x)
=
-d_n^{-1}x^2\frac{\widetilde{\DUAF}_n(x)-\widetilde{\DUAF}_n(x-\delta)}{\delta}-\frac{9}{5}
=
-\frac{x^2}{d_n\delta}\int_{-\delta}^{0}\widetilde{\DUAF}_n^{(1)}(t+x)\,dt
-\frac{9}{5}.
\]
We also need uniform bounds for these finite-difference networks.  For
$0\le r\le n-1$, differentiating the averaged-derivative representation gives
\[
\partial_x^r\psi_\delta(x)
=\frac1{c_n}\sum_{\ell=0}^{\min\{2,r\}}\binom r\ell
\partial_x^\ell (x+1)^2\,\frac1\delta\int_0^\delta
\widetilde{\DUAF}_n^{(r-\ell+1)}(x+t)\,dt.
\]
The same computation for $\psi^*_{\delta}$ gives a finite sum with the factor
$x^2$ and the average over $[-\delta,0]$.  Thus no negative power of $\delta$
appears after differentiation: for $0<\delta\le1$ there exists
$\tilde C_1(M,n)>0$, independent of $\delta$, such that
\[
\max\left\{\|\psi_\delta\|_{W^{n-1,\infty}([-M,M])},
\|\psi^{*}_{\delta}\|_{W^{n-1,\infty}([-M,M])}\right\}
\le
\tilde C_1(M,n)\,\|\widetilde{\DUAF}_n\|_{W^{n,\infty}([-M-1,M+1])}.
\]
We therefore define 
\begin{equation}\label{K}
K:= 1+A+B+ \tilde C_1(M, n)B. 
\end{equation}
Notice that $K$ serves as a uniform Sobolev bound for $\DUAF_n$, $\widetilde{\DUAF}_n$, $\psi_\delta$, and $\psi^{*}_{\delta}$ on $[-M, M]$ up to order $n-1$.
Now, let us begin to define the global Sobolev approximator $\Phi_{\eta}$, with the
scaling parameter $\eta>0$. Before providing the exact formula for $\Phi_{\eta}$, we need to set up a few auxiliary functions, which will be the building blocks of $\Phi_{\eta}$ and can be realized as $\widetilde{\DUAF}_n$--activated networks. Define
\[
\alpha:=1+\frac{2}{n},\qquad 
t_\eta(x):=\eta^{-\alpha}\Big(\frac{\eta}{2}-x\Big),
\qquad
u(t):=\frac{1-\widetilde{\DUAF}_n(t)}{2}\in[0,1].
\]
and let $\eta\le 10^{-\frac{1}{\alpha-1}}=10^{-\frac{n}{2}}$ through the rest of the proof.  We introduce a two--sided smooth cutoff (with exponent fixed as $n^2$)
\[
G(x):=\frac{x^{n^2}}{x^{n^2}+(1-x)^{n^2}},\qquad
g_{\eta}(x)
:=G\circ u\circ t_\eta(x).
\]
Set $m:=n^2$. We now record the exact size of the cutoff subnet. Recall from Step~1
that $\widetilde{\DUAF}_n$ can reproduce identity maps and the multiplier
$\Gamma_{\tilde M}$ on any prescribed bounded cube. Repeating the monomial
construction in Lemma~\ref{lemma:poly} with this multiplier gives a block of width
$m+3$ and depth $2m-2$ that equals $x^m$ on $[0,1]$. Running two such blocks in
parallel gives
\[
  a(x):=x^m,\qquad b(x):=(1-x)^m,
\]
with width $2(m+3)$ and depth $2m-2$. We choose the affine output of this block to
produce the two channels $(a,a+b)$. Since $a+b\in[2^{1-m},1]$ for $x\in[0,1]$,
Lemma~\ref{lem:reciprocal-dyadic} with $N=m$ gives $(a+b)^{-1}$ with width $1$ and
depth $m$; during these $m$ layers the channel $a$ is carried by an identity channel.
Thus the pair $(a,(a+b)^{-1})$ is realized with width $2(m+3)$ and depth
$(2m-2)+m=3m-2$. Composing this pair with the multiplier
$\Gamma_{2^{m-1}}$ adds two layers and does not increase the width because
$2(m+3)\ge6$. Hence
\[
  G(x)=\frac{x^m}{x^m+(1-x)^m},\qquad x\in[0,1],
\]
is realized with width $2m+6$ and depth $3m$. Finally, the map
$u\circ t_\eta$ is a width-one, depth-one $\widetilde{\DUAF}_n$ block and has range
in $[0,1]$. Therefore
$g_\eta=G\circ u\circ t_\eta$ is realized with the sharp size used below:
\[
  W_g=2n^2+6,
  \qquad
  L_g=3n^2+1.
\]

Using $g_{\eta}$, we define a partition of unity as follows:
\begin{equation}\label{chi}
\begin{alignedat}{2}
\chi_{1,\eta}(x)
&:= \Gamma_1\!\left(1-g_\eta\!\left(x+\frac92+\eta\right),\,1-g_\eta(x)\right)
&&= \left(1-g_\eta\!\left(x+\frac92+\eta\right)\right)\left(1-g_\eta(x)\right),\\[0.5em]
\chi_{2,\eta}(x)
&:= \Gamma_1\!\left(g_\eta\!\left(x+\frac92+\eta\right),\,1-g_\eta(x)\right)
&&= g_\eta\!\left(x+\frac92+\eta\right)\left(1-g_\eta(x)\right),\\[0.5em]
\chi_{3,\eta}(x)
&:= g_\eta(x),
\end{alignedat}
\end{equation}

 A direct computation shows that $\chi_{1,\eta}+\chi_{2,\eta}+\chi_{3,\eta}=1$ on $\mathbb R$. An illustration of the functions $\chi_{1,\eta},\chi_{2,\eta},\chi_{3,\eta}$ is provided in Figure~\ref{fig:chi}. 

\begin{figure}[h]
  \centering
  \includegraphics[width=\textwidth]{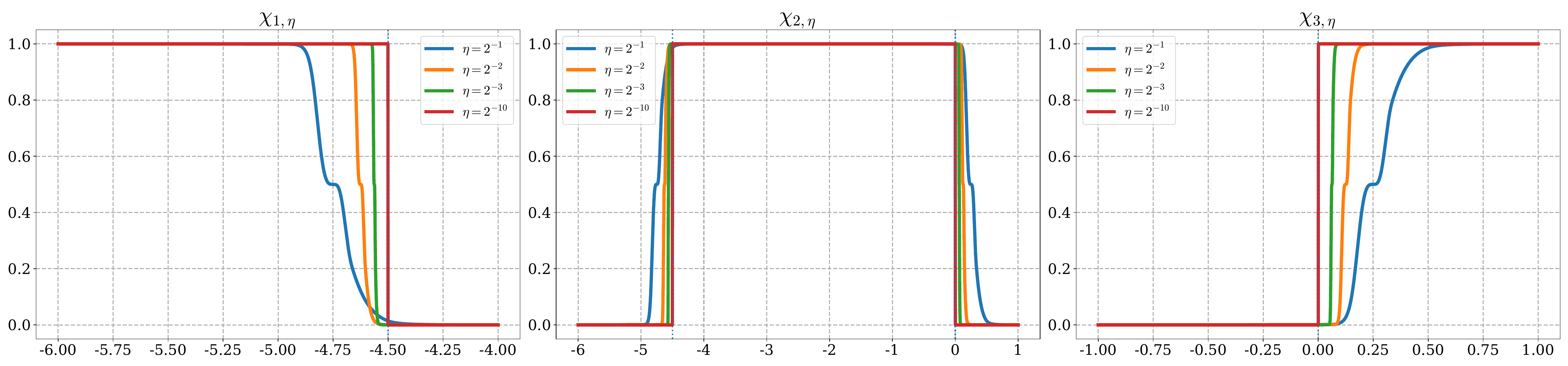}
  \caption{Behavior of the partition functions $\chi_{1,\eta},\chi_{2,\eta},\chi_{3,\eta}$ for $n=2$}
  \label{fig:chi}
\end{figure}

Intuitively speaking, these three functions decompose the real line into left, middle, and right regions, up to a smooth transition of width $\mathcal{O}(\eta)$. 
More precisely, $\chi_{1,\eta}$ is approximately supported on the left region $(-\infty,-9/2]$, $\chi_{2,\eta}$ localizes the middle region $[-9/2,0]$, and $\chi_{3,\eta}$ captures the right region $[0,\infty)$. 
The transitions between these regions occur only in narrow layers of size $\mathcal{O}(\eta)$ near $x=-9/2$ and $x=0$, where the functions vary smoothly between $0$ and $1$. 
In particular, away from these transition layers, each $\chi_{i,\eta}$ is approximately constant (either $0$ or $1$), so that the above construction behaves as a smooth partition of unity adapted to the three regions.

The polynomial order in the definition of $G$ is chosen to be $n^2$. 
This choice is mainly technical: it is sufficiently large to guarantee the 
Sobolev estimates needed in the subsequent approximation argument, while still 
keeping the cutoff construction relatively simple in the sense that they can be realized as $\widetilde{\DUAF}_n$-activated networks.

Since $\|\psi_\delta-\DUAF_n\|_{W^{n,\infty}([0,M])}\rightarrow{}0$ and $\|\psi^{*}_\delta-\DUAF_n\|_{W^{n,\infty}([-M,-9/2])}\rightarrow{}0$ as $\delta\rightarrow0$ (this is the conclusion of Step 2), for each $\eta>0$, we can choose a $\delta=\delta(\eta)>0$ such that
\[
\|\psi_\delta-\DUAF_n\|_{W^{n,\infty}([0,M])}\le \frac{\eta}{\|\chi_{3,\eta}\|_{W^{n-1,\infty}([0,M])}},
\]
\[
\|\psi^{*}_\delta-\DUAF_n\|_{W^{n,\infty}([-M,-9/2])}\le \frac{\eta}{\|\chi_{1,\eta}\|_{W^{n-1,\infty}([-M,-9/2])}}.
\]
With this choice of $\delta$, we then define
\begin{equation}\label{eq:sigmoidal-global-approximator}
\Phi_{\eta}
:=\Gamma_K(\psi^{*}_\delta,\chi_{1,\eta})+d_{n}^{-1}\,\Gamma_K(\widetilde{\DUAF}_n,\chi_{2,\eta})+\Gamma_K(\psi_{\delta},\chi_{3,\eta}).
\end{equation}
The architecture count below is layer-by-layer for this realization, with no extra
coarse padding. The two shifted copies of $g_\eta$ are computed in parallel with
width $2W_g=4n^2+12$ and depth $L_g=3n^2+1$. The two multipliers in
\eqref{chi} then form $\chi_{1,\eta}$ and $\chi_{2,\eta}$ in two additional layers;
while these two layers run, the channel $\chi_{3,\eta}=g_\eta$ is carried by an
identity channel. Thus the cutoff block outputs
$(\chi_{1,\eta},\chi_{2,\eta},\chi_{3,\eta})$ with width $4n^2+12$ and depth
$3n^2+3$.

The value block producing $(\psi^*_{\delta},d_n^{-1}\widetilde{\DUAF}_n,
\psi_{\delta})$ has depth $4$ and width $13$: its first two layers compute the two
square factors $\Psi_M(x)$, $\Psi_{M+1}(x+1)$ and the three activation values
$\widetilde{\DUAF}_n(x-\delta)$, $\widetilde{\DUAF}_n(x)$,
$\widetilde{\DUAF}_n(x+\delta)$, using width $2+2+3=7$; its last two layers run the
two product maps defining $\psi^*_{\delta}$ and $\psi_\delta$ in parallel and carry
$d_n^{-1}\widetilde{\DUAF}_n(x)$, using width $2\cdot6+1=13$.

To keep the depth minimal, we start this value block during the last four layers of
the cutoff block and carry the input $x$ by one identity channel before that point.
Consequently the only possible layer widths before the final products are
\[
4n^2+13,\qquad (4n^2+12)+7=4n^2+19,
\qquad 13+13=26.
\]
The three final products in \eqref{eq:sigmoidal-global-approximator} are then computed
by three parallel copies of $\Gamma_K$, using width $18$ and two further layers; the
last summation is affine. Hence
\[
  W_{\rm sim}(n)=\max\{4n^2+19,26\},
  \qquad
  L_{\rm sim}(n)=(3n^2+3)+2=3n^2+5.
\]
Therefore $\Phi_{\eta}$ is a $\widetilde{\DUAF}_n$--NN of width
$\max\{4n^2+19,26\}$ and depth $3n^2+5$. Moreover, since the
supremum norms of $\psi_\delta$, $\psi^{*}_{\delta}$, $\widetilde{\DUAF}_n$, $\chi_{1,\eta}$, $\chi_{2,\eta}$, and $\chi_{3,\eta}$ on
$[-M,M]$ are uniformly bounded by $K$, the defining property of $\Gamma_K$ yields the following point-wise
representation on $[-M,M]$:
\[
\Phi_{\eta}(x)
=
\psi^{*}_\delta(x)\,\chi_{1,\eta}(x)
+
d_{n}^{-1}\widetilde{\DUAF}_n(x)\,\chi_{2,\eta}(x)
+ 
\psi_\delta(x)\,\chi_{3,\eta}(x),
\qquad x\in[-M,M].
\]
Using the above point-wise representation, we can compare the overall differences of $\Phi_{\eta}$ and $\DUAF_n$ on $[-M, M]$ by
\begin{equation}\label{Phi-DUAF}
\Phi_{\eta}-\DUAF_n
=
(\psi^{*}_\delta-\DUAF_n)\,\chi_{1,\eta}
+
(d_{n}^{-1}\widetilde{\DUAF}_n-\DUAF_n)\,\chi_{2,\eta}
+ 
(\psi_\delta-\DUAF_n)\,\chi_{3,\eta}.
\end{equation}
Our goal is to show that $\|\Phi_{\eta}-\DUAF_n\|_{W^{n-1, \infty}([-M, M])}\xrightarrow{\,\eta\to 0^{+}\,}0$. This can be found in Step $4$.

\medskip
\noindent
\smallskip
\noindent\textbf{Step 4}: $\|\Phi_{\eta}-\DUAF_n\|_{W^{n-1,\infty}([-M,M])}\rightarrow0$.

To establish the desired convergence, it is necessary to analyze the behavior of the partition functions 
$\chi_{1,\eta},\chi_{2,\eta},\chi_{3,\eta}$ as the scaling parameter $\eta\to0^+$. 
These functions form a smooth decomposition of the domain, and their derivatives exhibit 
nontrivial scaling effects near the transition regions. The following lemma provides precise asymptotic estimates for these partition functions in 
Sobolev norms, which will be crucial in controlling the global approximation error.

\begin{lemma}\label{lem:chi-est} 
Let $K$ be the constant in \eqref{K}, and let $\chi_{1,\eta},\chi_{2,\eta},\chi_{3,\eta}$ be defined in \eqref{chi}. Then, for $0\le m\le n-1$, the following estimates hold in big-$\mathcal{O}$ notation (with constants depending on $M,n$ but independent of $\eta$): \[ \|\chi_{1,\eta}\|_{W^{n-1,\infty}([-\frac92,M])} = \mathcal{O}\!\left(\eta^{\,n-1+\frac2n}\right),\qquad \|\chi_{3,\eta}\|_{W^{n-1,\infty}([-M,0])} = \mathcal{O}\!\left(\eta^{\,n-1+\frac2n}\right). \] Moreover, for $\chi_{2,\eta}$ we have the regional estimates \[ \|\chi_{2,\eta}\|_{W^{n-1,\infty}([-M,-\frac92-\eta]\cup[\eta,M])} = \mathcal{O}\!\left(\eta^{\,n-1+\frac2n}\right), \] and on the transition regions \[ \|\chi_{2,\eta}^{(m)}\|_{L^\infty([-\frac92-\eta,-\frac92]\cup[0,\eta])} = \mathcal{O}\!\left(\eta^{-\alpha m}\right),\qquad 0\le m\le n-1. \] 
\end{lemma}

We will use the estimates from Lemma~\ref{lem:chi-est} repeated throughout the rest of the proof. The proof of this lemma is postponed to Section~\ref{sec:proof-chi-est}.
To simplify the presentation of the five intervals, we introduce the following notation for the transition region and its complements.
\[
T_{\eta}:=\Big[-\frac92-\eta,-\frac92\Big]\cup[0,\eta], \quad
E_\eta:=[-M,M]\setminus T_{\eta}
=
[-M,-\tfrac92-\eta]\cup[-\tfrac92,0]\cup[\eta,M].
\]
We intend to estimate $\|\Phi_{\eta}-\DUAF_n\|_{W^{n-1,\infty}([-M,M])}$. In order to do so, we split the estimate into two parts:
\[
\|\Phi_{\eta}-\DUAF_n\|_{W^{n-1,\infty}([-M,M])}\le
\|\Phi_{\eta}-\DUAF_n\|_{W^{n-1,\infty}(E_{\eta})}
+
\|\Phi_{\eta}-\DUAF_n\|_{W^{n-1,\infty}(T_{\eta})}.
\]
\medskip
\noindent\textbf{For $\|\Phi_{\eta}-\DUAF_n\|_{W^{n-1,\infty}(E_{\eta})}$.}

Applying the triangle inequality to \eqref{Phi-DUAF}, we obtain
\[
\begin{aligned}
\|\Phi_\eta-\DUAF_n\|_{W^{n-1,\infty}(E_\eta)}
\le{}&
\|(\psi_\delta^*-\DUAF_n)\chi_{1,\eta}\|_{W^{n-1,\infty}(E_\eta)}
\\
&+
\|(d_n^{-1}\widetilde{\DUAF}_n-\DUAF_n)\chi_{2,\eta}\|_{W^{n-1,\infty}(E_\eta)}
\\
&+
\|(\psi_\delta-\DUAF_n)\chi_{3,\eta}\|_{W^{n-1,\infty}(E_\eta)}.
\end{aligned}
\]
We estimate the three terms separately. First, on \([-M,-\frac92-\eta]\), by the choice of
\(\delta=\delta(\eta)\),
\[
\|(\psi_\delta^*-\DUAF_n)\chi_{1,\eta}\|_{W^{n-1,\infty}([-M,-\frac92-\eta])}
=
\mathcal{O}(\eta).
\]
On \([-\frac92,0]\cup[\eta,M]\), Lemma~\ref{lem:chi-est} gives $\|\chi_{1,\eta}\|_{W^{n-1,\infty}([-\frac92,M])}
=\mathcal{O}(\eta^{\,n-1+\frac2n})$.
Moreover,
\[
\|\psi_\delta^*-\DUAF_n\|_{W^{n-1,\infty}([-M,M])}\le 2K.
\]
Hence, by the product estimate,
\[
\begin{aligned}
\|(\psi_\delta^*-\DUAF_n)\chi_{1,\eta}\|_{W^{n-1,\infty}([-\frac92,0]\cup[\eta,M])}
&\le
C K
\|\chi_{1,\eta}\|_{W^{n-1,\infty}([-\frac92,M])}
=
\mathcal{O}\!\left(\eta^{\,n-1+\frac2n}\right).
\end{aligned}
\]
Therefore,
\[
\|(\psi_\delta^*-\DUAF_n)\chi_{1,\eta}\|_{W^{n-1,\infty}(E_\eta)}
=
\mathcal{O}(\eta)+\mathcal{O}\!\left(\eta^{\,n-1+\frac2n}\right)=\mathcal{O}(\eta).
\]

For the third term, on \([\eta,M]\), by the choice of \(\delta=\delta(\eta)\),
\[
\|(\psi_\delta-\DUAF_n)\chi_{3,\eta}\|_{W^{n-1,\infty}([\eta,M])}
=
\mathcal{O}(\eta).
\]
On \([-M,-\frac92-\eta]\cup[-\frac92,0]\), Lemma~\ref{lem:chi-est} gives $\|\chi_{3,\eta}\|_{W^{n-1,\infty}([-M,0])}
=
\mathcal{O}(\eta^{\,n-1+\frac2n})$. Since
\[
\|\psi_\delta-\DUAF_n\|_{W^{n-1,\infty}([-M,M])}\le 2K,
\]
the product estimate gives
\[
\begin{aligned}
\|(\psi_\delta-\DUAF_n)\chi_{3,\eta}\|_{W^{n-1,\infty}([-M,-\frac92-\eta]\cup[-\frac92,0])}
&\le
C K
\|\chi_{3,\eta}\|_{W^{n-1,\infty}([-M,0])}
=
\mathcal{O}\!\left(\eta^{\,n-1+\frac2n}\right).
\end{aligned}
\]
Thus
\[
\|(\psi_\delta-\DUAF_n)\chi_{3,\eta}\|_{W^{n-1,\infty}(E_\eta)}
=
\mathcal{O}(\eta)+\mathcal{O}\!\left(\eta^{\,n-1+\frac2n}\right)=\mathcal{O}(\eta).
\]

It remains to estimate the middle term. On \([-\frac92,0]\), by the definition of
\(\widetilde{\DUAF}_n\),
\[
\widetilde{\DUAF}_n=d_n\DUAF_n,
\qquad\text{hence}\qquad
d_n^{-1}\widetilde{\DUAF}_n-\DUAF_n=0.
\]
Therefore the middle term vanishes identically on \([-\frac92,0]\). On the two outer regions
\([-M,-\frac92-\eta]\cup[\eta,M]\), Lemma~\ref{lem:chi-est} yields
\[
\|\chi_{2,\eta}\|_{W^{n-1,\infty}([-M,-\frac92-\eta]\cup[\eta,M])}
=
\mathcal{O}\!\left(\eta^{\,n-1+\frac2n}\right).
\]
Also, since \(d_n\) is fixed and \(\widetilde{\DUAF}_n,\DUAF_n\) have fixed Sobolev bounds on
\([-M,M]\), namely
\[
\|d_n^{-1}\widetilde{\DUAF}_n-\DUAF_n\|_{W^{n,\infty}([-M,M])}
\le (d_n^{-1}+1)K.
\]
Hence the product estimate gives
\[
\begin{aligned}
\|(d_n^{-1}\widetilde{\DUAF}_n-\DUAF_n)\chi_{2,\eta}\|_{W^{n-1,\infty}(E_\eta)}
&\le
(d_n^{-1}+1)K
\|\chi_{2,\eta}\|_{W^{n-1,\infty}([-M,-\frac92-\eta]\cup[\eta,M])}
\\
&=
\mathcal{O}\!\left(\eta^{\,n-1+\frac2n}\right).
\end{aligned}
\]

Combining the three estimates, we obtain
\[
\|\Phi_\eta-\DUAF_n\|_{W^{n-1,\infty}(E_\eta)}
=
\mathcal{O}(\eta)+\mathcal{O}\!\left(\eta^{\,n-1+\frac2n}\right)=\mathcal{O}(\eta).
\]
\noindent\textbf{For $\|\Phi_{\eta}-\DUAF_n\|_{W^{n-1,\infty}(T_{\eta})}$.}

\noindent\emph{(A) Derivatives bounds for  $d_n^{-1}\widetilde{\DUAF}_n-\DUAF_n$ on $T_\eta$.}

Recall that on $[0,1]$ we have $\DUAF_n(x)=q_n(2x)$. Denote $I_n=\int_0^{1} t^{\,n}(1-t)^{n}\,dt$ and fix $r\in\{0,1,\dots,n\}$, $u\in[0,1]$. For $r=0$, using $(1-t)^n\le 1$,
\[
0\le q_n(u)=\frac{1}{I_n}\int_0^{u} t^{\,n}(1-t)^{n}\,dt\le \frac{1}{I_n}\int_0^{u} t^{\,n}\,dt=\frac{u^{n+1}}{(n+1)I_n}.
\]
For $r=1$, differentiating under the integral sign once gives $q_n'(u)=\frac{u^{n}(1-u)^{n}}{I_n}$. 
For $1\le r\le n$, repeatedly differentiating the polynomial
$u^n(1-u)^n$, we obtain a constant $E_1=E_1(n)>0$ which depends only on $n$ such that
\[
\begin{aligned}
|q_n^{(r)}(u)|
&\le
\frac{1}{I_n}\sum_{j=0}^{r-1}\binom{r-1}{j}
\frac{(n!)^2}{(n-j)!(n+j+1-r)!}\,
u^{\,n-j}(1-u)^{\,n+j+1-r} \\
&\le
E_1\,u^{\,n-r+1},\qquad u\in[0,1],
\end{aligned}
\]
where the second inequality is due to $(1-u)^{n+j+1-r}\le1$ in $[0,1]$.
Combining the cases $r=0$ and $1\le r\le n$, and making $E_1$ larger if necessary, we have the bounds $|q_n^{(r)}(u)|\le E_1\,u^{\,n+1-r}$.
Therefore, for $x\in[0,\eta]\subset[0,\frac{1}{2}]$ and $0\le r\le n$,
\begin{equation}\label{eq:sigma-small-from-def}
|\DUAF_n^{(r)}(x)|
=
2^{r}\,|q_n^{(r)}(2x)|
\le
2^{r}E_1(2x)^{n-r+1}
\le
2^{n+1}E_1\,\eta^{\,n+1-r}.
\end{equation}

Next, on $[0,\infty)$ we have $\widetilde{\DUAF}_n'(x)=\frac{c_n\,\DUAF_n(x)}{(x+1)^2}$. Fix $r\in\{1,\dots,n\}$ and differentiate $r-1$ times:
\[
\widetilde{\DUAF}_n^{(r)}(x)
=
\partial_x^{\,r-1}\!\Big(\frac{c_n\,\DUAF_n(x)}{(x+1)^2}\Big)
=
c_n\sum_{j=0}^{r-1}\binom{r-1}{j}\,
\DUAF_n^{(j)}(x)\,\partial_x^{\,r-1-j}(x+1)^{-2}.
\]
Since $\sup_{x\in[0,1]}|\partial_x^{\,p}(x+1)^{-2}|$ is bounded by constants depending only on the order $p$, using
\eqref{eq:sigma-small-from-def} yields a constant $E_2=E_2(n)>0$ such that for $x\in[0,\eta]$ and $1\le r\le n$,
\[
|\widetilde{\DUAF}_n^{(r)}(x)|
\le
E_2\,\eta^{\,n+2-r},
\qquad
|\widetilde{\DUAF}_n(x)|\le \sup_{t\in[0,\eta]}|\widetilde{\DUAF}_n'(t)||x|\le E_2 \,\eta^{n+2},
\]
where the second inequality comes from using $\widetilde{\DUAF}_n(0)=0$ and applying the mean value theorem. Consequently, for each $m\in\{0,1,\dots,n-1\}$,
\[
\bigl\|(d_{n}^{-1}\widetilde{\DUAF}_n-\DUAF_n)^{(m)}\bigr\|_{L^\infty([0,\eta])}
\le
d_{n}^{-1}\|\widetilde{\DUAF}_n^{(m)}\|_{L^\infty([0,\eta])}+\|\DUAF_n^{(m)}\|_{L^\infty([0,\eta])}
=
\mathcal{O}(\eta^{n+1-m}), 
\]
We also verify the left transition.  Put $a:=-9/2$ and
$F(x):=\DUAF_n(x)+9/5$ for $x\le a$.  By the endpoint flatness of the polynomial
piece defining $\DUAF_n$, $F^{(j)}(a)=0$ for $0\le j\le n$; hence, for
$x\in[a-\eta,a]$ and $0\le j\le n$,
\[
|F^{(j)}(x)|\le C(n)\eta^{n+1-j}.
\]
On the left of $a$, the defining relation gives
\[
\widetilde{\DUAF}_n'(x)=-d_n\frac{F(x)}{x^2}.
\]
Differentiating this identity $r-1$ times and using that $x$ stays in a compact
interval away from $0$, we obtain, for $1\le r\le n$,
\[
|\widetilde{\DUAF}_n^{(r)}(x)|
\le C(n)\eta^{n+2-r},
\qquad x\in[a-\eta,a].
\]
Together with $\widetilde{\DUAF}_n(a)=d_n\DUAF_n(a)$, the mean value theorem gives
the same order for $r=0$.  Combining these bounds with the preceding bounds for
$\DUAF_n^{(m)}$ yields, for every $0\le m\le n-1$,
\[
\bigl\|(d_{n}^{-1}\widetilde{\DUAF}_n-\DUAF_n)^{(m)}\bigr\|_{L^\infty([a-\eta,a])}
=\mathcal{O}(\eta^{n+1-m}).
\]
Therefore, we derive an overall bound
\begin{equation}\label{eq:dD-D-T}
\|(d_{n}^{-1}\widetilde{\DUAF}_n-\DUAF_n)^{(m)}\|_{L^\infty(T_\eta)}=\mathcal{O}(\eta^{n+1-m}).    
\end{equation}

\smallskip
\noindent
\emph{(B) Estimating the contribution.}

On the transition regions  \(T_{\eta}=\left[-\frac92-\eta,-\frac92\right]\cup[0,\eta]\), we still use
\[
\Phi_\eta-\DUAF_n
=
(\psi_\delta^*-\DUAF_n)\chi_{1,\eta}
+
(d_n^{-1}\widetilde{\DUAF}_n-\DUAF_n)\chi_{2,\eta}
+
(\psi_\delta-\DUAF_n)\chi_{3,\eta}.
\]
By the choice of \(\delta=\delta(\eta)\),
\[
\|(\psi_\delta^*-\DUAF_n)\chi_{1,\eta}\|_{W^{n-1,\infty}([-\frac92-\eta,-\frac92])}
=
\mathcal{O}(\eta),
\qquad
\|(\psi_\delta-\DUAF_n)\chi_{3,\eta}\|_{W^{n-1,\infty}([0,\eta])}
=
\mathcal{O}(\eta).
\]
On the remaining transition pieces, we use Lemma~\ref{lem:chi-est}. Since
\(\|\psi_\delta^*-\DUAF_n\|_{W^{n,\infty}}\le 2K\) and
\(\|\psi_\delta-\DUAF_n\|_{W^{n,\infty}}\le 2K\), the product estimate gives
\[
\|(\psi_\delta^*-\DUAF_n)\chi_{1,\eta}\|_{W^{n-1,\infty}([0,\eta])}
=
\mathcal{O}\!\left(\eta^{\,n-1+\frac2n}\right),
\]
and
\[
\|(\psi_\delta-\DUAF_n)\chi_{3,\eta}\|_{W^{n-1,\infty}([-\frac92-\eta,-\frac92])}
=
\mathcal{O}\!\left(\eta^{\,n-1+\frac2n}\right).
\]
Hence
\[
\begin{aligned}
&\|(\psi_\delta^*-\DUAF_n)\chi_{1,\eta}\|_{W^{n-1,\infty}(T_\eta)}
+
\|(\psi_\delta-\DUAF_n)\chi_{3,\eta}\|_{W^{n-1,\infty}(T_\eta)}
=
\mathcal{O}(\eta)+\mathcal{O}\!\left(\eta^{\,n-1+\frac2n}\right).
\end{aligned}
\]

It remains to estimate the middle term. By \eqref{eq:dD-D-T} and Lemma~\ref{lem:chi-est}, on the transition regions,
\[
\|(d_n^{-1}\widetilde{\DUAF}_n-\DUAF_n)^{(i)}\|_{L^\infty(T_\eta)}
=
\mathcal{O}(\eta^{n+1-i}),
\quad
\|\chi_{2,\eta}^{(i)}\|_{L^\infty(T_\eta)}
=
\mathcal{O}(\eta^{-\alpha i}),
\quad 0\le i\le n-1.
\]
Therefore, by the Leibniz rule, for \(0\le m\le n-1\),
\[
\begin{aligned}
&\left\|\partial_x^m\left[
(d_n^{-1}\widetilde{\DUAF}_n-\DUAF_n)\chi_{2,\eta}
\right]\right\|_{L^\infty(T_\eta)}
\\
&\qquad\le
\sum_{i=0}^m
\binom{m}{i}
\mathcal{O}(\eta^{n+1-i})
\mathcal{O}(\eta^{-\alpha(m-i)})
\\
&\qquad=
\sum_{i=0}^m
\mathcal{O}\!\left(\eta^{n+1-\alpha m+(\alpha-1)i}\right)
=
\mathcal{O}\!\left(\eta^{n+1-\alpha m}\right),
\end{aligned}
\]
where we used \(\alpha>1\). The order minimizes at \(m=n-1\), we get
\[
\|(d_n^{-1}\widetilde{\DUAF}_n-\DUAF_n)\chi_{2,\eta}\|_{W^{n-1,\infty}(T_\eta)}
=
\mathcal{O}\!\left(\eta^{\,n+1-\alpha(n-1)}\right)
=
\mathcal{O}\!\left(\eta^{2/n}\right).
\]

Combining the three terms yields
\[
\|\Phi_\eta-\DUAF_n\|_{W^{n-1,\infty}(T_\eta)}
=
\mathcal{O}(\eta)
+
\mathcal{O}\!\left(\eta^{\,n-1+\frac2n}\right)
+
\mathcal{O}\!\left(\eta^{2/n}\right)
=\mathcal{O}(\eta)+\mathcal{O}\!\left(\eta^{2/n}\right).
\]
This completes the proof.
\end{proof}

\subsection{Proof of Lemma~\ref{lem:chi-est}}\label{sec:proof-chi-est}
Let us recall the definition of $g_{\eta}$ 
\[
G(x):=\frac{x^{n^2}}{x^{n^2}+(1-x)^{n^2}},\qquad
g_{\eta}(x)
:=G\circ u\circ t_\eta(x),
\]
where $t_\eta(x):=\eta^{-\alpha}(\frac{\eta}{2}-x)$ and $u(t):=\frac{1}{2}(1-\widetilde{\DUAF}_n(t))$. We will use the following elementary facts about $G$ repeatedly throughout the proof:
  For each integer $m$ with $0\le m\le n-1$, there exists a constant $C_{n}^{*}>0$
  such that
 \begin{equation}\label{eq:G-global-u}
 |G^{(m)}(x)| \;\le\; C_{n}^{*}\,x^{\,n^2-m},\qquad |(1-G)^{(m)}(x)| \;\le\; C_{n}^{*}\,(1-x)^{\,n^2-m},
 \qquad \,x\in[0,1].
 \end{equation}
 Instead of directly estimating $\chi_{1, \eta}, \chi_{2, \eta},\chi_{3, \eta}$, it is more convenient to estimate their building block $g_{\eta}$.
 
\noindent (A) \textbf{Estimating $g_{\eta}$ on $[-2M, 0]$}

 For $t>0$, by the definition of $\widetilde{\DUAF}_n$,
\[
1-\widetilde{\DUAF}_n(t)=\int_{t}^{\infty}\frac{c_n\,\DUAF_n(s)}{(s+1)^2}\,ds
\le
\int_{t}^{\infty}\frac{c_n}{(s+1)^2}\,ds
=
\frac{c_n}{t+1}.
\]
Notice that when $x\in[-2M,0]$, we always have $t_\eta(x)\ge \frac{\eta^{1-\alpha}}{2}>0$. Therefore, 
\begin{equation}\label{eq:left-u0}
u(t_\eta(x))=\frac{1-\widetilde{\DUAF}_n(t_\eta(x))}{2}
\le
\frac{c_n}{2(t_\eta(x)+1)}
\le
\frac{c_n}{2}\,\eta^{\alpha-1}.
\end{equation}
Moreover, for $t>0$ we have $\widetilde{\DUAF}_n'(t)=c_n\,\DUAF_n(t)(t+1)^{-2}$.
Differentiating this identity repeatedly and using the global bounds
$\|\DUAF_n^{(k)}\|_{L^\infty(\R)}\le K$ for $0\le k\le n$, we obtain constants $D_n>0$ such that for every
$1\le j\le n-1$, $\bigl|\widetilde{\DUAF}_n^{(j)}(t)\bigr|
\le
D_nK(t+1)^{-2}$.

Consequently, by the chain rule, for $1\le j\le n-1$ and $x\in[-2M,0]$,
\begin{equation}\label{eq:left-uj}
\bigl|(u\circ t_\eta)^{(j)}(x)\bigr|
=
\frac{\eta^{-\alpha j}}{2}\,\bigl|\widetilde{\DUAF}_n^{(j)}(t_\eta(x))\bigr|
\le
\frac{\eta^{-\alpha j}}{2}\cdot \frac{D_n K}{(t_\eta(x)+1)^{2}}
\le
\frac{D_nK}{2}\,\eta^{2(\alpha-1)-\alpha j}.
\end{equation}

Recall that $g_\eta(x)=G\big(u(t_\eta(x))\big)$. 
Fix $m\in\{0,1,\dots,n-1\}$. By repeated application of the chain rule and product rule,
$g_\eta^{(m)}(x)$ is a finite sum of terms of the form
\[
G^{(k)}\big(u(t_\eta(x))\big)
\prod_{\ell=1}^{m}
\bigl((u\circ t_\eta)^{(\ell)}(x)\bigr)^{r_\ell}, \qquad 0\le k \le m,
\]
where the nonnegative integers $\{r_\ell\}$ satisfy
$\sum_{\ell} r_\ell=k$ and $\sum_{\ell} \ell r_\ell=m$.
Using \eqref{eq:G-global-u}, namely
$|G^{(k)}(u)|\le C_n^{*}\,u^{\,n^2-k}$ for $0\le k\le n-1$, together with bounds \eqref{eq:left-u0} and~\eqref{eq:left-uj}
on $u\circ t_\eta$ and its derivatives, we obtain
\[
\begin{aligned}
\bigl|G^{(k)}(u(t_\eta(x)))\bigr|
&\le A_1\,\eta^{(\alpha-1)(n^2-k)},\\
\prod_{\ell}\bigl|(u\circ t_\eta)^{(\ell)}(x)\bigr|^{r_\ell}
&\le A_1\,\eta^{2k(\alpha-1)-\alpha m}
 \le A_1\,\eta^{k(\alpha-1)-\alpha m},
\end{aligned}
\]
where $A_1=A_1(M, n)$ is a constant only depending on $M$ and $n$. Multiplying these estimates shows that every such term is bounded by 
$\mathcal{O}(\eta^{(\alpha-1)(n^2-k)}\,\eta^{k(\alpha-1)-\alpha m})
= \mathcal{O}(\eta^{(\alpha-1)n^2-\alpha m})
$.
Since the number of terms depends only on $m\le n-1$, it can be absorbed into the
constant. Therefore, there exists a constant $A_2=A_2(M, n)$ such that
\begin{equation}\label{eq:left-gW-fixed-1}
\begin{aligned}
\|g_\eta^{(m)}\|_{L^\infty([-2M,0])}
&\le
A_2\,\eta^{2n-\alpha m},
\\
\implies\quad
\|g_{\eta}\|_{W^{n-1,\infty}([-2M,0])}
&\le
A_2\,\eta^{\,n-1+\frac2n}.
\end{aligned}
\end{equation}
due to the equality \((\alpha-1)n^2=2n\).

\noindent (B) \textbf{Estimating $1-g_{\eta}$ on $[\eta, 2M]$}

For $t<-\frac92$, by the definition of $\widetilde{\DUAF}_n$ and the normalization of $d_n$,
\[
1+\widetilde{\DUAF}_n(t)
=
-d_n\int_{-\infty}^{t}\frac{\DUAF_n(s)+\frac95}{s^2}\,ds,
\qquad
\left|1+\widetilde{\DUAF}_n(t)\right|
\le
\frac{d_n(K+\frac95)}{|t|}.
\]
For $x\in[\eta,2M]$, recall that $\eta<10^{-\frac{1}{\alpha-1}}$, we have
\[
t_\eta(x)=\eta^{-\alpha}\left(\frac\eta2-x\right)
\le -\frac12\eta^{1-\alpha}<-5<-\frac{9}{2}.
\]
Using the fact that $|t_{\eta}(x)|\ge|t_{\eta}(\eta)|=\frac{1}{2}\eta^{1-\alpha}$ whenever $x\ge\eta$, we deduce that
\begin{equation}\label{eq:one-minus-u-eta}
1-u(t_\eta(x))
=
\frac{1+\widetilde{\DUAF}_n(t_\eta(x))}{2}
\le
d_n\left(K+\frac95\right)\eta^{\alpha-1}
,
\qquad
x\in[\eta,2M].
\end{equation}
Moreover, since
\[
\widetilde{\DUAF}_n'(t)
=
-d_n\frac{\DUAF_n(t)+\frac95}{t^2},
\qquad t<-\frac{9}{2},
\]
repeated differentiation and the global bounds
\(\|\DUAF_n^{(j)}\|_{L^\infty(\mathbb R)}\le K\), \(0\le j\le n\), give a constant
\(\tilde D_n>1\) such that
\[
\left|\widetilde{\DUAF}_n^{(j)}(t)\right|
\le
\tilde D_n d_n\left(K+\frac95\right)|t|^{-2},
\qquad 1\le j\le n-1.
\]
Therefore, for \(1\le j\le n-1\),
\begin{equation}\label{eq:ucirc-t-eta}
\left|(u\circ t_\eta)^{(j)}(x)\right|
=
\frac{\eta^{-\alpha j}}{2}
\left|\widetilde{\DUAF}_n^{(j)}(t_\eta(x))\right|
\le
\tilde D_n d_n\left(K+\frac95\right)
\eta^{2(\alpha-1)-\alpha j}.
\end{equation}
Denote $B_1:=\tilde D_n d_n\left(K+\frac95\right)$ and fix \(m\in\{0,\dots,n-1\}\). Since
\[
1-g_\eta(x)=(1-G)(u(t_\eta(x))),
\]
the chain rule gives that \(\partial_x^m(1-g_\eta)\) is a finite sum of terms
\[
(1-G)^{(k)}(u(t_\eta(x)))
\prod_{\ell=1}^{m}
\bigl((u\circ t_\eta)^{(\ell)}(x)\bigr)^{r_\ell},
\qquad
\sum_{\ell}r_\ell=k,\quad
\sum_{\ell}\ell r_\ell=m.
\]
Using
\[
|(1-G)^{(k)}(y)|\le C_n^*(1-y)^{n^2-k},
\qquad y\in[0,1],
\]
together with \eqref{eq:one-minus-u-eta} and \eqref{eq:ucirc-t-eta}, each such term is bounded by
\[
\left(C_n^{*}
B_1^{n^2-k}
\eta^{(\alpha-1)(n^2-k)}\right)
\cdot
\left(B_1^k
\eta^{2k(\alpha-1)-\alpha m}\right)
\le
C_n^{*}
B_1^{n^2}
\eta^{(\alpha-1)n^2-\alpha m}.
\]
Therefore, there exists a constant $B_2=B_2(M, n)>0$ such that
\begin{equation}\label{eq:left-gW-fixed-2}
\begin{aligned}
\|(1-g_\eta)^{(m)}\|_{L^\infty([\eta,2M])}
&\le
B_2
\eta^{(\alpha-1)n^2-\alpha m}
=B_2
\eta^{2n-\alpha m},
\\
\implies\quad
\|1-g_\eta\|_{W^{n-1,\infty}([\eta,2M])}
&\le
B_2\,
\eta^{\,n-1+\frac2n}.
\end{aligned}
\end{equation}
due to the equality \((\alpha-1)n^2=2n\).

\smallskip
\noindent (C) \textbf{A rough estimate for $1-g_\eta$ and $g_\eta$ on $[0,\eta]$. }

We first observe that we have the following rough bound in general, 
\[
\|(u\circ t_\eta)^{(\ell)}\|_{L^\infty([0,\eta])}
=
\frac{\eta^{-\alpha\ell}}{2}\,\|(1-\widetilde{\DUAF}_n)^{(\ell)}\|_{L^\infty(t_\eta([0,\eta]))}
\le
\frac{1+K}{2}\,\eta^{-\alpha\ell}, \qquad 0\le \ell \le n-1.
\]
Using \eqref{eq:G-global-u} with $y=u(t_\eta(x))\in[0,1]$ gives the uniform bound
\[
|G^{(k)}(u(t_\eta(x)))|\le C_n^{*},
\quad
|(1-G)^{(k)}(u(t_\eta(x)))|\le C_n^{*}, 
\qquad 0\le k \le n-1.
\]
By repeated application of the chain rule and product rule as we did for the previous two intervals, we conclude that there exists a constant $C_2=C_2(M,n)>0$ such that

\begin{equation}\label{eq:trans-1g-rough-full}
\max\left\{\|g_\eta^{(m)}\|_{L^\infty([0,\eta])}, \|(1-g_\eta)^{(m)}\|_{L^\infty([0,\eta])}\right\}
\le
C_2\,\eta^{-\alpha m},
\qquad 0\le m\le n-1.
\end{equation}
\smallskip
\noindent (D) \textbf{Sobolev estimate of $\chi_{1,\eta}$, $\chi_{3,\eta}$, and $\chi_{2,\eta}$ on the relevant subintervals }

Summarizing inequalities \eqref{eq:left-gW-fixed-1}, \eqref{eq:left-gW-fixed-2}, and \eqref{eq:trans-1g-rough-full}, we obtain the following bounds in big-\(\mathcal{O}\) notation, where the implicit constants may depend on \(M\) and \(n\), but are independent of \(\eta\).
\[
\begin{aligned}
\left\{
\begin{aligned}
g_\eta^{(m)} &= \mathcal{O}(\eta^{2n-\alpha m}) && \text{on } [-2M,0],\\
(1-g_\eta)^{(m)} &= \mathcal{O}(\eta^{2n-\alpha m}) && \text{on } [\eta,2M],
\end{aligned}
\right.
\qquad
\left\{
\begin{aligned}
g_\eta^{(m)} &= \mathcal{O}(\eta^{-\alpha m}) && \text{on } [0,\eta],\\
(1-g_\eta)^{(m)} &= \mathcal{O}(\eta^{-\alpha m}) && \text{on } [0,\eta].
\end{aligned}
\right.
\end{aligned}
\]
for \(0\le m\le n-1\). 

All subsequent estimates are obtained by repeated applications of the Leibniz rule, combined with the above pointwise bounds for $g_\eta$ and $1-g_\eta$ on each region.

On \([-\frac92,M]\), for
\(\chi_{1,\eta}=(1-g_\eta(x+\frac92+\eta))(1-g_\eta(x))\),
the shifted factor is evaluated in \([\eta,2M]\) and contributes
\(\mathcal{O}(\eta^{2n-\alpha i})\), while the second factor is controlled by
\(\mathcal{O}(\eta^{-\alpha j})\). Hence
\[
\|\chi_{1,\eta}\|_{W^{n-1,\infty}([-\frac92,M])}
= \mathcal{O}(\eta^{2n-\alpha(n-1)})=\mathcal{O}(\eta^{\,n-1+\frac2n}).
\]

On \([-M,0]\), since \(\chi_{3,\eta}=g_\eta\),
\[
\|\chi_{3,\eta}\|_{W^{n-1,\infty}([-M,0])}
= \mathcal{O}(\eta^{2n-\alpha(n-1)})=\mathcal{O}(\eta^{\,n-1+\frac2n}).
\]

For \(\chi_{2,\eta}=g_\eta(x+\frac92+\eta)(1-g_\eta(x))\), on the outer regions
\([-M,-\frac92-\eta]\cup[\eta,M]\), one factor is small and the other is rough,
thus
\[
\|\chi_{2,\eta}\|_{W^{n-1,\infty}([-M,-\frac92-\eta]\cup[\eta,M])}
= \mathcal{O}(\eta^{2n-\alpha(n-1)})=\mathcal{O}(\eta^{\,n-1+\frac2n}).
\]

On the transition regions \([-\frac92-\eta,-\frac92]\cup[0,\eta]\), one of the two factors is in its transition region, while the other one is uniformly bounded and its derivatives are controlled by the corresponding off-transition estimates. Hence, by the Leibniz rule, for every \(0\le m\le n-1\),
\[
\|\chi_{2,\eta}^{(m)}\|_{L^\infty([-\frac92-\eta,-\frac92]\cup[0,\eta])}
=
\mathcal{O}(\eta^{-\alpha m}).
\]
This completes the proof.

\section{Conclusion}
\label{sec:conclusion}

This paper develops a fixed-size approximation theory for neural networks in
Sobolev norms.  We first showed that the elementary universal activation function
\EUAF{} yields arbitrary-accuracy approximation in \(W^{1,\infty}\) for
\(W^{2,\infty}\) targets.  We then introduced the differentiable universal
activation functions \(\DUAF_n\) and \(\DUAF_\infty\), and proved fixed-size
approximation of \(W^{s,\infty}\) functions in the \(W^{s-1,\infty}\)-norm for
\(1\le s\le n+1\).  We further constructed bounded monotone sigmoidal variants
\(\widetilde{\DUAF}_n\), showing that the fixed-size Sobolev approximation
phenomenon persists in a sigmoidal class.  For structured Sobolev targets admitting
regular Kolmogorov-type superpositions, we also obtained architectures with
dimensionally more efficient width.

These results show that super-expressive activations can be designed not only for
function-value approximation in \(C([a,b])\) or \(L^p([a,b])\), but also for controlling
weak derivatives.  Several questions remain open.  It would be interesting to
generalize the \DUAF construction, optimize the dependence of the architecture on
\(s\), \(d\), and \(n\), and study the optimization error and generalization errors
of such fixed-size networks.

A particularly natural question is whether there exists a single activation that is
elementary, analytic, sigmoidal, and super-expressive in Sobolev
norms.

\begin{Backmatter}

%
\paragraph{Funding statement}
H. Y. was partially supported by the US National Science Foundation under awards DMS-2244988, the Office of Naval Research Award N00014-23-1-2007, and the Department of Energy ASCR Award DE-SC0026052. 
S. Z. was partially supported by
start-up fund  P0053092  from the
Hong Kong Polytechnic University.

\paragraph{Competing interests}
The authors declare none.

\paragraph{Data availability statement}
Data sharing is not applicable to this article as no datasets were generated or analyzed.

\paragraph{Author contributions}
All authors contributed to the development of the theory and preparation of the manuscript.


\bibliographystyle{abbrvnat}
\bibliography{references}

\end{Backmatter}

\end{document}